\documentclass[Afour,sageh,times]{IEEEtran}

\usepackage{amsmath,amssymb,euscript,yfonts,psfrag,latexsym,dsfont,graphicx}
\usepackage{bbm,color,amstext,wasysym,parskip,balance}

\usepackage{amsthm}
\usepackage{caption}
\usepackage{subcaption}
\usepackage{graphics} 
\usepackage{times} 
\usepackage{cite}
\usepackage{mathtools}
\usepackage{lipsum}
\usepackage{float} 

\usepackage[ruled,vlined,linesnumbered]{algorithm2e}

\usepackage{amsmath}    
\usepackage[normalem]{ulem}

\newcommand{\bmu}{{\boldsymbol{\mu}}}

\newcommand{\bK}{\mathbf{K}}
\newcommand{\bh}{\mathbf{h}}
\newcommand{\bX}{\mathbf{X}}

\newcommand{\cP}{{\mathbb P}}
\newcommand{\cN}{{\mathcal N}}
\newcommand{\cY}{{\mathcal Y}}
\newcommand{\mE}{{\mathbb E}}
\newcommand{\mP}{{\mathbb P}}
\newcommand{\mR}{{\mathbb R}}
\newcommand{\HH}{{\mathrm H}}

\DeclareMathOperator*{\argmin}{argmin}

\newtheorem{thm}{Theorem}

\newtheorem{proposition}{Proposition}
\newtheorem{lemma}{Lemma}

\usepackage{enumitem}
\setlist{  
  listparindent=\parindent,
  parsep=0pt,
}

\newcommand{\tr}{\operatorname{Tr}}

\usepackage{subcaption}
\usepackage{color}
\usepackage{mathtools}
\usepackage{hyperref}
\hypersetup{
    colorlinks=true,
    linkcolor=blue,
    filecolor=magenta,      
    urlcolor=cyan,
    pdftitle={Overleaf Example},
    pdfpagemode=FullScreen,
    }

\newcommand{\hongzhe}[1]{{\color{red}{#1}}}
\definecolor{emerald}{rgb}{0.31, 0.78, 0.47}

\newlength{\bibitemsep}
\newlength{\bibparskip}
\let\oldthebibliography\thebibliography
\renewcommand\thebibliography[1]{%
  \oldthebibliography{#1}%
  \setlength{\parskip}{\bibitemsep}%
  \setlength{\itemsep}{\bibparskip}%
}

\title{
Stochastic Motion Planning as Gaussian Variational Inference: Theory and Algorithms
}

\author{ Hongzhe Yu, Zinuo Chang, and Yongxin Chen
\thanks{Financial support from NSF under grants 1942523, 2008513 are greatly acknowledged.}
\thanks{H. Yu and Y.\ Chen are with the School of Aerospace Engineering, Z. Chang is with the School of Electrical and Computer Engineering, Georgia Institute of Technology, Atlanta, GA; {\{hyu419, zchang40, yongchen\}@gatech.edu}}
\thanks{The authors acknowledge Shuo Cheng for providing valuable support on the hardware experiments.}
}

\begin{document}

\maketitle

\begin{abstract}
We present a novel formulation for motion planning under uncertainties based on variational inference, where the optimal motion plan is modeled as a posterior distribution. We propose a Gaussian variational inference-based framework, termed Gaussian Variational Inference Motion Planning (GVI-MP), to approximate this posterior by a Gaussian distribution over the trajectories. We show that the GVI-MP framework is dual to a special class of stochastic control problems and brings robustness into the decision-making in motion planning. We develop two algorithms to numerically solve this variational inference and the equivalent control formulations for motion planning. The first algorithm uses a natural gradient paradigm to iteratively update a Gaussian proposal distribution on the sparse motion planning factor graph. We propose a second algorithm, the Proximal Covariance Steering Motion Planner (PCS-MP), to solve the same inference problem in its stochastic control form with an additional terminal constraint. We leverage a proximal gradient paradigm where, at each iteration, we quadratically approximate nonlinear state costs and solve a linear covariance steering problem in closed form. The efficacy of the proposed algorithms is demonstrated through extensive experiments on various robot models. An implementation is provided in \url{https://github.com/hzyu17/VIMP}.

\end{abstract}

\thispagestyle{empty}
\pagestyle{empty}


\section{Introduction}\label{sec:introduction}

Motion planning is a fundamental task for robotics systems. It solves the problem of finding a smooth and `feasible' trajectory connecting two endpoints in the configuration space while achieving desired behaviors measured by `optimality' of the plan. On top of the optimality, in real-world practices, one often encounters ubiquitous uncertainties arising from modeling errors, sensor noise, and actuator noise. Scenarios such as UAVs navigating turbulent airstreams \cite{singh2017robust, majumdar2017}, legged robots traversing unknown terrains \cite{dai2012optimizing}, and manipulators grasping with sensor-induced noise \cite{kappler2015data} underscore the indispensability of motion planning techniques that accommodate dynamic and environmental uncertainties. Robust motion planning in the presence of uncertainties remains a challenging task. Instead of a deterministic trajectory, the notion of a \textit{`tube'}, \textit{`funnel'}, or \textit{`belief'} \cite{majumdar2017, prentice2009belief} of trajectories becomes our primary pursuit, where a collection or a distribution of trajectories is considered for accommodating stochastic or partially observed states. In the planning under uncertainty paradigm, in addition to the `feasible' and `optimal' objectives, `robustness' appears in the cost functionals as a risk indication in the decision-making of finding a motion plan \cite{singh2017robust} \cite{majumdar2017} \cite{kohler2020computationally}. 

Motion planning can be reformulated as a probabilistic inference problem in \cite{mukadam2016, mukadam2018}. In this formulation, the optimal trajectory for a given environment is modeled as a target posterior distribution under the linear Gaussian Process (GP) assumption on the robot dynamics prior and a collision-avoiding likelihood. In \cite{mukadam2016, mukadam2018}, the authors proposed an iterative algorithm to find the trajectory that maximizes the posterior probability. However, the solution is still deterministic, which does not handle uncertainties in planning problems. 

This work presents \textit{Gaussian Variational Inference Motion Planning (GVI-MP) framework}. We formulate motion planning under uncertainties as a Gaussian Variational Inference (GVI) over a sparse factor graph. This new formalism enables efficient GVI techniques that could not previously be applied in motion planning. GVI-MP solves for a joint trajectory distribution as its output. The trajectory distribution addresses the \textit{`funnel'} and \textit{`belief'} pursuits. We show that the distribution representation is intrinsically superior to deterministic representations, thanks to the equivalence of GVI-MP to an entropy-maximized robust motion planning formulation. The entropy of the trajectory distribution serves as a risk measure. Finally, the motion planning problem enjoys a sparse factor graph structure, where closed-form expressions and sparse Gauss-Hermite quadrature techniques are leveraged in evaluating the functional gradients to increase the scalability of the algorithms for high DOF robotics systems. 

The new GVI-MP formalism is closely related to the classical optimal control formulations. We present a theoretical analysis to bridge the GVI-MP and a stochastic optimal control \cite{bertsekas1996stochastic} problem for general time-varying linear Gauss-Markov systems. Starting from the optimal control formulation, we present our second algorithm, \textit{Proximal Covariance Steering Motion Planning (PCS-MP)}, that solves the stochastic control problem with a terminal covariance constraint. PCS-MP leverages the proximal gradient descent paradigm in the space of distributions, which has a sublinear convergence rate. At each iteration, it has a closed-form solution to a linear covariance steering problem, making the overall algorithm efficient for high DOF robotics systems.

\subsection{Related Works}
\label{sec:intro_related_works}
{\em (a) Motion Planning.} Path and motion planning often refer to the problem of directly finding an optimal or feasible sequence of robot configurations \cite{lavalle1998}. Sampling-based motion planning gained popularity in solving path-planning problems. Rapidly exploring Random Trees (RRT) \cite{lavalle1998}, Probabilistic Road Map (PRM) \cite{kavraki1996}, and their variations perform a tree search or a graph search in the configuration space to find feasible paths. Sampling-based methods became successful thanks to their effectiveness in finding feasible paths in high-dimensional configuration spaces. However, they do not consider the underlying continuous-time dynamics. They also become ineffective when the feasible set represents a very low probability within the search space, such as navigating through a narrow gap between obstacles. Other methods for motion planning have been proposed in the literature. The work \cite{donald1993kinodynamic} is a classical work in kinodynamic planning. In \cite{wang2022geometrically}, an efficient geometrically constrained trajectory optimization method is proposed for multi-copters. The work \cite{liu2022free} presents an algorithm for finding and navigating narrow passages for high-DOF kinematic chains. 

Trajectory optimization \cite{ratliff2009chomp, kalakrishnan2011stomp, schulman2014, mukadam2016, mukadam2018} formulates motion planning as an optimization, through which the `optimality' aspect of planning is encoded in the objective function, and system dynamics is encoded in the constraints. Covariant Hamiltonian Optimization for Motion Planning (CHOMP) \cite{ratliff2009chomp} \cite{zucker2013chomp} is a variational strategy that uses gradient-based covariate optimization in the continuous-time configuration space. TrajOpt \cite{schulman2014} developed a sequential convex optimization algorithm to solve the same problem formulation. Stochastic Trajectory Optimization for Motion Planning (STOMP) \cite{kalakrishnan2011stomp} approached a stochastic optimization formulation and used sampled noisy trajectory costs to iteratively update the current trajectory, which does not need gradient information.

The most relevant work to our work is the Gaussian Process Motion Planning (GPMP) \cite{mukadam2016, mukadam2018}, which formulated motion planning as a probability inference problem \cite{sarkka2013spatiotemporal} over a sparse factor graph \cite{dellaert2012factor}, and proposed a Maximum-a-Posteriori (MAP) inference algorithm to find the optimal trajectory. GPMP assumed a continuous-time linear Gauss-Markov process representation for the trajectory and lifted the collision-related constraints to a collision-related likelihood probability. The optimal trajectory is modeled as a posterior distribution that is the product of a dynamics Gaussian prior and a nonlinear collision-free likelihood. Our work is also related to the Partially Observable Markov Decision Process \cite{kurniawati2008sarsop, somani2013despot} framework.

{\em (b) Variational Inference.}
Variational Inference (VI) is a class of methods to solve graphical probabilistic inference \cite{jordan1999introduction} by formulating an optimization over distributions. Parametric VI is a class of methods where the approximating distribution is tractable, such as Gaussian or exponential family \cite{opper2009}. Gaussian variational inference was applied in robot skill learning, reinforcement learning, and deep latent space learning \cite{rezende2014stochastic, shankar2020learning}. Gaussian approximations of posteriors are widely used in nonlinear filtering \cite{garcia2015posterior}. In \cite{barfoot2020exactly}, a sparse GVI algorithm using natural gradient descent and numerical integration techniques was applied to a nonlinear batch state estimation filtering task. Despite the successes in these fields, to the authors' best knowledge, few studies on GVI yielded satisfying results in variational motion planning tasks.  

The other class is non-parametric VI. The most popular non-parametric VI method is the Stein Variational Gradient Descent (SVGD) \cite{liu2016stein} technique, where the approximating distribution is represented using particles. While SVGD is elegant mathematically, the need to define a Reproducing Kernel Hilbert Space a-priori poses practical issues such as mode collapse and scalability.  

{\em (c) Inference-control Duality and Covariance Steering.}
It is particularly interesting to investigate the duality of an inference-based planner with the standard control-based formulation, which provides a strong theoretical foundation for the former. The duality between a particular form of stochastic control and graphical probability inference problems in the discrete-time setting was studied under the names of Kullback–Leibler (KL) control, Maximum-Entropy reinforcement learning, and active-inference based control \cite{kappen2012optimal, todorov2008general, toussaint2009robot, millidge2020relationship, friston2016active}. 

A standard technique to solve nonlinear optimal control problems is to linearize the dynamics and quadratically approximate the cost iteratively, and perform a linear quadratic Gaussian (LQG) regulator to the approximated problem to obtain a locally optimal feedback controller. This procedure is the `iterative LQG' (iLQG) algorithm \cite{todorov2005generalized}. LQG relies on a designed terminal cost to achieve terminal-time behavior without guaranteeing the goal state covariance constraint. 

A linear stochastic control with terminal constraint is termed \textit{`Covariance Steering'} \cite{CheGeoPav15a, CheGeoPav15b, CheGeoPav17a, goldshtein2017finite}. The strict satisfaction of terminal covariance constraints is achieved by solving a coupled Riccati equation with proper initial values. ovariance steering is closely related to the Schr\"odinger's bridge problem (SBP) \cite{Chen2016relation}. In the same spirit of iLQG, the authors developed an iterative covariance steering algorithm in \cite{yu2023covariance} leveraging proximal gradient in the space of path distributions. Similar approaches have been investigated on discrete-time dynamical systems with chance constraints \cite{ridderhof2020chance, okamoto2018optimal} wherein the algorithm solves a convex optimization at each iteration. Compared to these works, in our method, we leverage a closed-form solution at each iteration to improve efficiency. 

\subsection{Contributions}
\label{sec:intro_contributions}
\noindent The contributions of this paper are summarized as follows.
\begin{enumerate}
    \item We present a novel variational motion planning paradigm that optimizes trajectory distributions using variational inference. We show that planning for a trajectory distribution is superior to deterministic baselines in terms of robustness.
    \item We provide a rigorous theoretical analysis of the GVI-MP formulation. We show that the GVI-MP is equivalent to a class of stochastic control problems. 
    \item We propose an efficient GVI-MP algorithm that leverages the sparse factor graph of planning problems. We present a complexity analysis of the algorithm.
    \item We propose PCS-MP, an efficient iterative algorithm to solve stochastic control problems with terminal constraints and nonlinear collision-avoiding state costs. 
    \item Extensive experiments are conducted to validate the two proposed algorithms on high DOF robotics systems.
\end{enumerate}

This paper is built on our previous work \cite{yu2023}. The rest of the paper is organized as follows. Section \ref{sec:background} introduces background knowledge. Section \ref{sec:GVIMP_SOC_formulation} introduces the GVI-MP paradigm and its equivalence to a stochastic control formulation. Section \ref{sec:GVIMP_algorithm} presents the Gaussian Variational Inference Motion Planner algorithm. Section \ref{sec:robust_MP} discusses the robustness of GVI-MP solution. Section \ref{sec:PCSMP} presents the Proximal Covariance Steering Motion Planner (PCS-MP) algorithm for the control formulation. Section \ref{sec:experiments} presents experiment results, followed by the conclusion in Section \ref{sec:conclusion}.

\section{Background}\label{sec:background}
This section introduces the necessary background for our methods, including trajectory optimization, probabilistic inference motion planning, and linear covariance steering.

\subsection{Trajectory Optimization Problem Formulations}
\label{sec:background_trajectory_optimization}
{\em a) Motion planning as an optimal control.}
Trajectory optimization leverages optimal control theory and formulates motion planning problems as an optimization in the time window $t\in [0,T]$ as follows
\begin{equation}
\begin{split}
    &\min_{X_t, u_t} \; J_0(X_t, u_t) \\
    &\; \; {\rm s.t.} \;\; g^c_i(X_t, u_t) \leq 0,\; i=1, \dots, N_g\\
    &\; \; \;\;\;\; \;\; f^c_i(X_t, u_t) = 0,\; i=1, \dots, N_f.
\end{split}
\label{eq:trajectory_optimization_formulation}
\end{equation}
Problem \eqref{eq:trajectory_optimization_formulation} is an optimization over the continuous-time variables $(X_t, u_t)$, where $X_t: [0,T]\to \mR^n$ is the state function of time of dimension $n$, and $u_t: [0,T]\to \mR^m$ represents the control signal function of dimension $m$. $J_0$ is a cost function that often integrates a running cost function of $(X_t, u_t)$ over $[0, T]$. $\{g^c_i\}_{i=1}^{N_g}$, $\{f^c_i\}_{i=1}^{N_f}$ are the functions defining the inequality and equality constraints that the motion planning problem is subject to, respectively.

{\em b) Discrete formulations.}
To transform the continuous-time formulations \eqref{eq:trajectory_optimization_formulation} into computable optimization programs, we define the time discretization
\begin{equation}
\label{eq:time_discretization}
    \mathbf{t}\triangleq \left[t_0, \dots, t_N\right],\; t_0 = 0, \; t_N=T
\end{equation}
that generates a vector of discretized states and control inputs of length $N+1$. We denote the \textit{support states} and \textit{support controls} as the discretized variables
\begin{equation}
    \label{eq:definition_bX}
    \bX \triangleq [X_0, \dots, X_N]^T, \; \mathbf{U} \triangleq [u_0, \dots, u_N]^T,
\end{equation}
where 
\begin{equation*}
    X_i = X_{t_i}, \; u_i = u_{t_i}, i=1,\dots,N.
\end{equation*}
When $\bX$ is a random variable, the covariance matrix of $\bX$ is denoted as $\bK$. The discretized variables have the dimensions $\bX \in \mathbb{R}^{(N+1)\times n}$, $\mathbf{U} \in \mathbb{R}^{(N+1)\times m}$, and $\bK \in \mathbb{R}^{(N+1)n \times (N+1)n}$.

\subsection{Linear Gauss-Markov Prior}
\label{sec:background_linear_Gauss_markov}
Following \cite{mukadam2016,mukadam2018}, we assume that a \textit{prior} robot trajectory is sampled from a parameterized Gaussian Process (GP) 
\[
X_t\sim \mathcal{GP}(\mu_t, \mathcal{K}(t,t')),
\]
where $\mu_t$ represents the mean function of the GP, and $\mathcal{K}(t,t')$ 
is the covariance function. Specifically, we assume the kernel $\mathcal{GP}(\mu_t, \mathcal{K}(t,t'))$ is generated by an \textit{uncontrolled} LTV system 
\begin{equation}\label{eq:dynamics_uncontrolled}
    d X_t =\! A_tX_t dt \! +\! a_t dt \!+\! B_t dW_t, \; X_0 \sim \mathcal{N}(\mu_0, K_0).
\end{equation} 
Here the time-varying matrices are $A_t \in \mathbb{R}^{n\times n}$ and $B_t \in \mathbb{R}^{n\times m}$, $a_t \in \mathbb{R}^{n}$ is a drift term, and $W_t$ is a standard Wiener process. $\mathcal{N}(\mu_0, K_0)$ denotes a Gaussian with mean $\mu_0$ and covariance $K_0$.

Denote the state transition matrix associated with system \eqref{eq:dynamics_uncontrolled} as $\Phi(t,s)$, that is, $\Phi(t,t) = I, ~ \Phi^{-1}(t,s) = \Phi(s,t),$ and $\partial_t \Phi(t,s) = A_t\Phi(t,s)$ for all $(t,s)$. The solution to \eqref{eq:dynamics_uncontrolled} is 
\begin{equation}\label{eq:sol_LTV_SDE}
    X_t = \Phi(t,0)X_0 + \int_{0}^{t} \Phi(t,s) a_s d s + \int_{0}^{t}\Phi(t,s)B_sd W_s,
\end{equation}
and thus the mean and covariance functions of the GP are
\begin{subequations}
\begin{eqnarray}
    \mu_t \!\!\!&=&\!\!\! \Phi(t,0)\mu_0 + \int_{0}^{t} \Phi(t,s)a_s ds, \label{eq:prior_mean_dyn}\\
    K(t,t') \!\!\!&=&\!\!\! \Phi(t,0)K_0\Phi(t',0)^T + \nonumber
    \\ 
    \!\!\!&&\!\!\!\hspace{-0.3cm}\int_{0}^{{\rm min}(t,t')} \Phi(t,s)B_sB_s^T\Phi(t',s)^T ds. \label{eq:prior_kernel_dyn}
\end{eqnarray}
\end{subequations}
We denote the \textit{support mean} on the discretization \eqref{eq:time_discretization} as
\begin{equation}
\label{eq:supported_mean}
    \bmu \triangleq [\mu_0, \dots, \mu_N], \; \mu_i = \mu_{t_i}, \; i=1, \dots, N.
\end{equation}
We also use the simplified notation 
\begin{equation}
    \label{eq:state_transition_discrete}
    \Phi_{i+1, i} \triangleq \Phi(t_{i+1}, t_i)
\end{equation}
to represent the state transition between two adjacent support states.
We assume the system $(A_t, B_t)$ is controllable in the sense that 
the controllability Grammian associated with \eqref{eq:dynamics_uncontrolled} 
\begin{equation}
\label{eq:def_Grammian}
    Q_{s,t} \triangleq \int_{s}^{t} \Phi(t, \tau)B_\tau B_\tau^T \Phi(t, \tau)^T d \tau
\end{equation}
is invertiable for all $s<t$.

We denote the deviation from state $X_t$ to the mean $\mu_t$ as $\Tilde{X}_t \triangleq X_t - \mu_t$. By \eqref{eq:sol_LTV_SDE} and \eqref{eq:prior_mean_dyn}, we have
\begin{equation}
    \label{eq:tilde_Xt}
    \Tilde{X}_t = \Phi(t,0)\Tilde{X}_0 + \int_{0}^{t}\Phi(t,s)B_sd W_s.
\end{equation}

\subsection{Probabilistic Inference Motion Planning}
\label{sec:background_gpmp}
Motion Planning problems admit a probabilistic inference dual formulation \cite{toussaint2009robot, mukadam2016, mukadam2018}. The Gaussian Process Motion Planning (GPMP) algorithm formulated the motion planning problem \eqref{eq:trajectory_optimization_formulation} for system \eqref{eq:dynamics_uncontrolled} as a posterior over $\bX$
\begin{equation}
\label{eq:MP_posterior}
    p(\bX|Z) \propto p(\bX) p(Z|\bX),
\end{equation}
where $Z$ encodes the environment. 
Consider the state, control, and deviated states over the same time discretization as in \eqref{eq:time_discretization}, \eqref{eq:definition_bX}, and \eqref{eq:supported_mean}. Let $t_0 = 0$ and $t_N = T$. $p(\bX)$ and $p(Z|\bX)$ have the following structures.

{\em a) Gaussian trajectory prior.}
For a GP, the prior joint distribution is a Gaussian, which can be written as
\begin{equation*}
    \hat{p}(\bX) \triangleq p(X_0)p(X_{1}|X_{0}) \dots p(X_{N}|X_{N-1}).
\end{equation*}
A terminal condition at $t_N$ can be imposed by conditioning $\hat{p}(\bX)$ on a fictitious observation $O_{t_N} \sim \mathcal{N}(\mu_N, K_N)$, leading to a conditional Gaussian \textit{trajectory prior} \cite{mukadam2018}
\begin{equation}\label{eq:prior_factor}
    p(\bX) \triangleq \hat{p}(\bX|O_{t_N}) \propto \exp (-\frac{1}{2}\lVert \bX - \bmu \rVert_{\mathbf{K}^{-1}}^2).
\end{equation}
For linear Gauss-Markov process \eqref{eq:dynamics_uncontrolled}, 
the joint precision matrix $\bK^{-1}$ is sparse \cite{barfoot2014batch} and
\begin{equation}
\label{eq:def_K_inv}
    \bK^{-1} \triangleq G^{T}Q^{-1}G,
\end{equation}
where
\begin{equation}\label{eq:sparse_G}
    G = \begin{bmatrix}
    I    &     &   &   & \\
    -\Phi(t_1, t_0) & I &   &   & \\
     &  &  \dots   &  &\\
     &  &  & -\Phi(t_N, t_{N-1}) &  I\\
     &  &  &  0 & I 
\end{bmatrix},
\end{equation} 
and
\begin{equation}\label{eq:sparse_Q}
    Q^{-1} = {\rm diag}(K_0^{-1}, Q_{0,1}^{-1},\ldots, Q_{N-1,N}^{-1}, K_N^{-1}).
\end{equation}
Here $Q_{i,i+1}$ is exactly the Grammian defined in \eqref{eq:def_Grammian} on $(t_i, t_{i+1})$, that is, $Q_{i,i+1}=Q_{t_i,t_{i+1}}$, and $K_0, K_N$ are start and goal penalty terms. 

{\em b) Collision avoidance likelihood.}
The likelihood factor is defined by collision costs 
\begin{equation}\label{eq:collision_factor}
    p(Z|\bX) \propto \exp \left(-\lVert \Tilde{\bh}_{\epsilon_{\rm sdf}}\left(S\left(F\left(\bX\right)\right) \right) \rVert_{\Sigma_{\rm obs}}^2\right),
\end{equation}
where $\tilde{\mathbf{h}}_{\epsilon_{\rm sdf}}(\cdot)$ is a hinge loss function with cut-off distance $\epsilon_{\rm sdf}$, $S(\cdot)$ is the signed distance to the obstacles, $F(\cdot)$ represents the robot's forward kinematics, and $\Sigma_{obs}$ is a weight applied on the hinge losses.
For simplicity, we denote the composite loss function by $\mathbf{h}(\bX) \triangleq \Tilde{\bh}(S(F(\bX))).$

Combining \eqref{eq:prior_factor} and \eqref{eq:collision_factor}, a maximum-a-posterior (MAP) algorithm was proposed in GPMP to solve 
\begin{subequations}
\begin{eqnarray}
    \bX^\star \!\!\!&=&\!\!\! \underset{\bX}{\arg\max}\; p(\bX|Z)  
    \\ 
    \!\!\!&=&\!\!\! \underset{\bX}{\arg\max}\; p(\bX)p(Z|\bX)
    \\
    \!\!\!&=&\!\!\! \underset{\mathbf{X}}{\arg\min}\; \lVert \bh(\bX) \rVert_{\Sigma_{\rm obs}}^2 + \frac{1}{2}\lVert \bX - \bmu \rVert_{\bK^{-1}}^2.
\end{eqnarray}
\label{eq:MAP_formulation}
\end{subequations} 

\subsection{Trajectory Distribution and Girsanov Theorem}
\label{sec:background_Girsanov}

Girsanov \cite{sarkka2019applied} theorem is an important theorem used to transform the probability measures of path distributions. 
Given two stochastic differential equations (SDEs) with the same initial conditions as 
\begin{subequations} 
\begin{eqnarray}
    dX_t \!\!\!&=&\!\!\! f_1(X_t, t) d t + d W_t, \label{eq:nonlinear_SDE_1}
    \\
    dX_t \!\!\!&=&\!\!\! f_2(X_t, t) d t + d W_t,\label{eq:nonlinear_SDE_2}
\end{eqnarray}
\end{subequations}
Girsanov states that the ratio between the measures generated by \eqref{eq:nonlinear_SDE_1} and \eqref{eq:nonlinear_SDE_2} is given by 
\begin{equation*}
\begin{aligned}
    Z_t &= \exp \left(-\frac{1}{2} \int_{0}^{t}[f_1(X_\tau) - f_2(X_\tau)]^T[f_1(X_\tau) - f_2(X_\tau)] d \tau \right.
    \\
    & \left. \hspace{1.1cm} + \int_0^t[f_1(X_\tau) - f_2(X_\tau)]^TdW_\tau \right).
\end{aligned}
\end{equation*}
Denote the history of continuous-time state until time $t$ generated by \eqref{eq:nonlinear_SDE_1} and \eqref{eq:nonlinear_SDE_2} as $\mathcal{X}_t, \mathcal{Y}_t$ respectively, i.e.,
\begin{equation*}
    \mathcal{X}_t = \{X_\tau|0\leq \tau \leq t\}, \;\; \mathcal{Y}_t = \{Y_\tau|0\leq \tau \leq t\}.
\end{equation*}
For an arbitrary function $w(\cdot)$ of the paths, the expectation over the distribution induced by the Brownian motion has the relation $\mE[w(\mathcal{X}_t)] = \mE[Z_t w(\mathcal{Y}_t)].$

\subsection{Linear Covariance Steering}\label{sec:background_linear_CS}
The linear covariance steering problem is presented in the following form \cite{CheGeoPav15b}
\begin{subequations}\label{eq:lin_cov_steering}
\begin{eqnarray}\label{eq:lin_cov_steering1}
    \min_{u} && \mE \left\{\int_{0}^{T} [\frac{1}{2}\|u_t\|^2 + \frac{1}{2} X_t^T Q_t X_t]dt\right\}
    \\\label{eq:lin_cov_steering2}
    &&\hspace{-0.8cm} dX_t = A_t X_t dt + a_t dt + B_t (u_t dt + \sqrt{\epsilon} d W_t)
    \\\label{eq:lin_cov_steering3}
    && 
    \hspace{-0.8cm} X_0 \sim \mathcal{N}(\mu_0, K_0),\quad X_T \sim \mathcal{N}(\mu_T, K_T). 
\end{eqnarray}
\end{subequations}
The optimal solution to \eqref{eq:lin_cov_steering} is obtained by controlling the mean and the covariance separately thanks to the linearity. The control for the mean is an optimal control problem
\begin{subequations}
	\begin{eqnarray*}
		&&\min_{v} \int_{0}^{T} [\frac{1}{2}\|v_t\|^2 + \frac{1}{2} x_t^T Q_t x_t]dt
		\\ &&\dot x_t = A_tx_t + B_tv_t, \; \; x_0 = \mu_0, \quad x_T = \mu_T,
	\end{eqnarray*}
	\end{subequations}
where $v_t\in\mR^m$ are the mean control inputs. The control of the covariance is obtained by solving the coupled Riccati equations \cite{CheGeoPav17a}
\begin{subequations}\label{eq:LQschrodinger}
    \begin{eqnarray*}
   	-\dot\Pi_t\!\!\!\!&=&\!\!\!\!A_t^T\Pi_t\!+\!\Pi_tA_t\!-\!\Pi_tB_tB_t^T\Pi_t \!+\!Q_t\label{eq:LQschrodinger1}
    \\
    -\dot\HH_t\!\!\!\!&=&\!\!\!\!A_t^T\HH_t\!+\!\HH_tA_t\!+\!\HH_tB_tB_t^T\HH_t \!-\!Q_t\label{eq:LQschrodinger2}
    \\
    \epsilon K_0^{-1}\!\!\!\!&=&\!\!\!\!\Pi_0+\HH_0\label{eq:LQschrodinger3}, \; \epsilon K_T^{-1} = \Pi_T+\HH_T,
    \end{eqnarray*}
    \end{subequations}
where $\Pi_t\in\mR^{n\times n}$ is the quadratic value function matrix, and $\HH_t\in\mR^{n\times n}$ is an auxiliary matrix. The above equation system has a closed-form solution \cite{CheGeoPav15a}.
Denote the optimal mean control as $v_t^\star$ and the controlled state as $x_t^\star$. Combining with the covariance control yields the optimal feedback policy
\[
    u_t^\star = -B_t^T \Pi_t (X_t-x_t^\star) + v_t^\star.
\]

\section{Stochastic Motion Planning as Gaussian Variational Inference}
\label{sec:GVIMP_SOC_formulation}
This section introduces our formulation of motion planning as Gaussian variational inference. We then show its equivalence to a stochastic control problem. 
\subsection{Gaussian Variational Inference Motion Planning}
\label{sec:GVIMP_SOC_formulation_GVIMP}
Our formulation builds on the important posterior probability formulation \eqref{eq:MP_posterior} of an optimal motion plan. GVI-MP seeks to find the optimal distribution solution by solving
\begin{equation}
\begin{split}
    q^{\star} &= \underset{q\in \mathcal{Q}}{\arg\min}\; {\rm KL} \left ( q(\bX) \parallel p(\bX|Z) \right )
    \\
    &= \underset{q\in \mathcal{Q}}{\arg\min}\; \mE_q{ \left[\log q(\bX)-\log p(Z|\bX)-\log p(\bX) \right]}
    \\
    &= \underset{q\in \mathcal{Q}}{\arg\max}\; \mE_q\left[\log p(Z|\bX)\right] - {\rm KL} \left( q(\bX) \parallel p(\bX) \right)
    \\
    &= \underset{\mu_{\theta}, \Sigma_{\theta}}{\arg\min}\; \mE_{q\sim \mathcal{N}(\mu_{\theta}, \Sigma_{\theta})}{\left[\lVert \mathbf{h}(\bX)\rVert_{\Sigma_{\rm obs}}^2\right]} + {\rm KL} \left (q \parallel \mathcal{N}(\bmu, \bK)\right)
\end{split}
\label{eq:GVI_formulation}
\end{equation}
where $\mathcal{Q} \triangleq \{ q_\theta: q_\theta \sim \mathcal{N}(\mu_{\theta}, \Sigma_{\theta}) \}$ consists of the parameterized proposal Gaussian distributions.
\eqref{eq:GVI_formulation} formulates the probability inference for an optimal robot trajectory as an optimization within the Gaussian distribution space, where the \hongzhe{objective} function is a distributional distance measure, the Kullback–Leibler (KL) divergence \cite{van2014renyi}, between the proposal Gaussian and the target posterior. This formulation strategically connects the extensive methodologies of variational inference with the problem of motion planning in robotics.

In \eqref{eq:GVI_formulation}, the prior $p(\bX)$ is in the form \eqref{eq:prior_factor} with sparsity in $\bK^{-1}$ consisting of \eqref{eq:sparse_G} and \eqref{eq:sparse_Q}. The likelihood is in the form \eqref{eq:collision_factor}. The output of the program \eqref{eq:GVI_formulation} is a Gaussian distribution over the joint sparse support states $\bX$. Once $q^\star$ is computed, the continuous-time Gaussian trajectory distribution can be obtained efficiently (in $O(1)$ complexity) using GP interpolation \cite{barfoot2014batch, mukadam2018}.

GVI-MP \eqref{eq:GVI_formulation} is an optimization over the distribution of the discretized state $X_t$, which is seemingly distant from the motion planning problem formulation \eqref{eq:trajectory_optimization_formulation} where the optimization variables include state $X_t$ and control $u_t$, and system dynamics appears as a constraint in the formulation. However, a deeper insight reveals that the two problems are equivalent. Next, we provide a theoretical analysis of this duality. 

\subsection{GPMP as Minimum-energy Feedback Motion Planning}
\label{sec:GVIMP_SOC_formulation_MAP_minimum_energy}
The formulation \eqref{eq:MAP_formulation} is an optimization over $\bX$ for finding the MAP estimation of $p(\bX|Z)$. It is equivalent to a minimum-energy tracking problem.
\begin{proposition}\label{thm:lemma1}
The solution to the following problem is the MAP solution to the probabilistic inference \eqref{eq:MAP_formulation} over the same time discretization \eqref{eq:time_discretization}. 
\begin{subequations}\label{eq:OCP_formulation_MAP}
\begin{eqnarray}
    \min_{X(\cdot), u(\cdot)}\!\!\!\!\!\!\!&& \int_{t_0}^{t_N} \frac{1}{2}\|u_t\|^2  + \lVert \mathbf{h}(X_t) \rVert_{\Sigma_{\rm obs}}^2 dt \nonumber
    \\&&\hspace{-0.3cm}+\frac{1}{2}\lVert X_0 - \mu_0 \rVert_{K_0^{-1}}^2 + \frac{1}{2} \lVert X_N - \mu_N \rVert_{K_N^{-1}}^2
    \\
    {\rm s.t.} &&\hspace{-0.1cm}  \dot {X}_t = A_t X_t + a_t + B_t u_t. \label{eq:deviated_dyn}
\end{eqnarray}
\end{subequations}
\end{proposition}

The proof can be found in Appendix \ref{sec:proof_lemma_OCP_MAP}. The equivalence in Proposition \ref{thm:lemma1} is thanks to \eqref{eq:OCP_formulation_MAP} \hongzhe{admitting} a closed-form solution whose gain is related to the Grammian \eqref{eq:sparse_Q}. From formulation \eqref{eq:OCP_formulation_MAP}, we know that the direct MAP estimation on the states is indeed a feedback motion planning \cite{lavalle1998} around the nominal $\mu_t$. The initial and final marginal constraints are enforced by boundary costs $\lVert X_0 - \mu_0 \rVert_{K^{-1}_0}^2$ and $\lVert X_N - \mu_N \rVert_{K^{-1}_N}^2$. 

\subsection{Stochastic Motion Planning as Variational Inference}
\label{sec:GVIMP_SOC_formulation_GVI_stochastic_control}
\begin{figure}[t]
    \centering
    \includegraphics[width=0.95\linewidth]{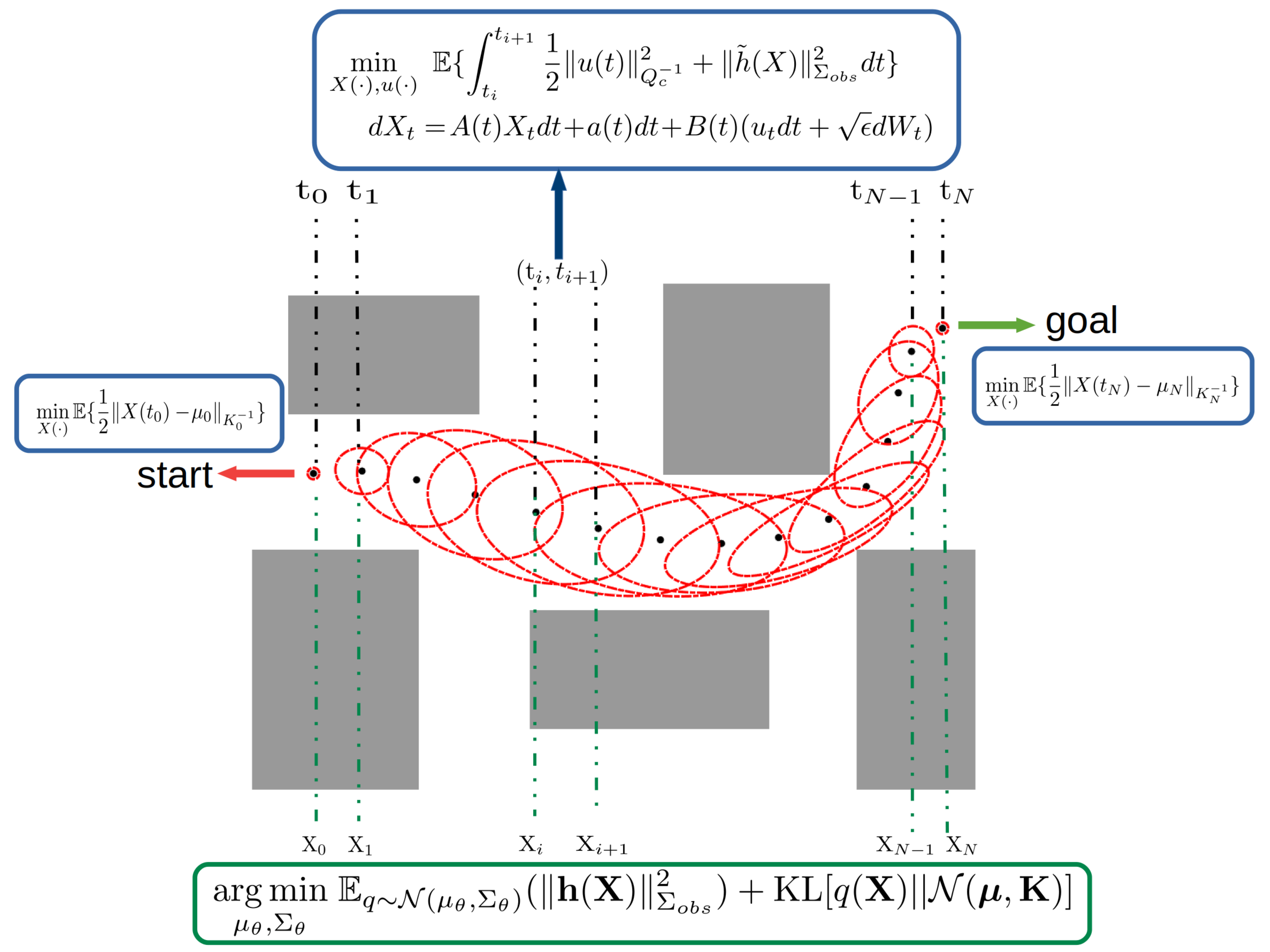}
    \caption{A demonstration of the connection between Gaussian variational inference motion planning (GVI-MP) and stochastic optimal control problems over the same time discretization in a point robot planning example. 
    }
    \label{fig:duality_SOC_GVI}
\end{figure}
To handle uncertainties, stochastic motion planning injects noise into the formulation. Consider the \textit{controlled} linear stochastic process 
\begin{equation}\label{eq:dynamics_controlled}
 d X_t =\! A_tX_t dt \! +\! a_t dt \!+\! B_t(u_t dt + dW_t),
\end{equation} which has a control $u_t$ on top of the prior GP prior \eqref{eq:dynamics_uncontrolled}, and has stochasticity on top of \eqref{eq:deviated_dyn}. This injected stochasticity results in a controlled trajectory distribution, and the optimality index in \eqref{eq:MAP_formulation} will be lifted into a statistical index, such as the expected cost, in the new formulation. Here, we assume that the noise comes into the system via the same channel as the control input. The following theorem shows the duality between GVI-MP and a stochastic control problem.
\begin{figure*}
\centering
    \begin{subfigure}[b]{0.25\textwidth}
    \centering
    \includegraphics[width=\textwidth]{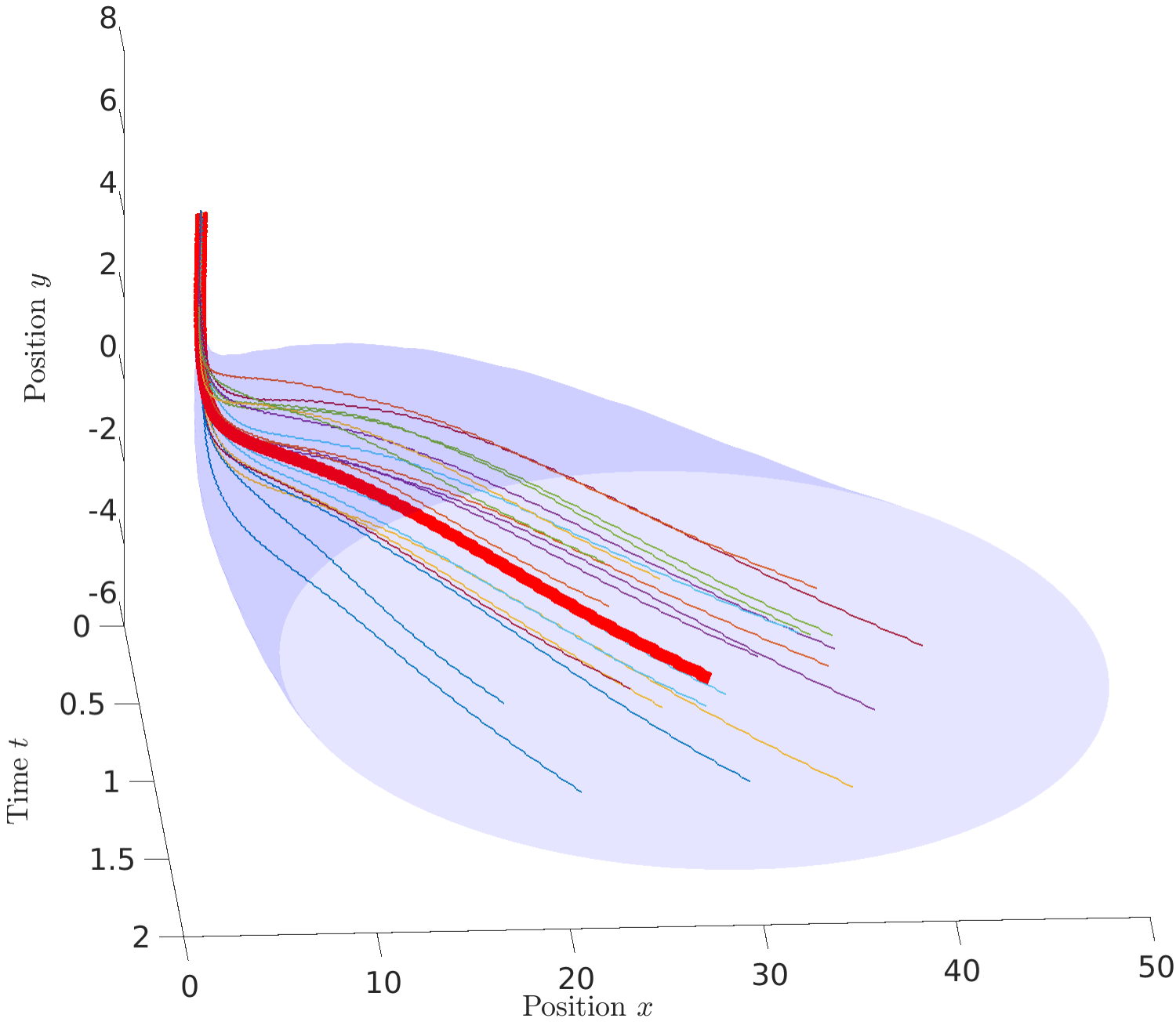}
    \caption{Uncontrolled path distribution}
    \label{fig:LTVSDE-uncontrolled}
    \end{subfigure}
    \hfill
    \begin{subfigure}[b]{0.25\textwidth}
    \centering
    \includegraphics[width=\textwidth]{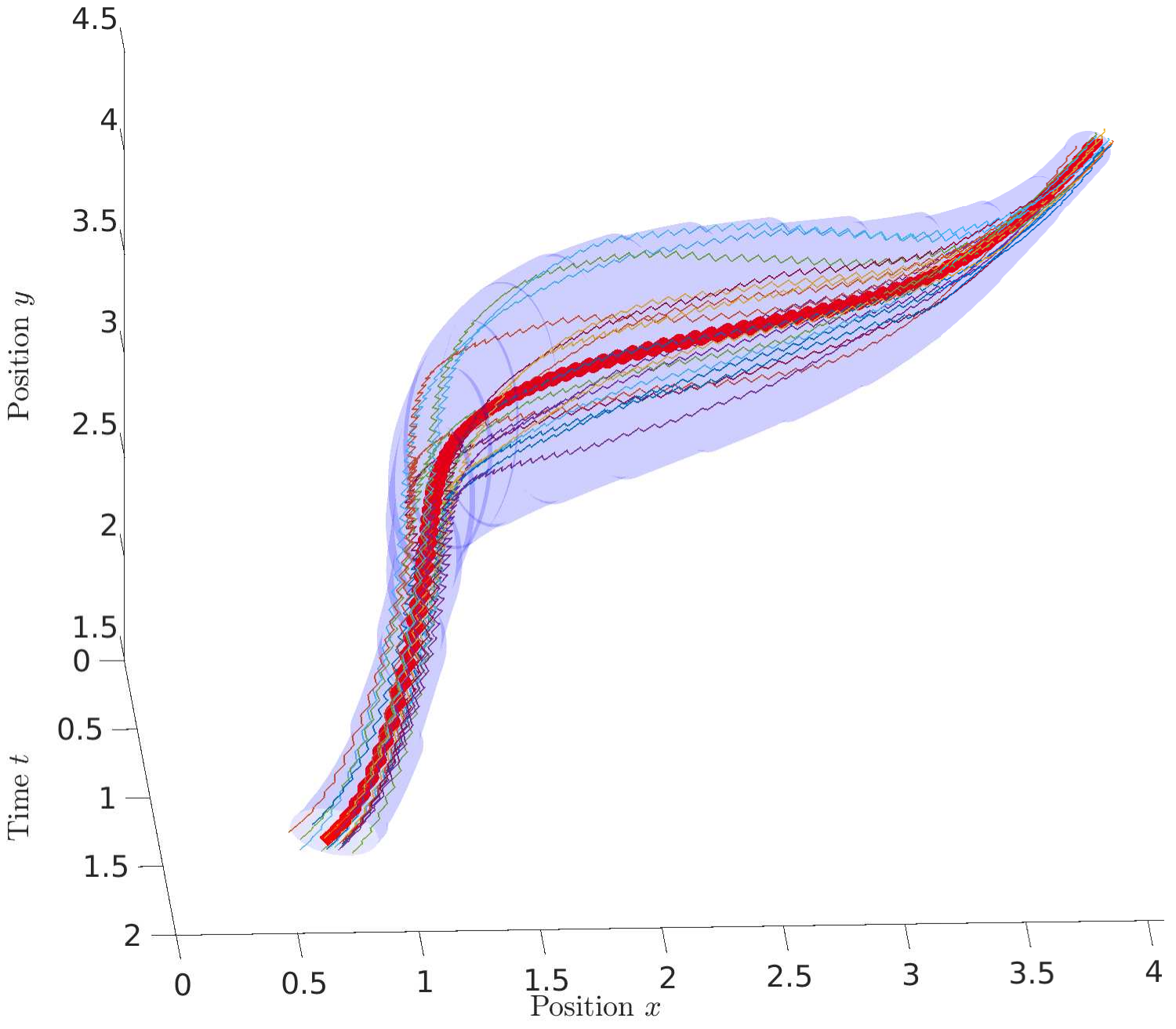}
    \caption{Controlled path distribution}
         \label{fig:LTVSDE-controlled}
    \end{subfigure}
    \hfill
    \begin{subfigure}[b]{0.25\textwidth}
    \centering
    \includegraphics[width=\textwidth]{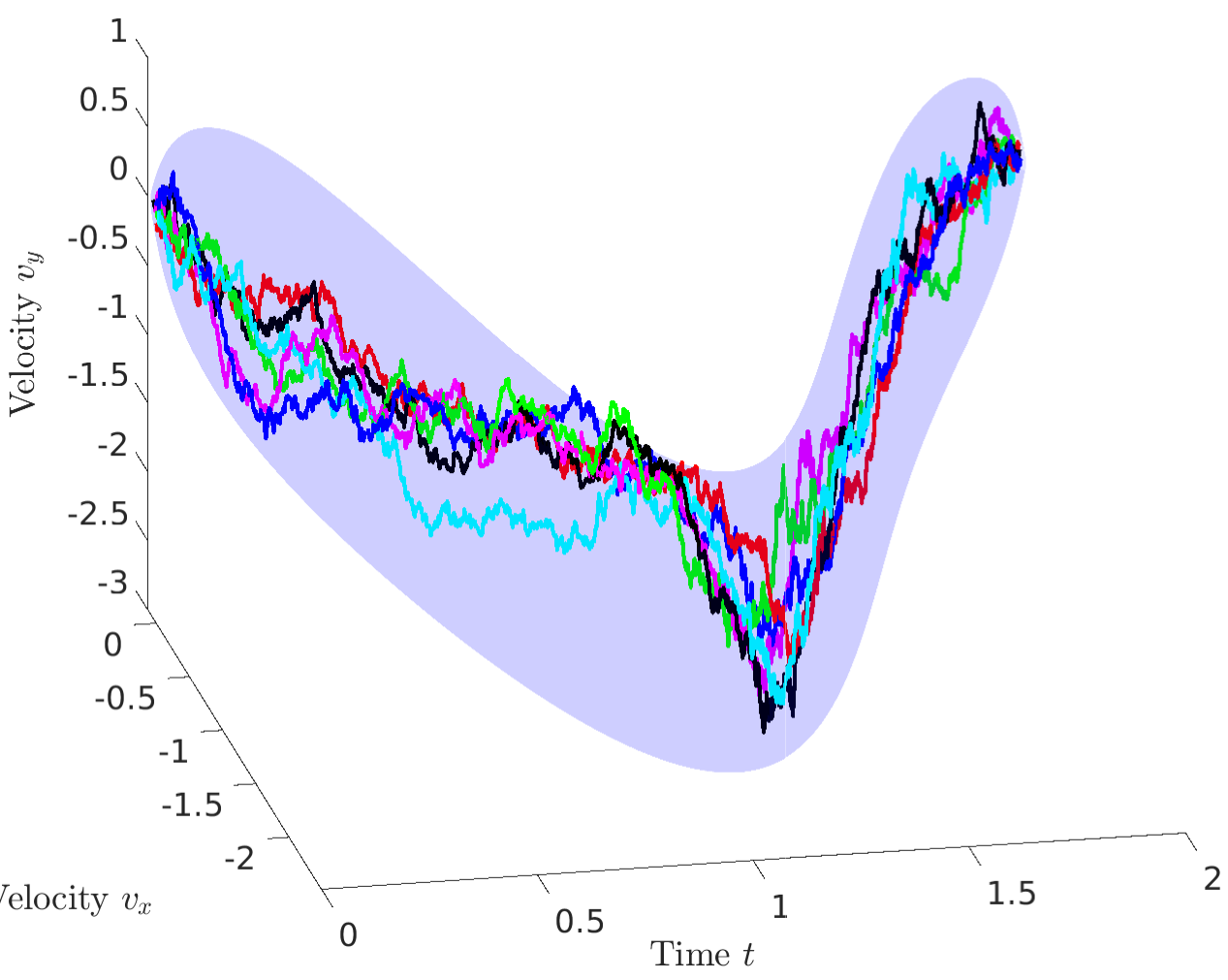}
    \caption{Controlled velocity distribution}
         \label{fig:LTVSDE-input}
    \end{subfigure}
  
  \caption{Trajectory distributions of uncontrolled \eqref{eq:dynamics_uncontrolled} and controlled \eqref{eq:dynamics_controlled} LTV-SDE, and sampled velocities. The system is obtained by linearizing a planar quadrotor (Eq. \eqref{eq:planar_quad}). The mean trajectory is in red, and the blue shade represents the $3-\sigma$ covariance contour. The Girsanov theorem states that the KL divergence between the two processes equals the expected control energy.}
  \label{fig:KL-LTV-SDE}
\end{figure*}

\begin{thm}[Main Theorem]\label{thm:lemma_connection_GVI_SOC}
    The solution to the following problem is the solution to the Gaussian Variational Inference problem \eqref{eq:GVI_formulation} over the same time discretization \eqref{eq:time_discretization}.
\begin{subequations}\label{eq:SCP_formulation_GVI}
\begin{eqnarray}\label{eq:SCP_formulation_GVI_1}
    \min_{X(\cdot), u(\cdot)}\!\!\!\!\!\!\!\!\!&& \mE \left \{\int_{t_0}^{t_N} \frac{1}{2}\|u_t\|^2 dt + \lVert \bh(\bX) \rVert_{\Sigma_{\rm obs}}^2 \right.
    \\
    && \hspace{-0.2cm} \nonumber \left.
    + \frac{1}{2}\lVert X_0 -\! \mu_0 \rVert_{K_0^{-1}} 
    + \frac{1}{2}\lVert X_N - \mu_N \rVert_{K_N^{-1}} 
    \right \} 
    \\
    \label{eq:SCP_formulation_GVI_2} 
    {\rm s.t.} && \hspace{-0.3cm} d X_t =\! A_tX_t dt \! +\! a_t dt \!+\! B_t(u_t dt + dW_t).
\end{eqnarray}
\end{subequations}
\end{thm}
\begin{proof}
\label{sec:proof_lemma_GVI_SOC}
For the prior SDE \eqref{eq:dynamics_uncontrolled}, denote the measure of the induced continuous-time path as $d \cP^0$. Similarly, denote the measure of the path induced by the controlled process \eqref{eq:dynamics_controlled} as $d \cP^u$. Girsanov theorem states that the density ratio
\begin{equation*}
    \frac{d \cP^u}{d \cP^0} = \exp\left(\frac{1}{2} \int_{t_0}^{t_N} \lVert B_tu_t \rVert_{(B_tB_t^T)^\dagger}^2 d t + 
    \int_{t_0}^{t_N} u_t^Td W_t \right),
\end{equation*}
from which we have
\begin{equation*}
    {\rm{KL}}\left( \cP^u \parallel \cP^0 \right) = \int \log \frac{d\cP^u}{d \cP^0} d\cP^u = \frac{1}{2} \mE \left \{\int_{t_0}^{t_N}\lVert u_t \rVert^2 dt \right \}
\end{equation*}
for a full column rank $B_t$, due to the fact that $\mE[w(\cdot)dW_t] = 0, \forall w(\cdot),$ and $\lVert B_t u_t \rVert_{(B_tB_t^T)^\dagger}^2 = \lVert u_t \rVert^2$. We can thus rewrite the objective \eqref{eq:SCP_formulation_GVI} into a distribution control problem 
\begin{subequations}
    \begin{eqnarray*}
    \min_{\cP^u} \!\!\!\!&&\!\!\!\!\int \left [ \log \frac{d\cP^u}{d \cP^0}  +  \lVert \bh(\bX) \rVert_{\Sigma_{\rm obs}}^2 + \frac{1}{2} \lVert X_0 -\! \mu_0 \rVert_{K_0^{-1}} \right.
    \\ 
    && \left. + \frac{1}{2}\lVert X_N - \mu_N \rVert_{K_N^{-1}}^2 \right] d\cP^u,
\end{eqnarray*}
\end{subequations}
which by definition of the KL-divergence, is equivalent to
\begin{equation} \label{eq:objective_SOC_KL}
\begin{split}
    \min_{\cP^u}\;\;\;\;& \mE \left \{ \lVert \bh(\bX) \rVert_{\Sigma_{\rm obs}}^2 +  \frac{1}{2} \lVert X_0 -\! \mu_0 \rVert_{K_0^{-1}} \right. 
    \\ & \left. + \frac{1}{2}\lVert X_N - \mu_N \rVert_{K_N^{-1}}^2 \right \}  + {\rm{KL}}\left ( \cP^u \parallel  \cP^0 \right).
\end{split}
\end{equation}
Take time discretization $\mathbf{t} = [t_0, \dots, t_N]$, the prior probability 
\begin{equation}\label{eq:discretized_prior}
    d \cP^0 \approx \hat{p}(\bX) \triangleq p(X_1|X_0)\dots p(X_N|X_{N-1}).
\end{equation}
The right hand side of \eqref{eq:discretized_prior} is a Gaussian distribution $\hat{p}(\bX) \sim \mathcal{N}(\hat{\bmu}, \hat{\bK})$, and $\hat{\bK}^{-1} = \hat{G}^T\hat{Q}^{-1}\hat{G}$, with \cite{barfoot2014batch}
\begin{equation*}
    \hat{G} = \begin{bmatrix}
    -\Phi(t_1, t_0) & I &   &   & \\
     &  &  \dots   &  &\\
     &  &  & -\Phi(t_N, t_{N-1}) &  I
\end{bmatrix},
\end{equation*} 
and $\hat{Q}^{-1} = {\rm diag}(Q_{0,1}^{-1},\ldots, Q_{N-1,N}^{-1}),$
where $Q_{i, i+1}$ is the Grammian \eqref{eq:def_Grammian}. The difference between $\hat{p}(\bX)$ and $p(\bX)$ lies in the initial and terminal conditions $\mathcal{N}(\mu_0, K_0)$ and $\mathcal{N}(\mu_N, K_N)$. Adding boundary cost factors 
\begin{equation*}
    \frac{1}{2}\mE \left [ \lVert X_0 - \mu_0 \rVert_{K_0^{-1}}^2 \right ], \;\; \frac{1}{2}\mE \left [\lVert X_N - \mu_N \rVert_{K_N^{-1}}^2 \right ],
\end{equation*}
the joint distribution
\begin{equation*}
    \mE \left \{ - \log d \cP^0 + \frac{1}{2}\lVert X_0 - \mu_0 \rVert_{K_0^{-1}}^2 + \frac{1}{2}\lVert X_N - \mu_N \rVert_{K_N^{-1}}^2 \right \} 
    \\ 
\end{equation*}
coincides with \eqref{eq:prior_factor}. The joint probability of the discrete support states induced by controlled process $d \cP^u$ is also a Gaussian $d \cP^u \approx \hat{q}(\bX)$. We thus arrive at the discrete approximation of the objective \eqref{eq:objective_SOC_KL} as
\begin{equation*}
    \min_{\hat{q}\in \mathcal{\hat{Q}}} \mE \left [ \lVert \bh(\bX) \rVert_{\Sigma_{\rm obs}}^2 \right] + {\rm{KL}}\left( \hat{q}({\bX)} \parallel \mathcal{N}(\bmu, \bK) \right ),
\end{equation*}
which coincide with the last line in GVI-MP formulation \eqref{eq:GVI_formulation}.
\end{proof}
Theorem \ref{thm:lemma_connection_GVI_SOC} shows that GVI-MP is equivalent to a stochastic control problem. The nominal in \eqref{eq:SCP_formulation_GVI} is the stochastic prior \eqref{eq:dynamics_uncontrolled}, and the controlled dynamics also considers uncertainty. The objective in \eqref{eq:OCP_formulation_MAP} is equivalent to minimizing the weighted distance of the next controlled support state to the reachable states resulting from passive dynamics
\begin{equation*}
    \lVert X_{t_{i+1}} - \mu_{t_{i+1}} - \Phi(t_{i+1}, t_i)(X_i-\mu_{t_i}) \rVert_{Q_{i,i+1}^{-1}}^2,
\end{equation*}
while in \eqref{eq:SCP_formulation_GVI}, the minimization of the expected control equals minimizing the distance between the two path distributions induced by \eqref{eq:dynamics_controlled} and \eqref{eq:dynamics_uncontrolled}, measured by ${\rm{KL}}\left( \cP^u \parallel  \cP^0\right)$. 
Fig. \ref{fig:duality_SOC_GVI} visualizes the connections via an example on a 2-DOF point robot, and Fig. \ref{fig:KL-LTV-SDE} showcases an LTV system example. 

\section{GVI-MP algorithm}
\label{sec:GVIMP_algorithm}
This section presents the Gaussian Variational Inference Motion Planner (GVI-MP) algorithm.
\subsection{Natural Gradient Descent Optimization Scheme.}
\label{sec:GVIMP_algorithm_NGD}
We propose a natural gradient numerical scheme for solving \eqref{eq:GVI_formulation}. Define the negative log probability for the posterior as $\psi(\bX) \triangleq -\log p(\bX|Z)$. The objective in \eqref{eq:GVI_formulation} reads 
\begin{equation*}
    \mathcal{J}(q) = {\rm KL}\left( q(\bX) \parallel p(\bX|Z) \right) = \mE[\psi(\bX)] - \mathcal{H}(q),
\end{equation*}
where $\mathcal{H}(q)$ denotes the entropy of the distribution $q$.
The sparsity of the Gaussian prior leads to a sparse structure in the inverse of the covariance matrix. We parameterize the proposed Gaussian using its mean $\mu_\theta$ and inverse of covariance $\Sigma^{-1}_\theta$. The derivatives with respect to $\mu_{\theta}$ and $\Sigma^{-1}_{\theta}$ are \cite{opper2009}
\begin{subequations}
    \begin{align}
        \!\!\!\!\!\frac{\partial \mathcal{J}(q)}{\partial \mu_{\theta}} &= \Sigma^{-1}_{\theta}\mE \left[(\bX-\mu_{\theta})\psi \right]
        \label{eq:GVI_derivatives1}
        \\
        \!\!\!\!\!\frac{\partial^2\mathcal{J}(q)}{\partial \mu_{\theta} \partial \mu_{\theta}^T} &= \Sigma^{-1}_{\theta}\mE\left[(\bX-\!\mu_{\theta})(\bX\!-\!\mu_{\theta})^T\psi\right]\Sigma^{-1}_{\theta} \!- \!\!\Sigma^{-1}_{\theta}\mE\left[\psi\right]
        \label{eq:GVI_derivatives2}
        \\
        \!\!\!\!\!\frac{\partial \mathcal{J}(q)}{\partial \Sigma^{-1}_{\theta}} &= \frac{1}{2}\Sigma_{\theta}\mE\left[\psi\right]\!-\!\frac{1}{2} \mE\left[(\bX-\!\mu_{\theta})(\bX-\!\mu_{\theta})^T\psi\right]  + \frac{1}{2}\Sigma_{\theta}.
        \label{eq:GVI_derivatives3}
    \end{align}
    \label{eq:GVI_derivatives}
\end{subequations}
Comparing \eqref{eq:GVI_derivatives2} and \eqref{eq:GVI_derivatives3} we obtain
\begin{equation}
\frac{\partial^2\mathcal{J}(q)}{\partial \mu_{\theta} \partial \mu_{\theta}^T} = \Sigma^{-1}_{\theta} - 2\Sigma^{-1}_{\theta}\frac{\partial \mathcal{J}(q)}{\partial \Sigma^{-1}_{\theta}} \Sigma^{-1}_{\theta},
\label{eq:GVI_equivalent_derivatives}
\end{equation}
which makes it unneccessary to evaluate both \eqref{eq:GVI_derivatives2} and \eqref{eq:GVI_derivatives3}. Natural gradient descent update w.r.t. objective $\mathcal{J}$ reads \cite{magnus2019} 
\begin{equation*}
    \begin{bmatrix}
    \delta \mu_{\theta}\\
    vec(\delta \Sigma^{-1}_{\theta})
    \end{bmatrix}
    = -\begin{bmatrix}
    \Sigma_{\theta} & 0 \\
    0 & 2(\Sigma^{-1}_{\theta} \otimes \Sigma^{-1}_{\theta})
    \end{bmatrix}
    \begin{bmatrix}
    \frac{\partial \mathcal{J}(q)}{\partial \mu_{\theta}}\\
    vec(\frac{\partial \mathcal{J}(q)}{\partial \Sigma^{-1}_{\theta}})
    \end{bmatrix},
\end{equation*}
which in matrix form is written as
\begin{equation}
    \Sigma^{-1}_{\theta} \delta \mu_{\theta} = -\frac{\partial \mathcal{J}(q)}{\partial \mu_{\theta}},\;\;\; \delta \Sigma^{-1}_{\theta} = -2\Sigma^{-1}_{\theta}\frac{\partial \mathcal{J}(q)}{\partial\Sigma^{-1}_{\theta}}\Sigma^{-1}_{\theta}.
\label{eq:GVI_delta_mu}
\end{equation}
Comparing \eqref{eq:GVI_derivatives} and \eqref{eq:GVI_delta_mu}, we have
\begin{equation}
    \begin{aligned}
    \delta \Sigma^{-1}_{\theta} &= \frac{\partial^2\mathcal{J}(q)}{\partial \mu_{\theta} \partial \mu_{\theta}^T} - \Sigma^{-1}_{\theta}.
\end{aligned}
\label{eq:GVI_delta_Sigma}
\end{equation}
Equation \eqref{eq:GVI_delta_mu} and \eqref{eq:GVI_delta_Sigma} tells that, to calculate the update $\delta \mu_{\theta}, \delta \Sigma^{-1}_{\theta}$, we only need to compute the expectations \eqref{eq:GVI_derivatives1} and \eqref{eq:GVI_derivatives2}. The updates use a step size $\eta < 1$ 
\begin{equation}
    \mu_{\theta} \leftarrow \mu_{\theta} + \eta^R \, \delta \mu_{\theta}, \;\;\;  \Sigma^{-1}_{\theta} \leftarrow \Sigma^{-1}_{\theta} + \eta^R \, \delta \Sigma^{-1}_{\theta},
\label{eq:GVI_backtracking}
\end{equation}
with an increasing $R=1,2,\ldots$ to shrink the step size for backtracking until the cost decreases. 

\subsection{Posterior Structures in Motion Planning Inference Tasks}
\label{sec:GVIMP_algorithm_factorization}
Natural gradient descent updates necessitate computing the expectations \eqref{eq:GVI_derivatives} for the proposed Gaussian $q$. Leveraging the structure of the posterior, we streamline these computations. By the definition \eqref{eq:MAP_formulation}, we express $\psi(\bX)$ as 
\begin{equation*}
    \psi(\bX) = - \log p(\bX|Z) = \lVert \mathbf{h}(\bX) \rVert_{\Sigma_{\rm obs}}^2 + \frac{1}{2}\lVert \bX - \bmu \rVert_{\bK^{-1}}^2,
\end{equation*}
which consists of a quadratic $\psi_{\rm{Prior}}$ for the prior and a nonlinear factor $\psi_{\rm{NL}}$ for the collision checking defined as
\begin{subequations}
\label{eq:GVI_define_linear_NL_factors}
\begin{eqnarray}
    \psi_{\rm{Prior}}(\bX) \!\!\!&=&\!\!\! \frac{1}{2}\lVert \bX - \bmu \rVert_{\bK^{-1}}^2, 
    \label{eq:defn_prior_factor}
    \\
    \psi_{\rm{NL}}(\bX) \!\!\!&=&\!\!\! \lVert \mathbf{h}(\bX) \rVert_{\Sigma_{\rm obs}}^2.
    \label{eq:defn_nonlinear_factor}
\end{eqnarray}
\end{subequations}

{\em (a) The motion prior factors.} The factor $\psi_{\rm Prior}$ contains a sparse precision matrix $\bK^{-1}$ \eqref{eq:sparse_G}. Denote the discretization of the deviated state \eqref{eq:tilde_Xt} as $ \Tilde{X}_j = \Tilde{X}_{t_j}, \;\; j = 0, \dots, N$, and adopt the state transition notation \eqref{eq:state_transition_discrete}, we have
\begin{equation*}
\begin{split}
    \psi_{\rm{Prior}}(\bX) &\propto \mE \{(\bX - \bmu)^TG^TQ^{-1}G(\bX-\bmu)\}
\\
    & = \mE \{\lVert [\Tilde{X}_0^T, (\Tilde{X}_1 - \Phi_{1,0}\Tilde{X}_0)^T, \dots, 
\\
    &\hspace{0.2cm}(\Tilde{X}_N - \Phi_{N,N-1} \Tilde{X}_{N-1})^T, \Tilde{X}_N^T]^T \rVert_{Q^{-1}}^2\},
\end{split}
\end{equation*}
which can be further decomposed into 
\begin{equation*}
    \mathcal{J}_0(q_0) + \mathcal{J}_{0,1}(q_{0,1}) + \dots + \mathcal{J}_N(q_N)
\end{equation*}
by defining the factorized cost functionals
\begin{subequations}
\label{eq:prior_factors}
\begin{eqnarray}
    \mathcal{J}_0 \!\!\!&\triangleq&\!\!\! \mE_{q_0} \left[ \lVert \Tilde{X}_0 \rVert_{{K^{-1}_0}}^2 \right]
    \label{eq:prior_factors_1}
\\
    \mathcal{J}_{i,i+1} \!\!\!&\triangleq&\!\!\! \mE_{q_{i,i+1}} \left[ \lVert \Tilde{X}_{i+1} - \Phi_{i+1,i} \Tilde{X}_{i} \rVert_{Q^{-1}_{i,i+1}}^2 \right]
    \label{eq:prior_factors_2}
\\
    \mathcal{J}_N \!\!\!&\triangleq&\!\!\! \mE_{q_N} \left[ \lVert \Tilde{X}_N \rVert_{{K^{-1}_N}}^2 \right].
    \label{eq:prior_factors_3}
\end{eqnarray}
\end{subequations}
$q_0$ and $q_N$ are the marginal Gaussian distributions of the first and last supported states, and $q_{i,i+1}$ stand for marginal Gaussian distributions of two adjacent states at $t_i$ and $t_{i+1}$.

The factors $\mathcal{J}_{i,i+1}(q_{i,i+1})$ promote the underlying uncontrolled dynamics on consecutive states, and the two isolate factors $\mathcal{J}_0$ and $\mathcal{J}_N$ enforce the initial and terminal conditions. 
Notice that because of the underlying linear dynamics, both $\mathcal{J}_{i,i+1}, \mathcal{J}_0$, and $\mathcal{J}_N$ have quadratic functions inside the expectation. The following Lemma \ref{lem:closed_form_prior} shows that the derivatives \eqref{eq:GVI_derivatives1} and \eqref{eq:GVI_derivatives2} for prior factors \eqref{eq:prior_factors} enjoy a closed-form. 
\begin{lemma}
    \label{lem:closed_form_prior}
    For the Gaussian $\mathbf{X}$ and for given matrices $\Lambda, \Psi$ with right dimensions, define $y \triangleq \bX-\mu_{\theta}$, the derivatives of $ \mathcal{J}_{\Lambda, \Psi} \triangleq \mE[\lVert \Lambda\bX - \Psi\bmu \rVert_{\bK^{-1}}^2]$ with respect to the mean $\mu_{\theta}$ are
\begin{subequations}
    \begin{align}
        \frac{\partial \mathcal{J}_{\Lambda, \Psi}}{\partial \mu_{\theta}} 
         &= 2\Lambda^T\bK^{-1}(\Lambda\mu_\theta - \Psi\bmu). 
        \label{eq:GVI_derivatives_prior1}
        \\
        \frac{\partial^2\mathcal{J}_{\Lambda, \Psi}}{\partial \mu_{\theta} \partial \mu_{\theta}^T} 
        &=\Sigma^{-1}_{\theta}\mE\left[yy^Ty^T\Lambda^T\bK^{-1}\Lambda y \right] \Sigma^{-1}_{\theta} \nonumber
        \\
        & \;\;\;\; - \Sigma^{-1}_{\theta}\tr \left[\Lambda^T\bK^{-1} \Lambda \Sigma_{\theta}\right].
        \label{eq:GVI_derivatives_prior2}
    \end{align}
    \label{eq:GVI_derivatives_prior}
\end{subequations}
\end{lemma}
The proof can be found in Appendix \ref{sec:proof_lemma_closedform_prior}. The closed-form derivatives for the motion prior can be obtained by letting
\begin{equation*}
    \Lambda = \Psi = \left[-\Phi_{i+1, i}, ~ I \right],
\end{equation*}
for $\mathcal{J}_{i,i+1}$, and letting $\Lambda = I, ~ \Psi = 0$ for $\mathcal{J}_0, ~ \mathcal{J}_N$, where the variables are the Gaussian marginals $q_0, q_{i,i+1}$, and $q_N$, respectively. The close forms greatly accelerate the algorithm compared with numerical estimations for the prior factors.

{\em (b) The collision likelihood factors. $\psi_{\rm{NL}}$} are evaluated on each support state, which forms a series of isolated factors 
\begin{equation}\label{eq:factor_collision}
    \psi_{\rm{NL}}(\bX) =  \lVert \bh(\bX) \rVert_{\Sigma_{\rm obs}}^2 = \sum_{i=1}^{N-1} \lVert \bh(X_i) \rVert_{\Sigma_{{\rm obs}_i}}^2.
\end{equation}
Computing expectations for Gaussian distributions over nonlinear factors has been extensively explored within the Gaussian filtering literature \cite{arasaratnam2007discrete, ito2000gaussian}. A prominent numerical technique employed is the Gauss-Hermite (GH) quadrature \cite{liu1994}, which leads to the Gauss-Hermite Filter (GHF) \cite{ito2000gaussian}. Quadrature approximation is precise and efficient for univariate integration estimation. However, the direct tensor product quadrature rule for multivariate functions has an exponential dependence on the dimension of the variables \cite{heiss2008likelihood}. The state-of-the-art sparse-grid quadrature is developed on the \textit{Smolyak rule} \cite{bungartz2004sparse}, which ignores the cross terms in different degrees of polynomial basis and achieves polynomial time complexity (c.f. Lemma \ref{lem:sparseGH_complexity} in Appendix \ref{sec:GH_quadratures}).

\subsection{Factorization and Marginal Updates}
\label{sec:GVIMP_algorithm_factorized_optimization}
The closed-form \eqref{eq:GVI_derivatives_prior} and the sparse quadratures in Section \ref{sec:GVIMP_algorithm_factorization}-(b) reduce the computation complexity. However, joint-level computations for the nonlinear factors are still demanding. Thanks to the factor graph \eqref{eq:prior_factors} and \eqref{eq:factor_collision}, we decompose the computations to marginal levels. The log probability of the posterior $\psi(\bX)$ can be factorized as
\begin{equation}
\label{eq:factorized_psi}
    \psi(\bX) = \sum_{l=1}^L \psi_\ell(\bX) \triangleq \sum_{l=1}^L \left[-\log p\left(\bX_\ell|Z\right) \right],
\end{equation}
where factor-level Gaussian variables are linearly mapped to and from the joint variables. Specifically, for
\[
q_i\sim \mathcal{N}(\mu_{\theta_i}, \Sigma_{\theta_i}), \; q_{i,i+1}\sim \mathcal{N}(\mu_{\theta_{i,i+1}}, \Sigma_{\theta_{i,i+1}}),
\]
we have 
\begin{equation*}
    \mu_{\theta_i} = M_i \mu_\theta, \; \mu_{\theta_{i,i+1}} = M_{i,i+1} \mu_\theta
\end{equation*}
and 
\begin{equation*}
    \Sigma_{\theta_i} = M_i \Sigma_\theta M_i^T, \; \Sigma_{\theta_{i,i+1}} = M_{i,i+1} \Sigma_\theta M_{i,i+1}^T
\end{equation*}
where 
\[
M_i \triangleq \begin{bmatrix}
    0& \dots & 0 & \overbrace{I_n}^{i} & 0 & \dots & 0
\end{bmatrix}
\]
and 
\[
M_{i,i+1} \triangleq \begin{bmatrix}
0   & \dots & \overbrace{I_n}^{i} & 0 & \dots & 0
\\
0   & \dots & 0 & \overbrace{I_n}^{i+1} & \dots & 0
\end{bmatrix},
\]
$I_n$ is the identity matrix of the state dimension $n$. Concatenating all the above marginals, assuming $L$ in total, into a collection $\{q_\ell\}_{\ell=1}^L$, we write
\begin{equation}\label{eq:factor_map}
    q_\ell \sim \mathcal{N}(\mu_\theta^\ell, \Sigma_\theta^\ell), \;\;
    \mu_\theta^\ell = M_{\ell} \mu_\theta, \;\;
    \Sigma_\theta^\ell = M_{\ell} \Sigma_\theta M_{\ell}^T.
\end{equation}
Given \eqref{eq:factor_map}, the gradient updates \eqref{eq:GVI_derivatives1} and \eqref{eq:GVI_derivatives2} can be evaluated on factor levels, then mapped to the joint level as
\begin{subequations}
\label{eq:update_factor_to_joint}
\begin{eqnarray}
    \frac{\partial \mathcal{J}(q)}{\partial \mu_\theta} \!\!\!&=&\!\!\! \sum_{l=1}^L M_{\ell}^T \frac{\partial \mathcal{J}_\ell(q_\ell)}{\partial \mu_{\theta}^\ell}
    \\
    \frac{\partial^2 \mathcal{J}(q)}{\partial \mu_\theta \partial \mu_\theta^T} \!\!\!&=&\!\!\! \sum_{l=1}^L M_{\ell}^T \frac{\partial^2 \mathcal{J}_\ell(q_\ell)}{\partial \mu_\theta^{l} (\partial \mu_\theta^{l})^T} M_{\ell},
\end{eqnarray}
\end{subequations}
where $\mathcal{J}_\ell(q_\ell) \triangleq \mE_{q_\ell}[\psi_\ell(\bX)]$ represents the $l{~ \rm th}$ factor objective corresponding to marginal Gaussian distribution $q_{\ell}$.

Combining the prior factors \eqref{eq:prior_factors} and the likelihood factors \eqref{eq:factor_collision}, for motion planning involving $N+1$ support states, the number of factors and the mapping $M_\ell$ for each factor is known a-prior. A total of $L=2N+3$ factors need to be computed. Among these, $N+2$ quadratic factors admit closed-form solutions by Lemma \ref{lem:closed_form_prior}, while the remaining $N+1$ nonlinear factors are evaluated via quadratures. The update rules \eqref{eq:GVI_derivatives} stay unchanged on the marginal levels with a much-reduced dimension.

\subsection{Algorithm and Complexity Analysis}
\label{sec:GVIMP_algorithm_complexity_analysis}
{\em (a) \textbf{GVI-MP Algorithm.}}
We are now ready to present the GVI-MP algorithm. It is formalized in the Algorithm \ref{alg:gvimp}. The algorithm entails the following features to maximize efficiency.
\begin{itemize}
    \item Sparse factor graph factorization \eqref{eq:factor_map}
    \item Closed-form estimation for prior factors in Lemma \ref{lem:closed_form_prior}
    \item Sparse Gauss-Hermite quadrature estimation for nonlinear factors in Lemma \ref{lem:sparseGH_complexity}
\end{itemize}

{
\setlength{\algomargin}{1.55em}
\begin{algorithm}[h]
    \caption{Gaussian Variational Inference Motion Planner (GVI-MP)}
    \label{alg:gvimp}
    \DontPrintSemicolon

    \SetKwInOut{Input}{input}\SetKwInOut{Output}{output}
    \Input{Initial proposal $(\mu_\theta, \Sigma^{-1}_\theta)$, Stepsize $\eta$, SDF.
    }
    \Output{Optimized trajectory distribution $\cN(\mu_\theta^*, \Sigma_\theta^*)$.}
        \For{$i = 1,2,\ldots$}{
            \tcc{Compute marginals using \eqref{eq:factor_map}.} 
            $\{q_l\sim \mathcal{N}(\mu_{\theta}^{\ell}, \Sigma_{\theta}^{\ell})\}_{\ell=1}^L \leftarrow (\mu_\theta, \Sigma^{-1}_\theta, \{M_l\}_{l=1}^{L}).$

            \For{$l=1,2,\dots, L$}{
            
            Compute \eqref{eq:GVI_derivatives} for \eqref{eq:defn_prior_factor} using closed-form \eqref{eq:GVI_derivatives_prior} and marginal Gaussian $q_l$.
            
            Compute \eqref{eq:GVI_derivatives} for \eqref{eq:defn_nonlinear_factor} using sparse GH quadratures in Appendix \ref{sec:GH_quadratures} and marginal Gaussian $q_l$.
        }
            \tcc{Compute joint using \eqref{eq:update_factor_to_joint}.}

            $(\frac{\partial \mathcal{J}(q)}{\partial \mu_\theta}, \frac{\partial^2 \mathcal{J}(q)}{\partial \mu_\theta \partial \mu_\theta^T}) \leftarrow \{(\frac{\partial \mathcal{J}_\ell(q_\ell)}{\partial \mu_{\theta}^\ell}, \frac{\partial^2 \mathcal{J}_\ell(q_\ell)}{\partial \mu_\theta^\ell (\partial \mu_\theta^\ell)^T}, M_{\ell})\}_\ell$ 
            
            Compute $\delta \mu_{\theta}, \delta \Sigma^{-1}_\theta$ using \eqref{eq:GVI_delta_mu} and \eqref{eq:GVI_delta_Sigma}.

            $\mu_{\theta} \leftarrow \mu_{\theta} + \eta \, \delta \mu_{\theta}, \;\;\;  \Sigma^{-1}_{\theta} \leftarrow \Sigma^{-1}_{\theta} + \eta \, \delta \Sigma^{-1}_{\theta},$ 
            
        }
\end{algorithm}
}

At each iteration of Algorithm \ref{alg:gvimp}, line $2$ computes the collection of factorized marginal Gaussian distritbuions $\{q_l\}$. Then, factorized increments are computed on factor levels in line $4$ and line $5$. We collected all the gradients on the factor levels and used them to calculate the joint level gradients in line $6$. Finally, we updated our proposal for Gaussian distribution in line $7$ and line $8$. A backtracking search \eqref{eq:GVI_backtracking} can be deployed in Algorithm \ref{alg:gvimp} to determine step size.

{\em (b) Complexity Analysis.}
Following our definitions of the dimensions, the state dimension is $n$, and time discretization is with length $N+1$. The factor graph has $L=2N+3$ factors; the maximum dimension is $d=2n$. Despite the sparse GP representation, the trajectory length is often larger than the state dimension. In our algorithm, at each iteration, the mapping from the joint to the marginal distributions involves marginal covariance computation for the tree-structured factor graph, which has the complexity of $O(Ld^3)$ using efficient algorithms such as Gaussian belief propagation \cite{wainwright2000tree, chou1994multiscale, su2015convergence}.

At marginal levels, for the prior factors \eqref{eq:prior_factor}, the matrix multiplication closed-form in Lemma \ref{lem:closed_form_prior} has time complexity $O(d^3)$. For the nonlinear collision factor \eqref{eq:collision_factor}, GH quadrature's complexity depends upon the \textit{polynomial exactness} $k_q$ that we desire \cite{heiss2008likelihood}. Lemma \ref{lem:sparseGH_complexity} shows that the time complexity is $O(d^{k_q})$. The polynomial precision is assumed to satisfy $k_q \geq 3$; thus, the nonlinear factor update is the computation bottleneck of the GVI-MP algorithm.

Combining the above, the total time complexity of one iteration in the GVI-MP algorithm is $O(Nn^{k_q})$. Thanks to factorization and sparse-grid quadratures, we remark that our method depends \textit{linearly} on the number of discretizations, $N$, and \textit{polynomially} on the state dimension $n$.

\section{Robust Variational Motion Planning}
\label{sec:robust_MP}
This section shows that the variational inference motion planning brings robustness to the solution compared with the deterministic counterpart.

\subsection{Entropy Regularized Formulation}
\label{sec:robust_MP_entropy_maximization_formulation}
Gaussian Variational Inference Motion Planning is equivalent to an entropy-regularized motion planning. This equivalence can be shown by rewriting the equation \eqref{eq:GVI_formulation} as follows
\begin{equation}
\begin{split}
    q^{\star} &= \underset{q\in \mathcal{Q}}{\arg\min}\; {\rm KL} \left( q(\bX) \parallel p(\bX|Z) \right)\\
        &= \underset{q \in \mathcal{Q}}{\arg\max}\; \mE_q\left[\log p(\bX|Z)\right] + \mathcal{H}(q(\bX)).
\end{split}
   \label{eq:GVI_entropy_formulation}
\end{equation} 

\begin{figure}[t]
    \centering
    \includegraphics[width=0.9\linewidth]{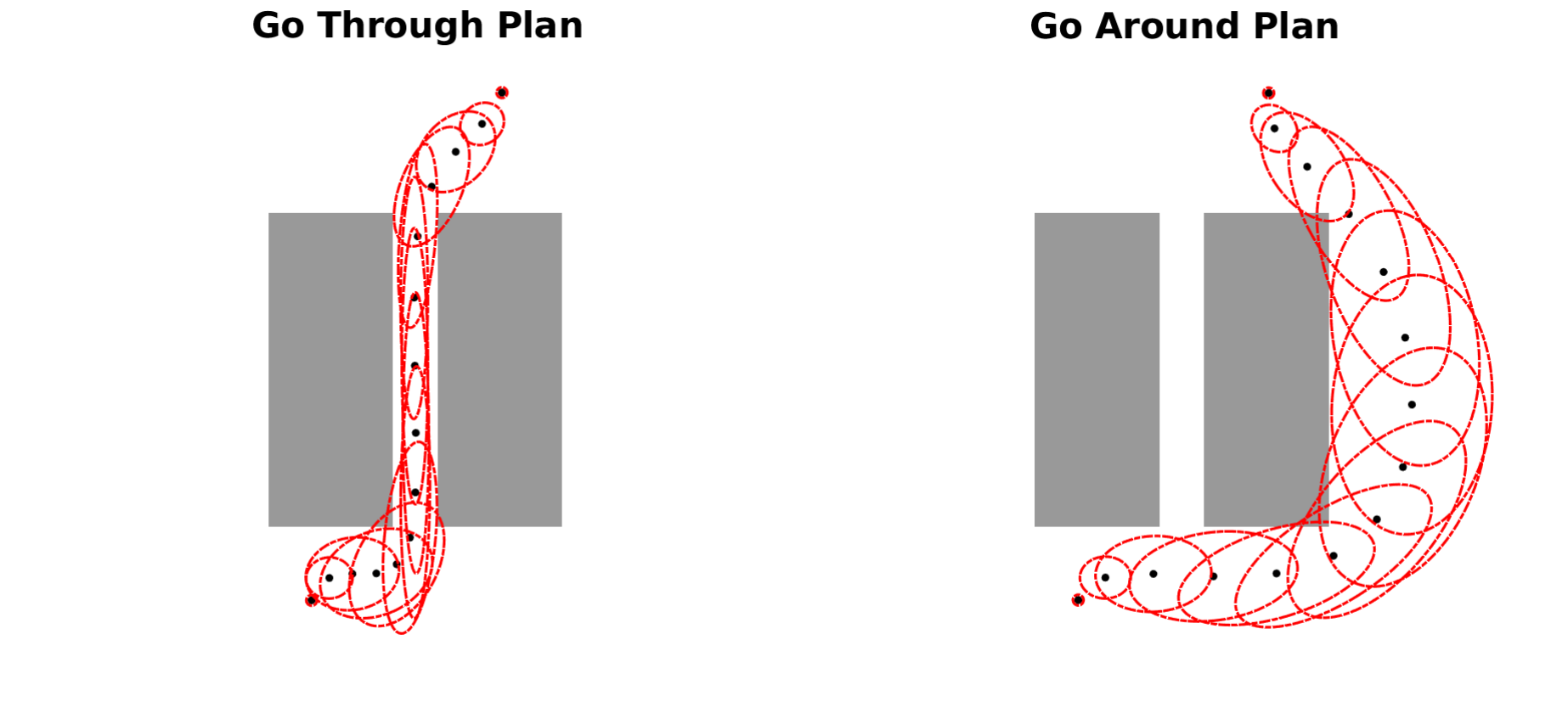}
    \caption{Entropy regularized robust motion planning for a narrow gap environment \cite{yu2023}.}
    \label{fig:entropy_MP_comparison_narrow}
\end{figure}
\begin{table}[t]
\centering
\begin{tabular}{|c|c|c|c|c|c|}
\hline
 & \textbf{Prior} & \textbf{Collision} & \textbf{MP} & \textbf{Entropy} & \textbf{Total} \\ \hline
Left & \textbf{34.4583} & 9.1584 & \textbf{43.6168}  & 44.1752 & 87.7920 \\ \hline
Right & 42.9730 & \textbf{2.0464} & 45.0193 & \textbf{39.9193} & \textbf{84.9387} \\ \hline
\end{tabular}
\caption{Comparing costs for plans in Fig. \ref{fig:entropy_MP_comparison_narrow}. The entropy cost $\log(\lvert\Sigma_{\theta}^{-1}\rvert)$ prioritized a less risky solution. \textit{`MP'} represents the sum of prior and collision costs.}
\label{tab:narrow_comparison}
\end{table}
\begin{figure*}[ht]
\centering
    \begin{subfigure}[ht]{0.9\textwidth}
    \centering
    \includegraphics[width=\textwidth]{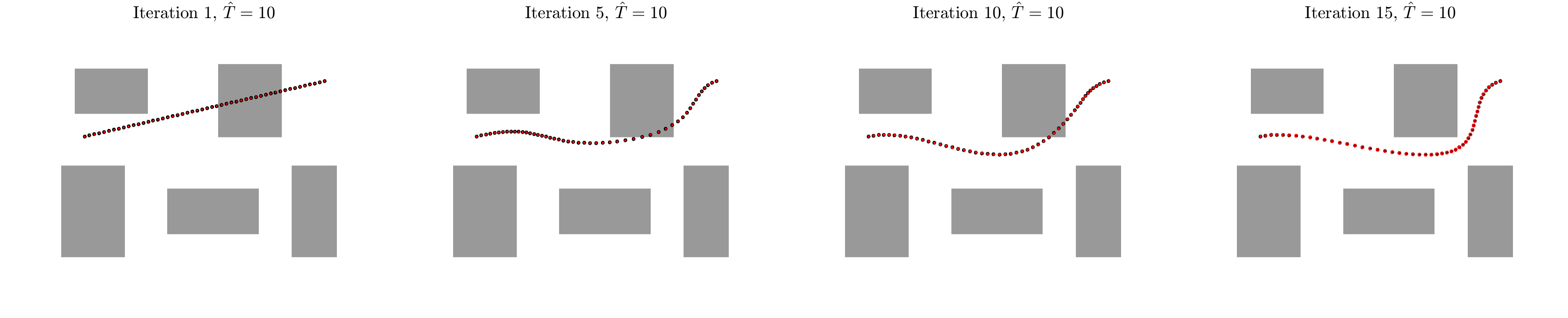}
    \end{subfigure}
    \hfill
    \begin{subfigure}[ht]{0.9\textwidth}
    \centering
    \includegraphics[width=\textwidth]{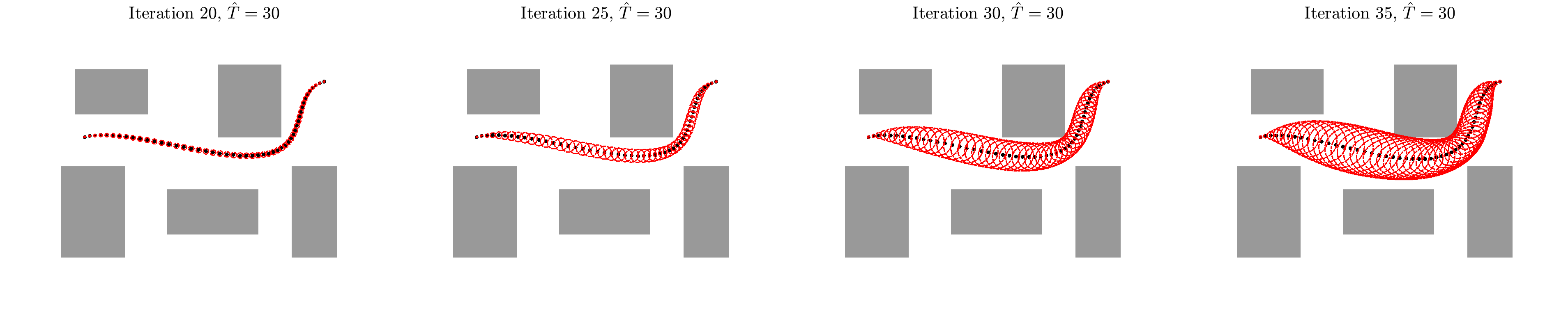}
    \end{subfigure}
  \caption{Switching temperature for balancing optimality and maximum-entropy robustness. We show intermediate iterations with low and high temperatures. Black dots represent the mean trajectory, and ellipsoids are the $3\sigma$ covariance contours.
  }
  \label{fig:low_high_temperature_planning}
\end{figure*}
\begin{figure}[ht]
    \centering
    \includegraphics[width=\linewidth]{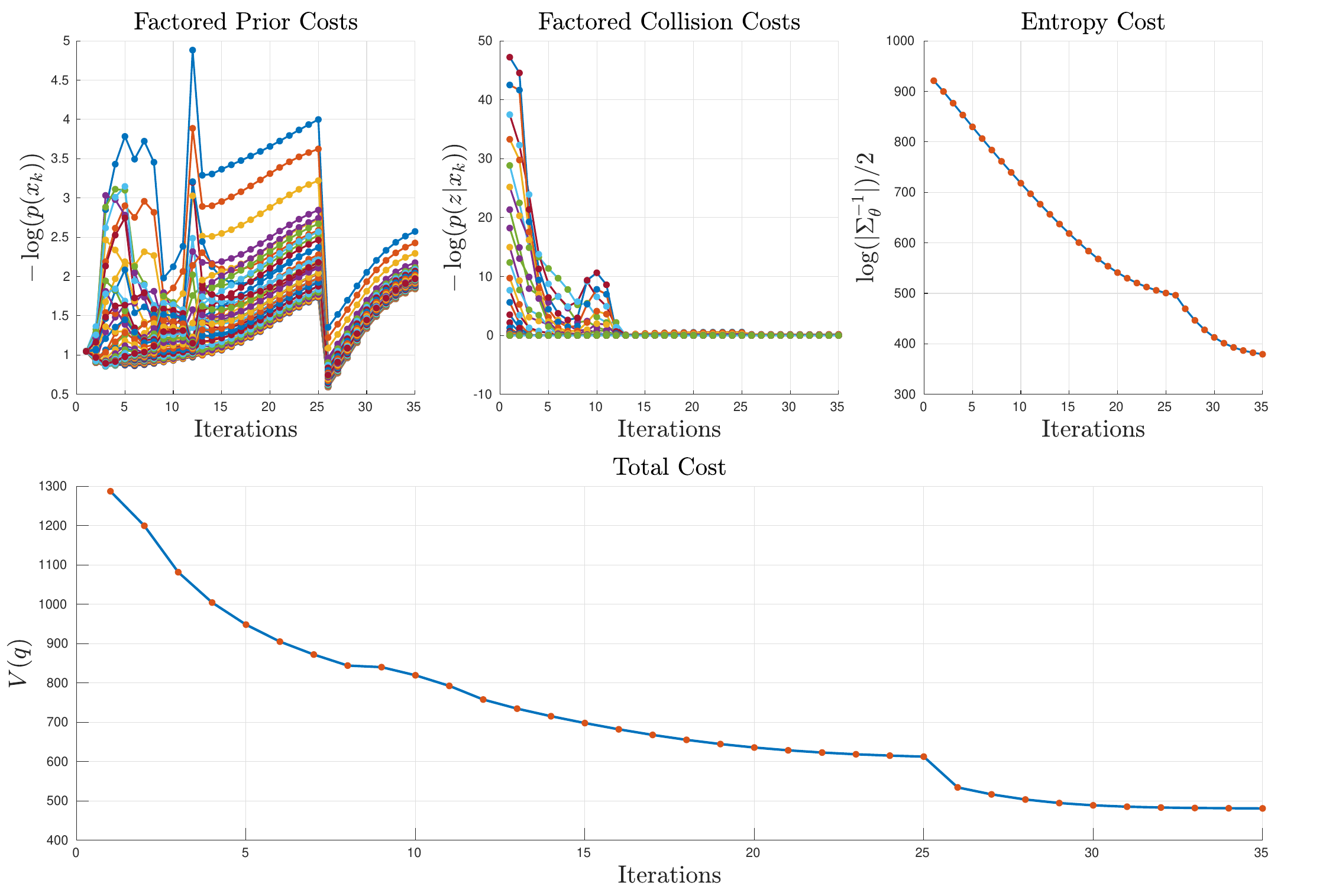}
    \caption{Factorized and total costs in Fig. \ref{fig:low_high_temperature_planning}. At iteration $25$, we switched the temperature from $\hat{T}=10$ to $\hat{T}=30$. Entropy costs decrease faster from the transition. }
    \label{fig:low_high_temperature_cost}
\end{figure}
In \eqref{eq:GVI_entropy_formulation}, our objective function consists of two parts, the \textit{optimality} and the \textit{robustness}. The optimality objectives find expression in maximizing the expected posterior probability $\mE\left[p(\bX|Z)\right]$. The \textit{robustness} is characterized by maximized Entropy during the planning to accommodate tracking errors and environment noise during execution. For Gaussian distributions, we have $\mathcal{H}(q) = -\mE_q \log q(\bX) \propto \log \det \Sigma_\theta$.
The term $\det \Sigma_\theta$ scales with the equi-density contour ellipsoids' volume $\mathcal{E}$ \cite{Stra22}. Maximizing $\mathcal{E}$ entails expanding the feasible and safe trajectory distribution around the mean to encompass larger space for a given confidence level. 

An example is presented in Fig. \ref{fig:entropy_MP_comparison_narrow}. Two solutions are obtained for the same environment and task. On the left is a shortcut solution, which is visibly riskier due to a narrow gap between obstacles. Conversely, the right-side solution entails a longer trajectory, yet its smaller entropy (evident in the extended covariance matrices) encapsulates enhanced robustness. Table \ref{tab:narrow_comparison} compiles the costs for the plans depicted in Figure \ref{fig:entropy_MP_comparison_narrow}. The difference in entropy costs impacts the total cost between the left and right solutions. The comparison showcases that entropy brings robustness to uncertainties in the decision-making for motion planning.

\subsection{High-temperature Optimization for Robustness}
\label{sec:robust_MP_high_temperature}

To trade the weights on motion planning optimality versus robustness, we introduce a hyperparameter termed \textit{temperature} $\hat{T}$ into the objective \eqref{eq:GVI_entropy_formulation}
\begin{subequations}
\begin{eqnarray*}
    q^{\star} \!\!\!&=&\!\!\! \underset{q \in \mathcal{Q}}{\arg\max}\; \mE_q\left[\log p\left(\bX|Z\right)\right] + \hat{T} \, \mathcal{H}(q(\bX))
\end{eqnarray*}
\end{subequations}
The temperature parameter $\hat{T}$ is a weighting factor between the optimality and the robustness. This is similar to the idea behind $\beta$-VAE \cite{higgins2016beta}, which solves an underlying constrained inference problem, and the temperature is the Lagrangian multiplier. Integrating $\hat{T}$ into Algorithm \ref{alg:gvimp}, we write the above as
\begin{subequations}
\label{eq:high_temperature_objective}
    \begin{eqnarray}
    q^{\star} \!\!\!&=&\!\!\! \underset{q \in \mathcal{Q}}{\arg\max}\; \mE_q \left[\frac{1}{\hat{T}} \log p(\bX|Z) \right] + \mathcal{H}(q(\bX))
    \\
    \!\!\!&=&\!\!\! \underset{q \in \mathcal{Q}}{\arg\min}\; \mE_q \left[\frac{1}{\hat{T}} \psi(\bX) \right] - \mathcal{H}(q(\bX)).
    \end{eqnarray}
\end{subequations}

Comparing \eqref{eq:high_temperature_objective} with the formulation \eqref{eq:GVI_entropy_formulation}, it is clear that Algorithm \ref{alg:gvimp} can be directly applied in the high-temperature optimization \eqref{eq:high_temperature_objective} with a direct re-scaling on the cost function values. We modify Algorithm \ref{alg:gvimp} to include a low-temperature and a high-temperature phase. In Algorithm \ref{alg:gvimp}, we first set a lower temperature and optimize until convergence or a pre-defined maximum iteration. We then switch to a higher temperature and continue the algorithm with the re-scaled objective function \eqref{eq:high_temperature_objective}. The optimization in the second phase puts more weight on maximizing the entropy.

Fig. \ref{fig:low_high_temperature_planning} showcases the switching-temperature optimization process and Fig. \ref{fig:low_high_temperature_cost} shows the corresponding costs. In Fig. \ref{fig:low_high_temperature_planning}, the first row presents selected low-temperature iterations where the optimizer promotes a collision-free trajectory. The second row represents the high-temperature phase, where the maximum-entropy robustness objective dominates the optimization. We observe a steeper decrease in entropy loss while an increase in the prior loss in the high-temperature phase in Fig. \ref{fig:low_high_temperature_cost}.

\begin{figure}[ht]
    \centering
    \includegraphics[width=0.8\linewidth]{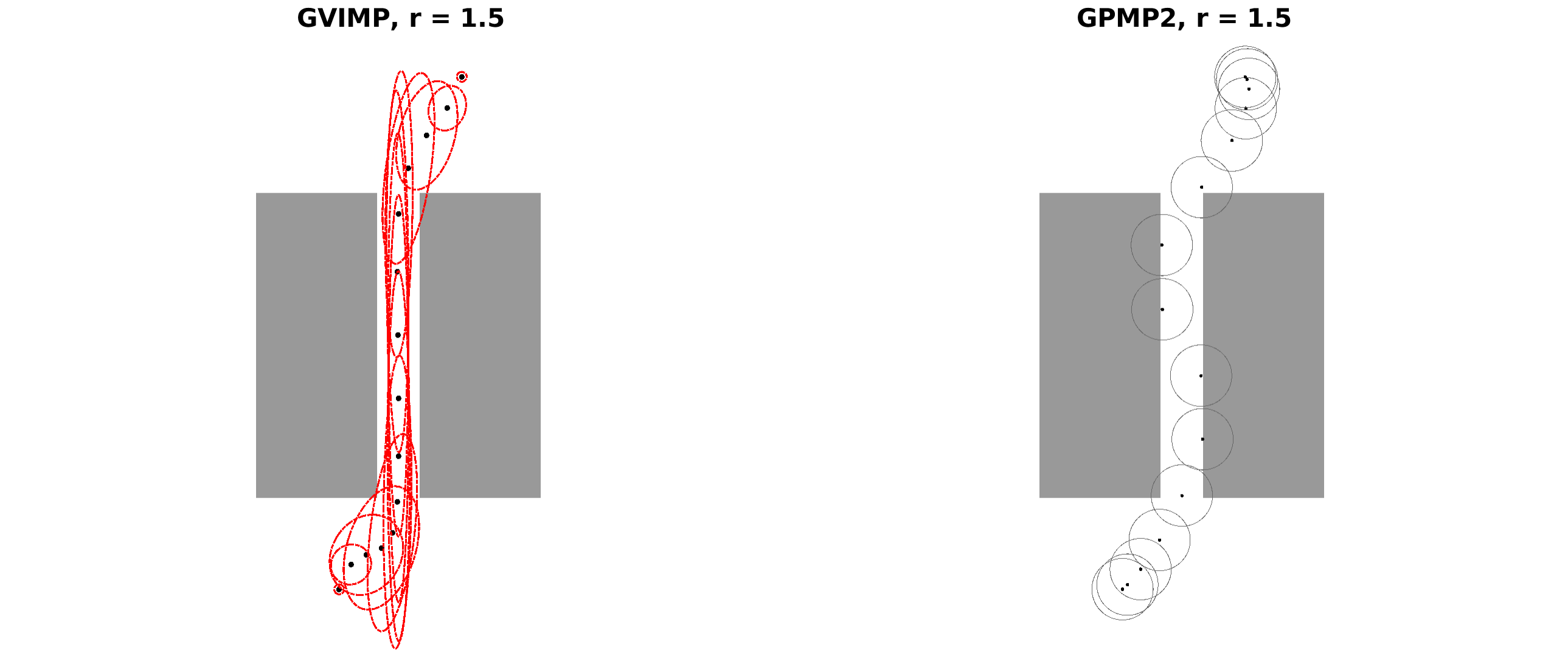}
    \caption{Optimizable covariances of GVI-MP compared with fixed collision-checking radius in GPMP.}
    \label{fig:comparison_gothrough_radius}
\end{figure}
\subsection{Intrinsic Robustness of Variational Representation}
\label{sec:robust_MP_intrinsic_robustness}
The distributional modeling in GVI-MP brings intrinsic robustness to the motion plans. Firstly, the covariance becomes an optimization variable. Fig. \ref{fig:comparison_gothrough_radius} illustrates an example. In GPMP2, the collision-checking radius must be designed to find a feasible solution. In GVI-MP, the optimizable covariance enables the algorithm to find a feasible trajectory distribution for the non-optimal radius.

Furthermore, the Gaussian assumption in GVI-MP allows efficient sampling of the trajectories from the optimized distribution. Combined with the entropy-maximized motion planning paradigm and high-temperature tuning, the output trajectory distribution is tolerant of uncertainties in dynamics modeling, such as errors in inertia matrix or approximations in the linearization of dynamics, as well as sensor measurement noises \cite{nguyen2012modeling, chen2022should} in real-world applications.

\section{Proximal Covariance Steering Motion Planner (PCS-MP) Algorithm}
\label{sec:PCSMP}
The GVI-MP explores the sparse factor graph structure to increase the efficiency. However, as we saw in Section \ref{sec:GVIMP_algorithm_complexity_analysis}, the quadrature expectation estimation represents a computation bottleneck. The equivalent stochastic control formulation to GVI-MP \eqref{eq:SCP_formulation_GVI} represents a terminal-cost-based regulator problem. We develop in this section a more efficient algorithm to solve a more intricate constrained variant of \eqref{eq:SCP_formulation_GVI}, which aligns with the covariance steering paradigm \cite{CheGeoPav15a, CheGeoPav15b, CheGeoPav17a}.

\subsection{Covariance Steering with Nonlinear State Cost}
\label{sec:PCSMP_nonlinear_cs}
We consider the covariance steering problem with nonlinear state costs 
\begin{subequations}\label{eq:cov_steering_lin_dyn}
\begin{eqnarray}\label{eq:cov_steering_lin_dyn1}
    \min_{u} && \mE \left\{\int_0^T [\frac{1}{2}\|u_t\|^2 + V(X_t)]dt\right\}
    \\\label{eq:cov_steering_lin_dyn2}
    &&\hspace{-0.8cm} dX_t = A_t X_t dt + a_t dt + B_t (u_t dt + \sqrt{\epsilon} d W_t)
    \\\label{eq:cov_steering_lin_dyn3}
    && 
    \hspace{-0.8cm} X_0 \sim \mathcal{N}(\mu_0, K_0),\quad X_T \sim \mathcal{N}(\mu_T, K_T).
\end{eqnarray}
\end{subequations}
Let the state cost $V(X_t)$ be 
\begin{equation}
\label{eq:PGCS_state_cost_h}
    V(X_t)=\lVert \bh(X_t) \rVert_{\Sigma_{\rm obs}}^2,
\end{equation}
then formulation \eqref{eq:cov_steering_lin_dyn} is equivalent to the stochastic control problem \eqref{eq:SCP_formulation_GVI} with a constraint. We also inject a noise intensity $\epsilon$ in \eqref{eq:cov_steering_lin_dyn} compared with \eqref{eq:SCP_formulation_GVI}. By Girsanov theorem, \eqref{eq:cov_steering_lin_dyn} is equivalent to a distributional control 
\begin{subequations}\label{eq:cov_steering_KLformulation}
\begin{eqnarray}
    \!\!\!\!\!\!\!\!\! && \!\!\! \min_{\cP^u \in \Pi(\rho_0, \rho_T)}  \int \left\{\log\frac{d\cP^u}{d\cP^0} +\frac{1}{\epsilon}V\right\} d\cP^u \label{eq:cov_steering_KLformulation1}
    \\
    \!\!\!\!\!\!\!\!\! && \!\!\! (X_0)_\sharp \cP^u = \rho_0, \;\; (X_T)_\sharp\cP^u = \rho_T,\label{eq:cov_steering_KLformulation2}
\end{eqnarray}
\end{subequations}
where $(\cdot)_\sharp \cP^u$ is the push-forward operator induced by $\cP^u$, and $\rho_0 = \mathcal{N}(\mu_0, K_0)$, $\rho_T = \mathcal{N}(\mu_T, K_T).$ Problem \eqref{eq:cov_steering_KLformulation} is an optimization over the space $\Pi$ which is induced by the parameterized process \eqref{eq:cov_steering_lin_dyn2} with marginal distributions \eqref{eq:cov_steering_lin_dyn3}.

\subsection{Proximal Gradient in the Distribution Space}
\label{sec:pcsmp_proximal_gradient}
{\em a) Composite optimization scheme.}
Proximal gradient method \cite{Bec17} solves a composite optimization by splitting the objective \eqref{eq:cov_steering_KLformulation1} into \begin{equation}\label{eq:FG}
    \min_{\cP^u\in\Pi(\rho_0, \rho_T)} F(\cP^u) + G(\cP^u),
\end{equation}
where
	\begin{align*}\label{eq:FG_define}
		F(\cP^u) = \int [\frac{1}{\epsilon}V-\log d\cP^0] d\cP^u, \;\; G(\cP^u) = \int d\cP^u \log d\cP^u.
	\end{align*}
In the proximal gradient paradigm, only $F$ is required to be smooth. In the Euclidean spaces, for variable $y$, the algorithm follows the update
    \begin{equation}\label{eq:pgupdate}
    y^{k+1} \hongzhe{=} \argmin_{y\in \cY} G(y) \!+\! \frac{1}{2 \eta} \|y\!-\!y^k\|^2\!\! +\!\! \langle \nabla F(y^k), y\!-\!y^k\rangle.
    \end{equation}
The proximal gradient for non-Euclidean distances is built upon the mirror descent method \cite{Bec17}. In \eqref{eq:pgupdate}, the vector norm can be replaced by a Bregman divergence $D(\cdot, \cdot)$, and the generalized non-Euclidean proximal gradient update is
	\begin{equation*}\label{eq:proximal_gradient_update}
		y^{k+1} = \argmin_{y\in \cY} G(y) + \frac{1}{\eta} D(y, y^k) + \langle \nabla F(y^k), y-y^k\rangle.
	\end{equation*} 
The KL divergence is often a choice for $D(\cdot, \cdot)$ for distributions, and the proximal gradient step is then obtained from 
\begin{equation}\label{eq:PG}
    \cP_{k+1} = \argmin_{\cP^u\in\Pi(\rho_0, \rho_T)} G(\cP^u) + \frac{1}{\eta} {\rm KL}( \cP^u \parallel \cP_k)+ \langle \frac{\delta F}{\delta \cP}(\cP_k), \cP^u\rangle.
\end{equation}
At iteration $k$, $\mP_k$ is induced by an updated process
\begin{equation}
\label{eq:dynamics_controlled_k}
    d X_t = A_t^kX_t dt + a_t^k dt + \sqrt{\epsilon} B_t (u_t d t + dW_t). 
\end{equation}

{\em b) Proximal update for quadratic state cost.}
Denote $z_t$ as the mean trajectory of $\cP^k$ induced by \eqref{eq:dynamics_controlled_k} at iteration $k$. We approximate the state cost by a quadratic function 
\begin{equation}\label{eq:V_quadratic_approximation}
    \hat V(t,x) \approx V(z_t) + (x-z_t)^T \nabla V + \frac{1}{2}\lVert x-z_t \rVert_{\nabla^2 V}^2.
\end{equation}
With this approximation, each iteration of proximal gradient \eqref{eq:PG} amounts to solving the distributional control
\begin{equation}\label{eq:noncovcontrol}
    \min_{\cP^u\in\Pi(\rho_0, \rho_T)} \int [\frac{1}{\epsilon}\hat V-\log d \cP^0] d\cP^u + \int d\cP^u \log d\cP^u,
\end{equation}
and the component $F(\cP^u)$ in \eqref{eq:FG} becomes 
\begin{equation}\label{eq:approx_FP}
    F(\cP^u)=\int [\frac{1}{\epsilon}\hat V-\log d \cP^0] d\cP^u = \langle \frac{1}{\epsilon}\hat V-\log d\cP^0, \cP^u \rangle.
\end{equation}
We denote $\frac{\delta F}{\delta \cP}$ as the variation of $F(\cP)$ for a small variation $\delta \cP$. Before deriving the expressions for $\frac{\delta F}{\delta \cP}$, we distinguish 3 different trajectory distributions:
\begin{enumerate}
    \item Distribution $\cP^0$ induced by the passive dynamics \eqref{eq:dynamics_uncontrolled}.
    \item Distribution $\cP_k$ induced by the controlled dynamics \eqref{eq:dynamics_controlled_k} at the current iteration $k$.
    \item Candidate distribution $\cP^u$ in solution space $\Pi(\rho_0, \rho_T)$.
\end{enumerate}

\subsection{PCS-MP Algorithm}
\label{sec:PCSMP_pcs_algorithm}
{\em Notations.} In the following derivations, the gradients of $V$ are evaluated on the nominal mean trajectory with respect to states, i.e., $\nabla V \triangleq \nabla_{X} V(z_t), \; \nabla^2 V \triangleq \nabla_{X} ^2 V(z_t).$

\begin{lemma}\label{thm:variationP}
Let $\cP$ be the path distribution induced by a Gauss Markov process \eqref{eq:dynamics_uncontrolled}, then the variation of $F(\cP)$ w.r.t. $\cP$ for linear dynamics \eqref{eq:cov_steering_lin_dyn} is
\begin{eqnarray}\label{eq:gradietnF_linear}
    \frac{\delta F}{\delta \cP}(\cP)= \frac{1}{\epsilon}\hat V -\log d \cP^0 + \frac{1}{2\epsilon} x^T \nabla\tr(\nabla^2 V \Sigma_t).
\end{eqnarray}
\end{lemma}
The proof can be found in Appendix \ref{sec:proof_lemma_variationP}. Next, write the objective \eqref{eq:PG} using \eqref{eq:gradietnF_linear}. For simplicity, we drop the dependences on time $t$ and iteration $k$. The objective in \eqref{eq:PG} is
\begin{equation}
\label{eq:objKL_lin_dyn}
    \int [\frac{1}{\epsilon}\hat V+ \frac{1}{2\epsilon} x^T \nabla\tr((\nabla^2 V) \Sigma)] d\cP +{\rm KL} (\cP \parallel \cP^0) + \frac{1}{\eta} {\rm KL}(\cP \parallel \cP^u),
\end{equation}
where we just combined the term $ -\int \log d \cP^0 d\cP$ in \eqref{eq:noncovcontrol} with $G(\cP^u) = \int d\cP^u \log d\cP^u$ to form $G(\cP^u) - \int \log d \cP^0 d\cP^u = {\rm KL}\left(\cP^u \parallel \cP^0\right).$

Next, we show that each step in the PCS-MP is a linear covariance steering problem.
\begin{lemma}\label{thm:Qrk}
Each step of the proximal gradient covariance steering \eqref{eq:PG} is obtained by solving the following problem
	\begin{subequations}\label{eq:iterlinear}
	\begin{eqnarray}\label{eq:iterlinear1}
		\!\!\!\!\!\!\!\!\min_{u}\!\!\!\! &&\!\!\!\! \mE \left\{\int_0^T [\frac{1}{2}\|u_t\|^2 + \frac{1}{2} X_t^T Q_t^k X_t + X_t^T r_t^k]dt\right\}
		\\ 
        \!\!\!\!\!\!\!\!\!\!\!\! &&\!\!\!\!  \hspace{-0.1cm}dX_t = \bar A_t^k X_t dt +  \bar a_t^k dt + B_t(u_tdt+ \sqrt{\epsilon} dW_t) \label{eq:iterlinear2}
		\\ 
        \!\!\!\!\!\!\!\!\!\!\!\! && \!\!\!\! X_0 \sim \mathcal{N}(\mu_0, K_0),\quad X_T \sim \mathcal{N}(\mu_T, K_T),\label{eq:iterlinear3}
	\end{eqnarray}
	\end{subequations}
where 
\begin{equation} 
    \begin{split}
        \bar A_t^k &= \frac{1}{1+\eta} \left(A_t^k+\eta A_t\right) , \; \bar a_t^k = \frac{1}{1+\eta}\left(a_t^k + \eta a_t\right)
        \\	
		Q_t^k &= \frac{\eta}{(1+\eta)^2} (A_t^k-A_t)^T(B_tB_t^T)^{\dagger}\left(A_t^k- A_t\right) + \frac{\eta}{1+\eta}\nabla^2 V         
        \\ 
        r_t^k &= \frac{\eta}{(1+\eta)^2}(A_t^k-\hat A_t^k)^T(B_tB_t^T)^{\dagger}(a_t^k-\hat a_t^k) + 
        \\
        &\hspace{0.5cm} \frac{\eta}{1+\eta} \left( \nabla V  - \left(\nabla^2V\right) z^k_t \right),
    \end{split}       
\label{eq:Qr_lin_dyn}
\end{equation}
where $(\cdot)^\dagger$ represents pseudo-inverse, $z^k_t$ and $\Sigma^k_t$ are the mean and covariance of $\cP_k$, respectively. 
\end{lemma}
The proof can be found in Appendix \ref{sec:proof_lemma_Qrk}. Problem \eqref{eq:iterlinear} is a linear covariance steering problem that admits a closed-form solution (c.f., Section \ref{sec:background_linear_CS}) $u_t^\star = K_t^k X_t + d_t^k$, and the next-iteration Gauss Markov process follows
\begin{equation}\label{eq:updateAa}
    A_t^{k+1} =  \bar A_t^k + B_tK_t^k,\;\;\;\; a_t^{k+1} = \bar a_t^k + B_td_t^k.  
\end{equation}
The nominal trajectory $(z_t^k, \Sigma_t^k)$ thus evolves as
\begin{equation}
\label{eq:propagate_nominal}
    \dot z_t^k = A_t^k z_t^k + a_t^k, \;\; \dot \Sigma_t^k = A_t^k \Sigma_t^k + \Sigma_t^k (A_t^k)^T + \epsilon B_tB_t^T.
\end{equation}
{
\setlength{\algomargin}{1.55em}
\begin{algorithm}[t]    
\DontPrintSemicolon
\SetAlgoLined
    
\caption{Proximal Gradient Covariance Steering Motion Planning (PCS-MP)}\label{alg:pgcsmp}
\SetKwInOut{Input}{Inputs}\SetKwInOut{Output}{Outputs}
\Input{Noise $\epsilon$, Step size $\eta$,\\
Initial values $(A_t^0, a_t^0, z_t^0, \Sigma_t^0, V_0)$.}
\Output{Optimized nominal $(z_t^*, \Sigma_t^*)$, 
\\
feedback gain $(K_t^*, d_t^*)$.}
\For{$k = 1,2,\dots$} {

   \tcc{Propagate dynamics.}
   $(z_t^k, \Sigma_t^k) \leftarrow$ {\rm Propagate}$(z_t^k, \Sigma_t^k)$ \eqref{eq:propagate_nominal}.
   
    \tcc{Compute coefficients.}    
    $(\bar{A}_t^k, \bar{a}_t^k, Q_t^k, r_t^k) \leftarrow$ {\rm ComputeQr \eqref{eq:Qr_lin_dyn}.}
    
    \tcc{Linear covariance control.}   
    $(K_t^k, d_t^k) \leftarrow$ LinearCovarianceSteering \eqref{eq:iterlinear}.

}
$K_t^* \leftarrow K_t^{k}, \; d_t^* \leftarrow d_t^{k}, \; z_t^* \leftarrow z_t^{k}, \Sigma_t^* \leftarrow \Sigma_t^{k}.$ 
\end{algorithm}
}
We introduce the PCS-MP algorithm in Algorithm \ref{alg:pgcsmp}. Line $2-4$ are the main iterations for proximal gradient updates. At each iteration, the nominal trajectory is propagated through \eqref{eq:propagate_nominal} in line $2$, followed by the calculation of $(\bar{A}_t^k, \bar{a}_t^k, Q_t^k, r_t^k)$ for each proximal gradient step using \eqref{eq:Qr_lin_dyn} in line $3$. Subsequently, line $4$ solves the linear covariance steering problem \eqref{eq:iterlinear}. Line $6$ outputs the solved feedback gains and the nominal trajectories. We also add a backtracking procedure to select the step size in each iteration. Each step of PCS-MP involves a linear covariance control problem that admits a closed-form solution. The overall algorithm thus adopts a \textit{sub-linear} rate of convergence of the proximal gradient algorithm.

\section{experiments}
\label{sec:experiments}
We demonstrate simulated experiment results to validate the proposed algorithms. When evaluating the signed distance function, robots are modeled as balls with fixed radius $r$ at designated locations. We first consider the matrices $(A_t,a_t,B_t)$ in the linear stochastic dynamical system \eqref{eq:dynamics_controlled} under a constant velocity assumption
\begin{equation}
\label{eq:constant_vel_dynamics}
    A_t = \begin{bmatrix}
    0 & I \\
    0 & 0
    \end{bmatrix},~ 
    a_t=\begin{bmatrix}
        0\\
        0
    \end{bmatrix},~ 
    B_t=\begin{bmatrix}
    0 \\ I
    \end{bmatrix},
\end{equation}
for which the matrices $\Phi$, $Q_{i,i+1}^{-1}$, $Q^{-1}$, and $G$ in \eqref{eq:sparse_Q} and \eqref{eq:sparse_G} can be calculated explicitly \cite{barfoot2014batch}. The state $X_t$ concatenates the robots' configuration space position and velocity. 

\begin{figure*}[ht]
    \begin{subfigure}[b]{0.43\textwidth}
    \centering
    \includegraphics[width=\textwidth]{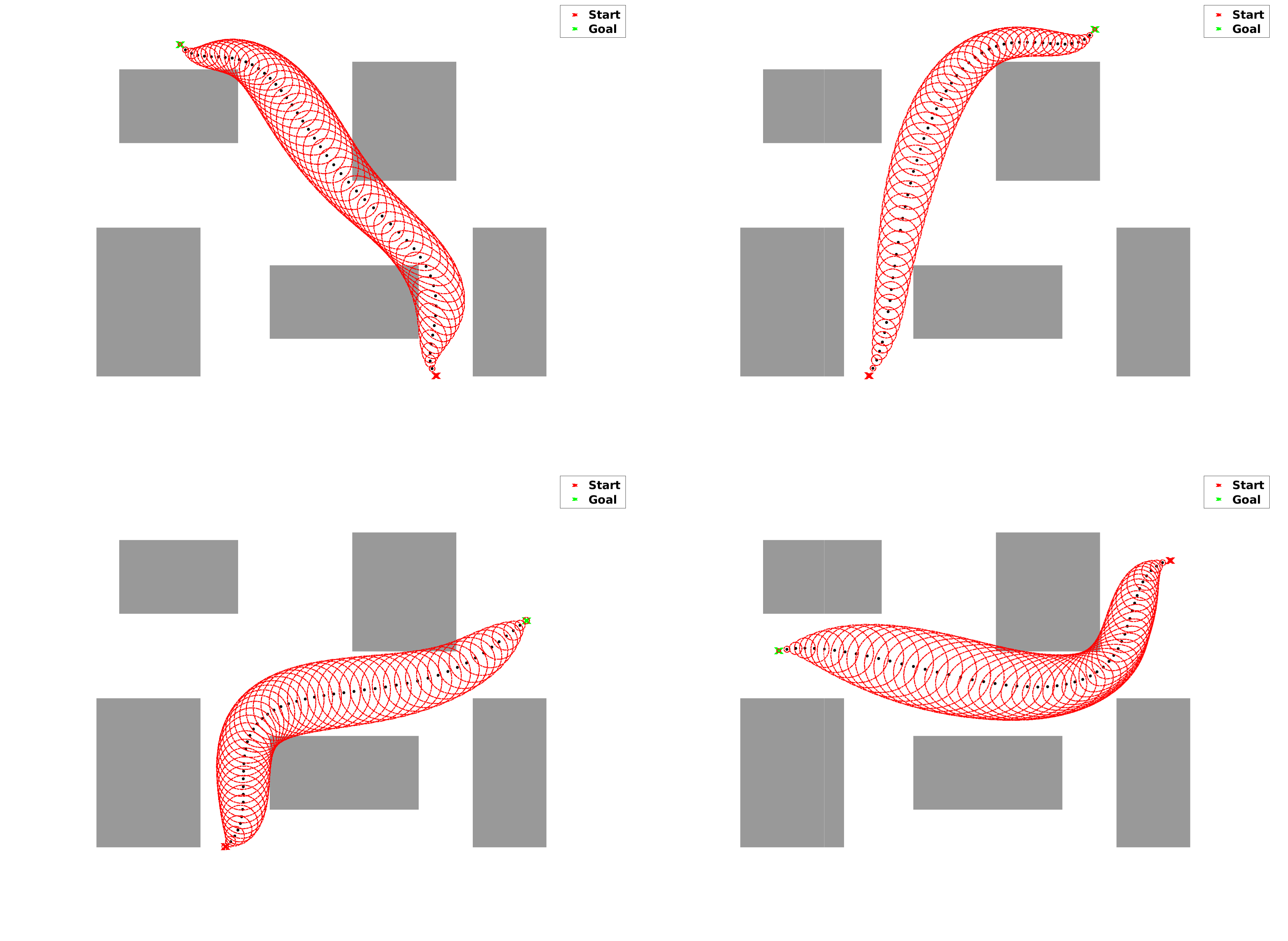}
    \caption{Trajectory distribution for $2$-DOF point robot, GVI-MP.}
    \label{fig:2dPR_GVIMP}
    \end{subfigure}
    \hfill
    \begin{subfigure}[b]{0.43\textwidth}
    \includegraphics[width=\textwidth]{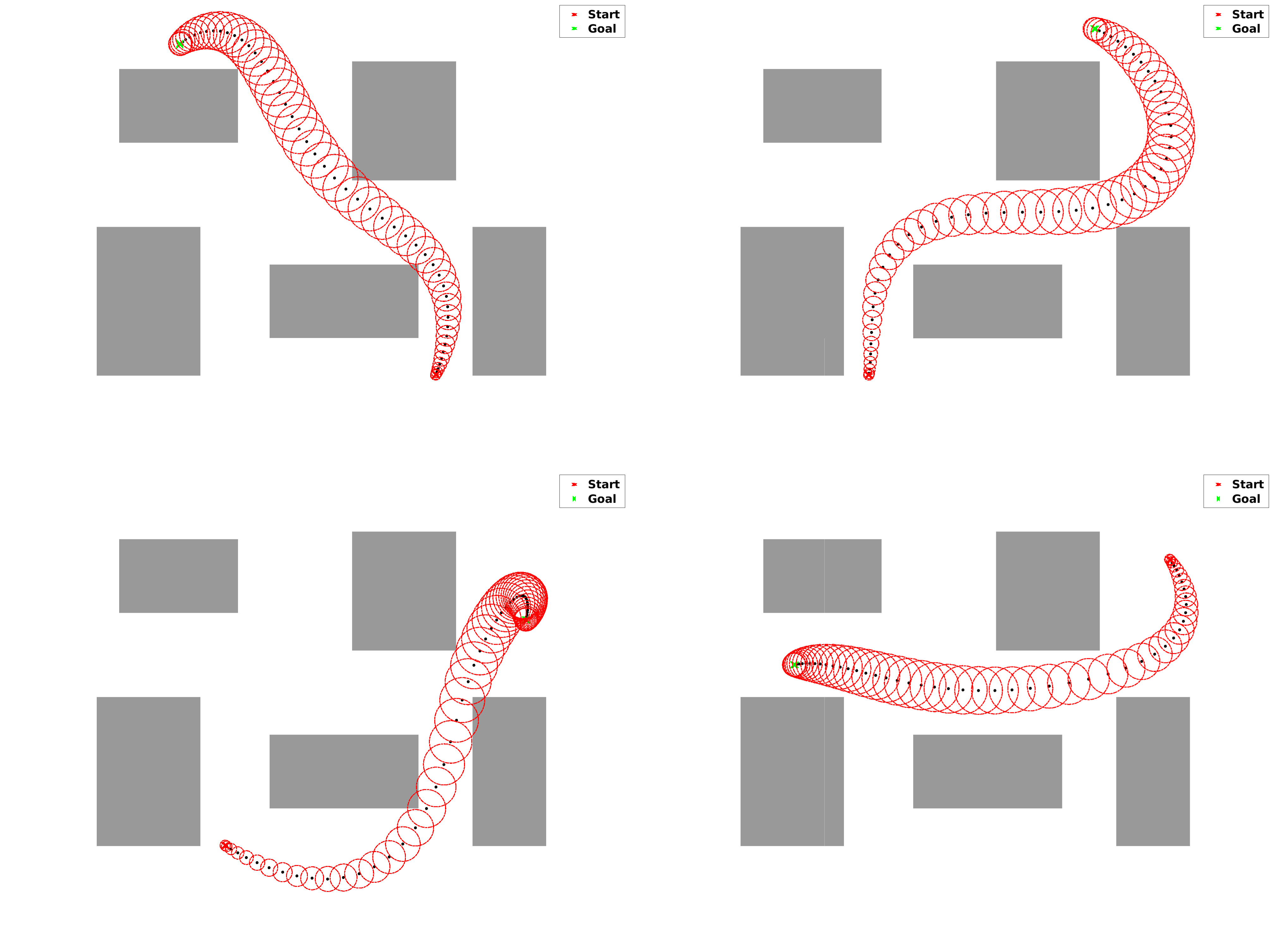}
    \caption{Trajectory distribution for $2$-DOF point robot, PCS-MP.}
    \label{fig:2dPR_PGCSMP}
    \end{subfigure}
    \caption{Motion planning for $2$D Point Robot for GVI-MP and PCS-MP. $N=50$ support states are used. For GVI-MP, time span $T=3.0 \sim 3.5$, radius $r+\epsilon_{\rm sdf} = 1.5$, $\Sigma_{\rm obs} = 15.0\sim20.0 I$, ratio between temperatures $\hat{T}_{\rm high} / \hat{T}_{\rm low} = 1.6\sim 3.0$. For PCS-MP, $T=10.5$, radius $r+\epsilon_{\rm sdf} = 2.1$, $\Sigma_{\rm obs}=10^3 I$. The covariance constraints are $K_0 = 0.01 I$, $K_T = 0.05 I$.}
    \label{fig:2dPR_MP}
\end{figure*}

\subsection{Point Robots}
{\em a) 2D point robot.} \label{exp_pr_1}
We first consider the 2-dimensional point robot motion planning in an obstacle-cluttered environment as shown in Fig. \ref{fig:2dPR_MP}. The state consists of the $2$D position and velocity $X_t = [p_x, p_y, v_x, v_y]^T.$ We conduct $4$ experiments with different start and goal states. The experiment setups, start and goal states, and the planning results of both GVI-MP and PCS-MP are shown in Fig. \ref{fig:2dPR_MP}. Both algorithms find collision-free trajectory distributions in all the $4$ experiment settings. The GVI-MP results are obtained from the optimization with temperatures (c.f., Section \ref{sec:robust_MP_high_temperature}). Under the maximum entropy motion planning formulation \eqref{eq:GVI_entropy_formulation}, the results spread wide in the feasible region in the environments compared to the PCS-MP results. We examine the terminal covariances to validate the PCS-MP algorithm in satisfying the constrained terminal covariance in \eqref{eq:cov_steering_lin_dyn}. In all the experiments in Fig. \ref{fig:2dPR_PGCSMP}, we set the initial covariance to be $K_0 = 0.01I$, and the terminal covariances are constrained to be $K_T = 0.05 I$. We record the solved terminal covariances, $K_T^*$. The norm difference $\lVert K_T^* - K_T \rVert = 6\times 10^{-4}$ in all $4$ experiments.

{\em b) $3$D point robot.}
Using PCS-MP, we conduct experiments for a $3$-DOF point robot in a cluttered environment. The state consists of the $3$D position and velocity of the robot $X_t = [p_x, p_y, p_z, v_x, v_y, v_z]^T.$
We did $4$ experiments of different start and goal states. Fig. \ref{fig:3d_pR} shows the experiment settings and planning results. In all 4 experiment settings, the norms of the terminal covariance error are less than $8\times 10^{-4}$.
\begin{figure*}[ht]
\centering
    \begin{subfigure}[ht]{0.23\textwidth}
    \centering
    \includegraphics[width=0.95\textwidth]{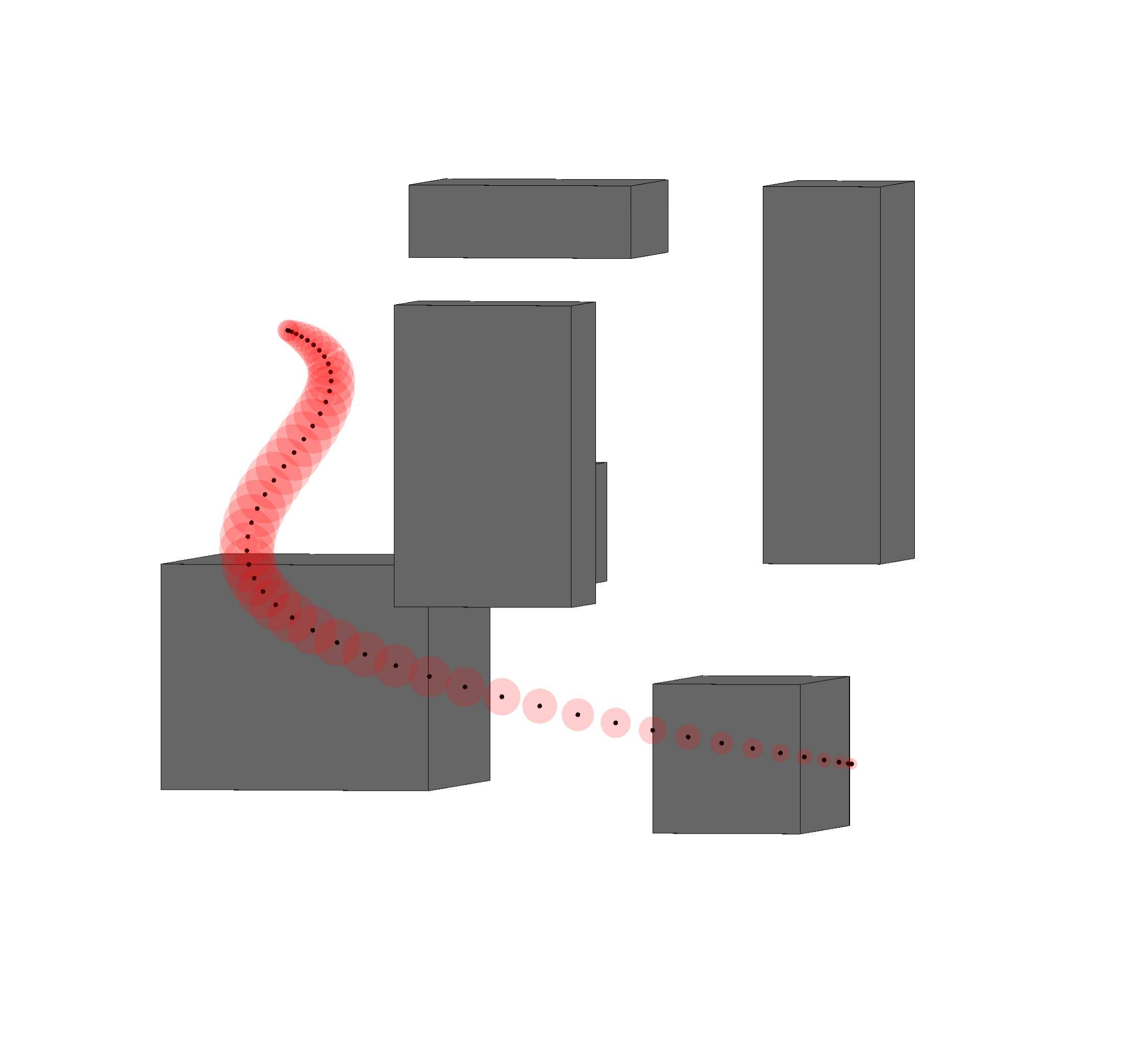}
    \end{subfigure}
    \hfill
    \begin{subfigure}[ht]{0.23\textwidth}
    \centering
    \includegraphics[width=0.95\textwidth]{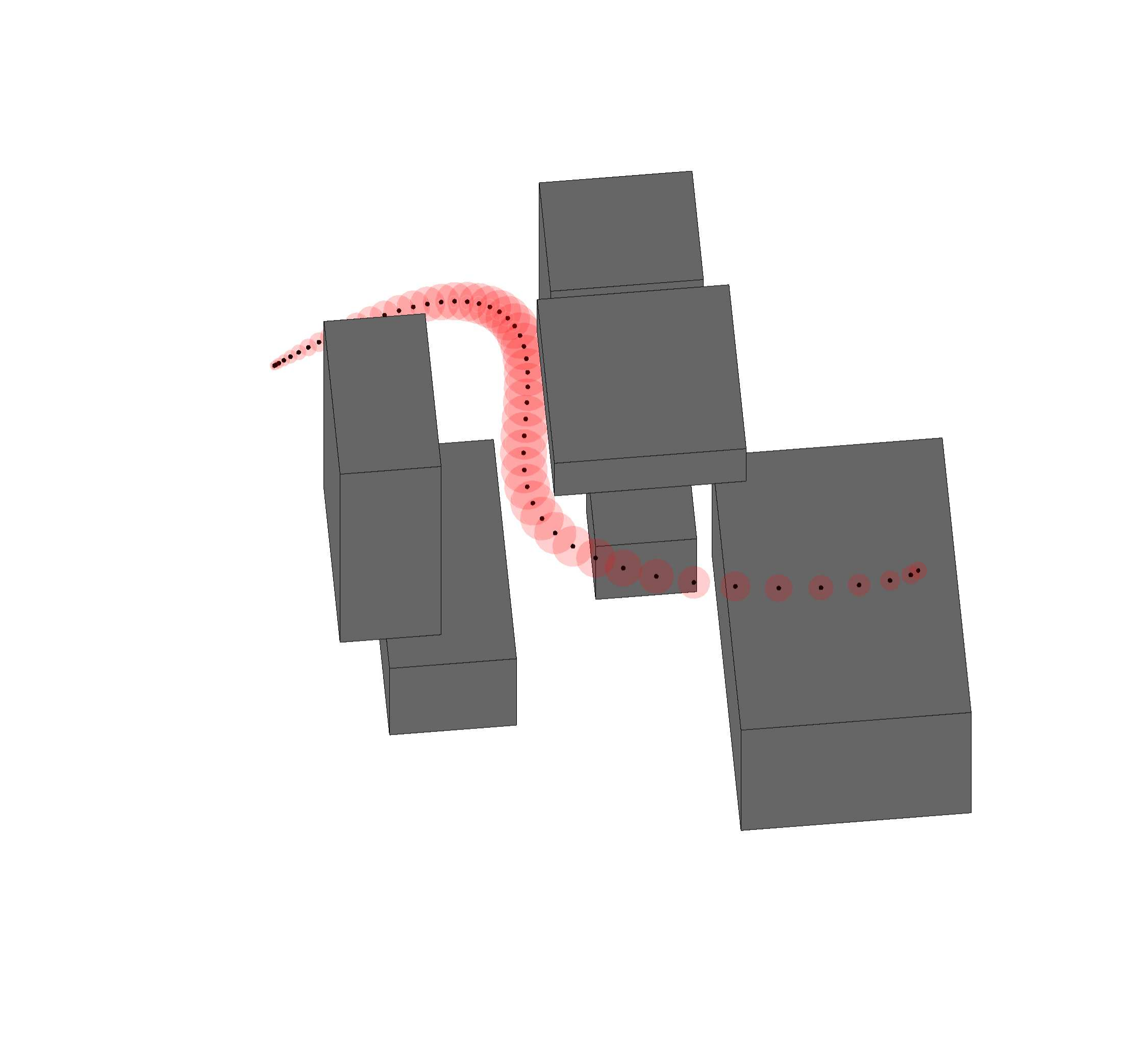}
    \end{subfigure}
    \hfill
    \begin{subfigure}[ht]{0.23\textwidth}
    \centering
    \includegraphics[width=0.95\textwidth]{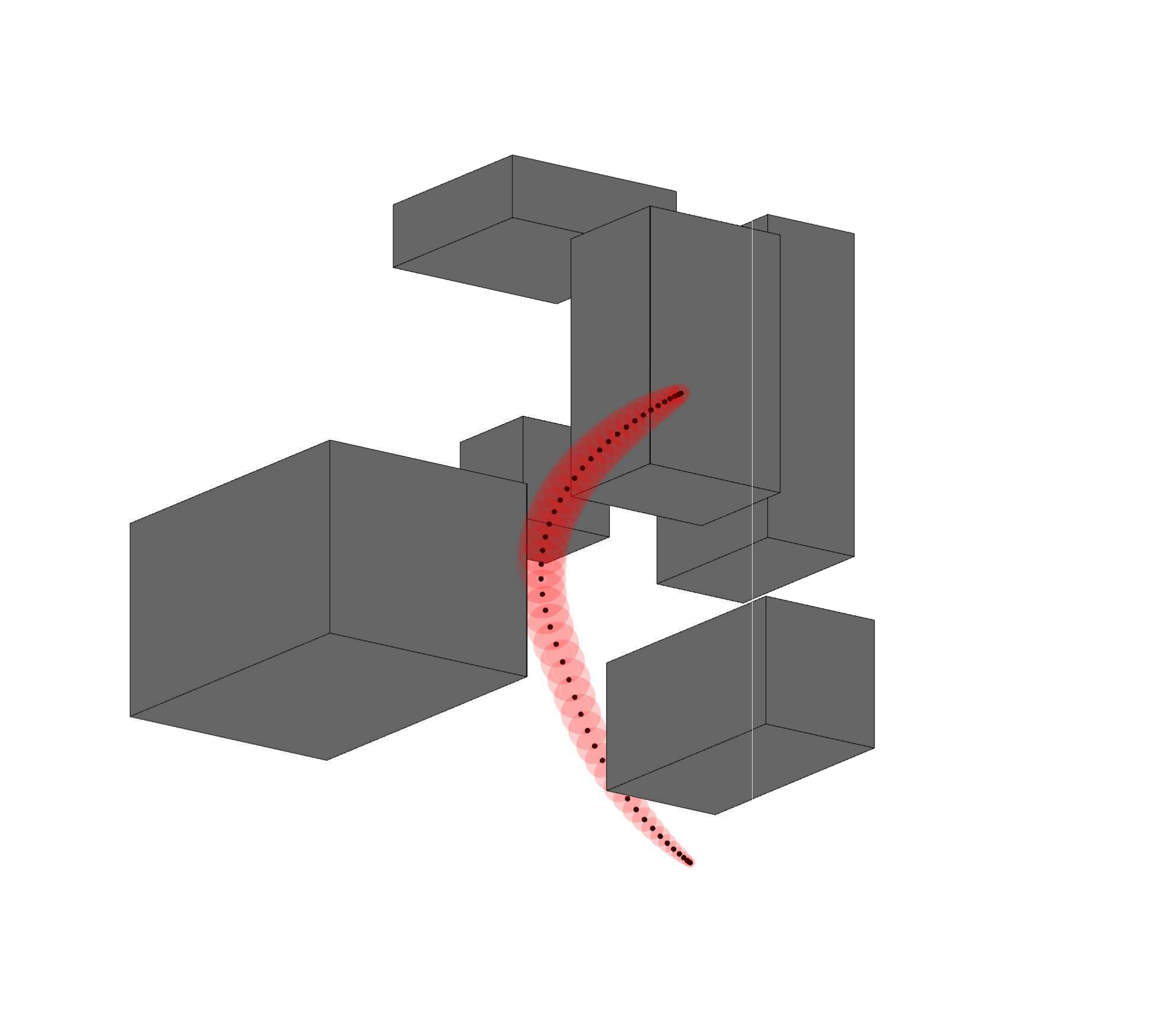}
    \end{subfigure}
    \hfill
    \begin{subfigure}[ht]{0.23\textwidth}
    \centering
    \includegraphics[width=0.95\textwidth]{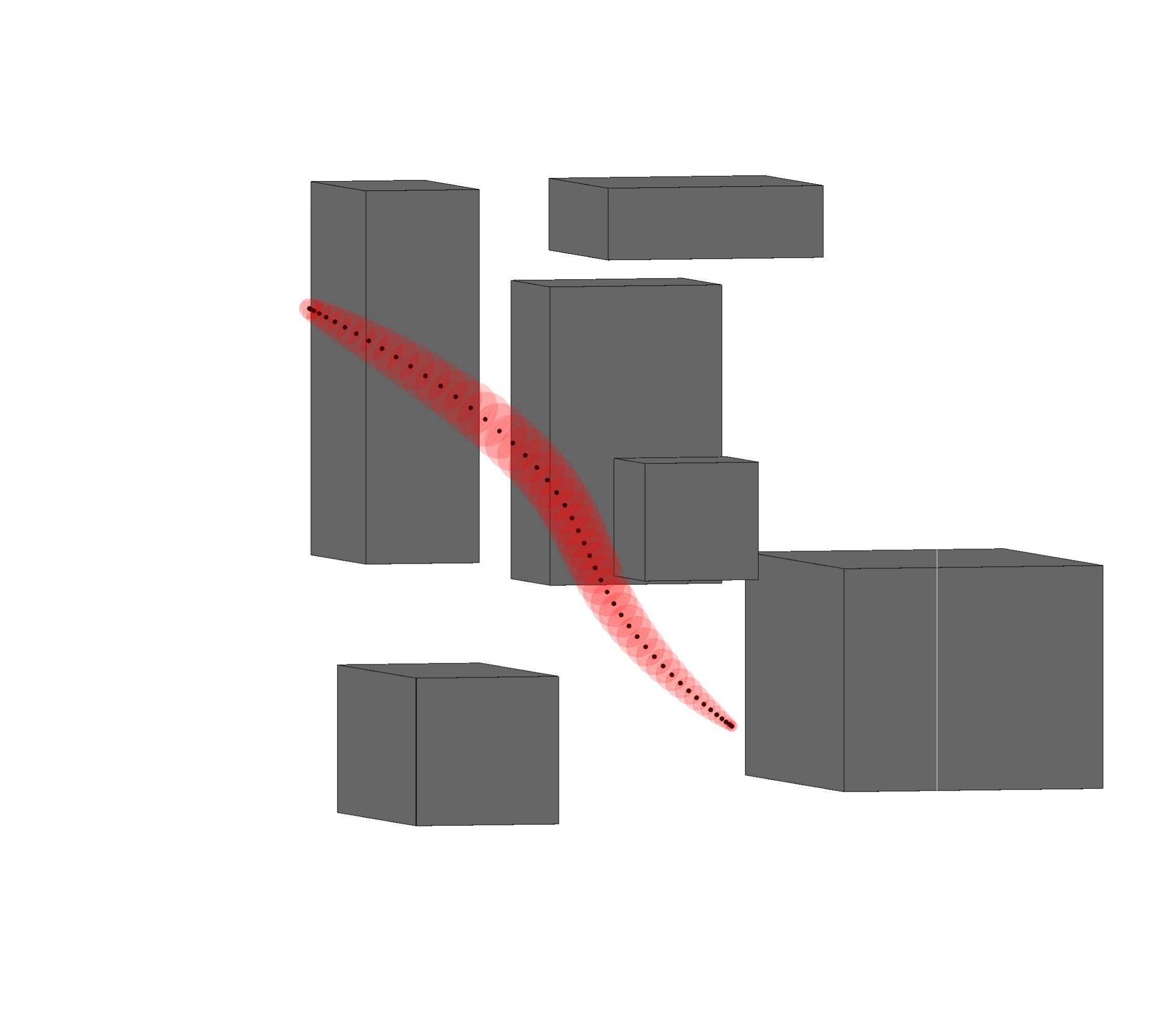}
    \end{subfigure}
  \caption{Motion planning for 3D robot, PCS-MP. $N=50$ support states are used for a time window $T = 12.5$. The covariance constraints are $K_0 = 0.01 I$, $K_T = 0.04 I$, radius $r=0.5$, $\epsilon_{\rm obs}=0.7$, $\Sigma_{\rm obs} = 7.5 \times 10^3 \sim 8 \times 10^3$.}
  \label{fig:3d_pR}
\end{figure*}
\begin{figure*}[h]
\centering
    \begin{subfigure}[ht]{0.15\textwidth}
    \centering
    \includegraphics[height=\textwidth]{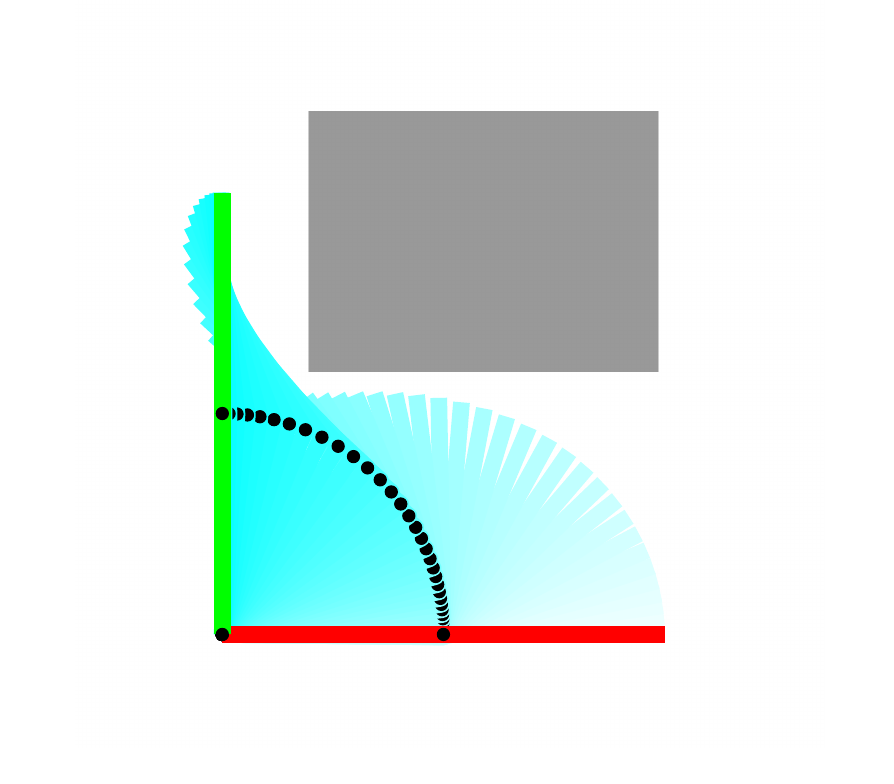}
    \caption{ GPMP2 Result }
    \label{fig:2_link_arm_1}
    \end{subfigure}
    \hfill
    \begin{subfigure}[ht]{0.15\textwidth}
    \centering
    \includegraphics[height=\textwidth]{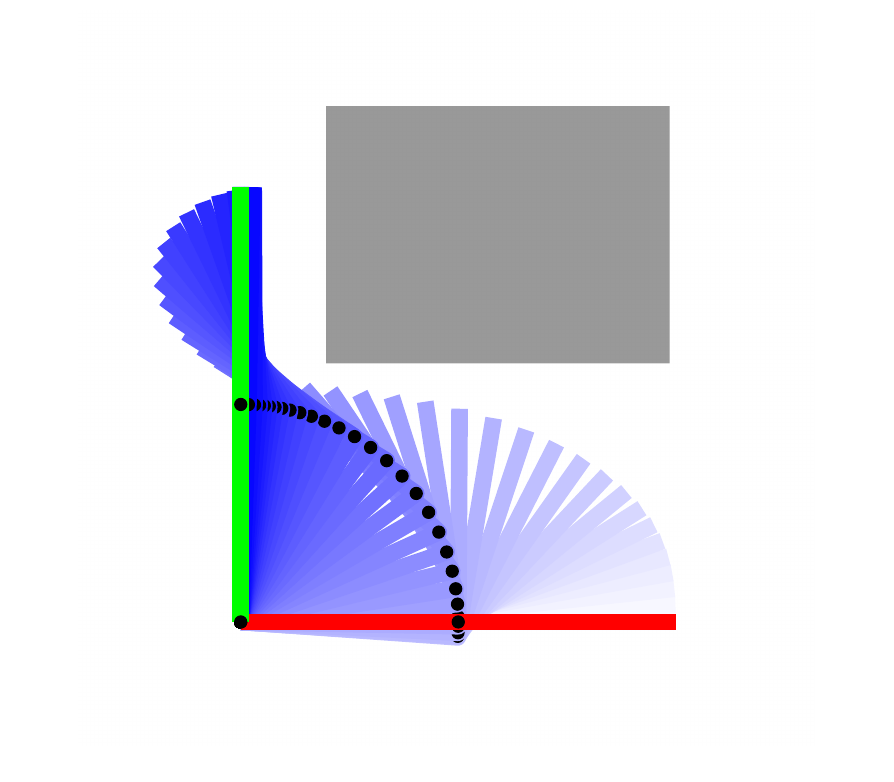}
    \caption{ GVI-MP Mean  }
         \label{fig:2_link_arm_2}
    \end{subfigure}
    \hfill
    \begin{subfigure}[ht]{0.15\textwidth}
    \centering
    \includegraphics[height=\textwidth]{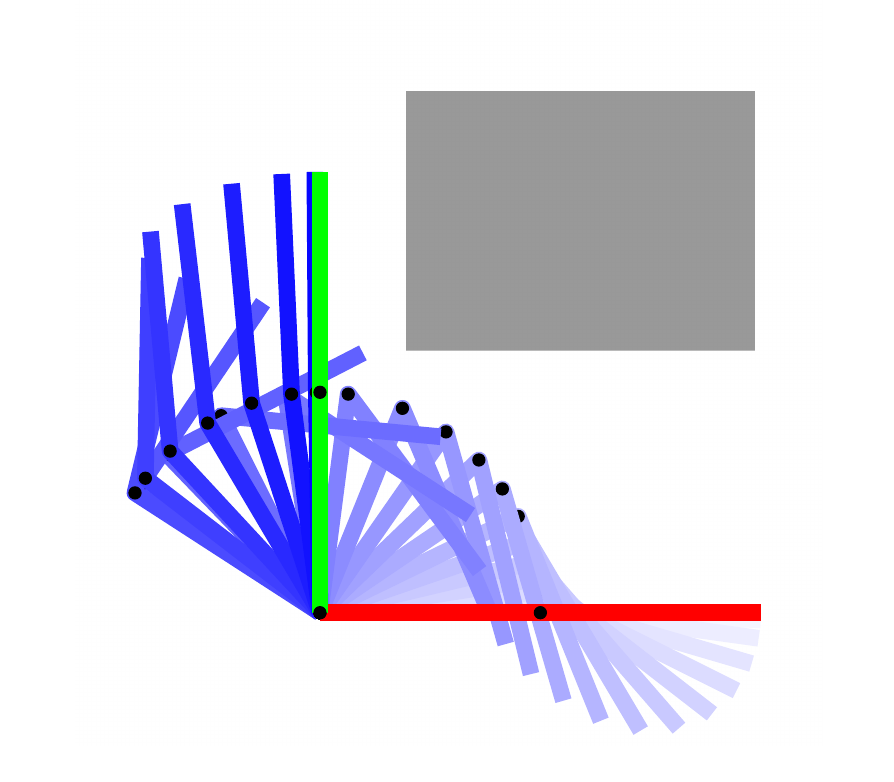}
    \caption{ PCS-MP Mean}
         \label{fig:2_link_arm_3}
    \end{subfigure}
    \hspace{0.05cm}
    \hfill
    \begin{subfigure}[ht]{0.17\textwidth}
    \centering
    \includegraphics[height=\textwidth]{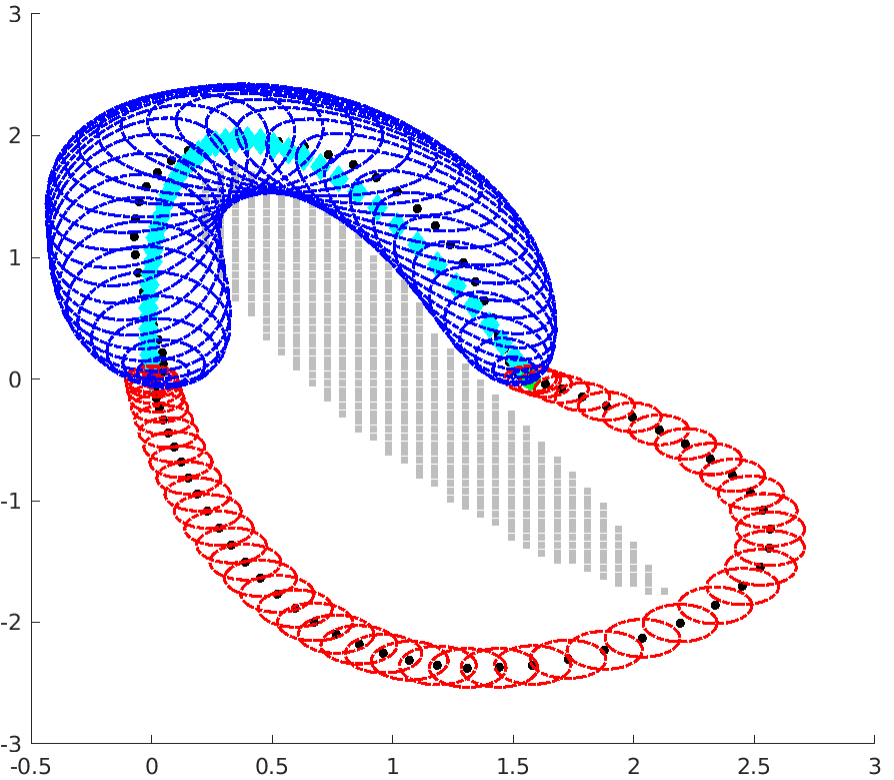}
    \caption{ Configuration space }
         \label{fig:2_link_arm_4}
    \end{subfigure}
    \hspace{0.05cm}
    \hfill
  \caption{Motion planning for a $2$-link arm. GVI-MP used $N=40$ support states for $T=4.0$, and $\Sigma_{\rm obs}=30.5 I$ for a radius $r+\epsilon_{\rm sdf}=0.13$. For PCS-MP, $N=50$ support states are used for $T=4.5$, and $\Sigma_{\rm obs}=2\times 10^4 I$ for a radius $r+\epsilon_{\rm sdf}=0.11$. The covariance constraints are $K_0=K_T=0.001 I$. Fig. \ref{fig:2_link_arm_4} shows the configuration space trajectories. Cyan diamond markers are the GPMP2 trajectory. Blue and red ellipsoids are the $3-\sigma$ covariances contours for GVI-MP and PCS-MP, respectively.
  }
  \label{fig:2_link_arm}
\end{figure*}

    

\subsection{Arm Robots}
\label{sec:experiments_armrobots}
{\em a) $2$-link Arm.}
We applied the two proposed algorithms on a $2$-DOF arm robot in a collision avoidance motion planning task. The state consists of the $2$D joint position and velocity $X_t = [p_x, p_y, v_x, v_y]^T.$
The start and goal states are horizontal and vertical poses, i.e., $X_0 = [0, 0, 0, 0], X_T = [0, \frac{\pi}{2}, 0, 0].$ We show the GPMP2, GVI-MP, and PCS-MP results in Fig. \ref{fig:2_link_arm}. 
\begin{figure*}
\centering
    \begin{subfigure}[ht]{0.225\textwidth}
    \centering
    \includegraphics[width=\textwidth]{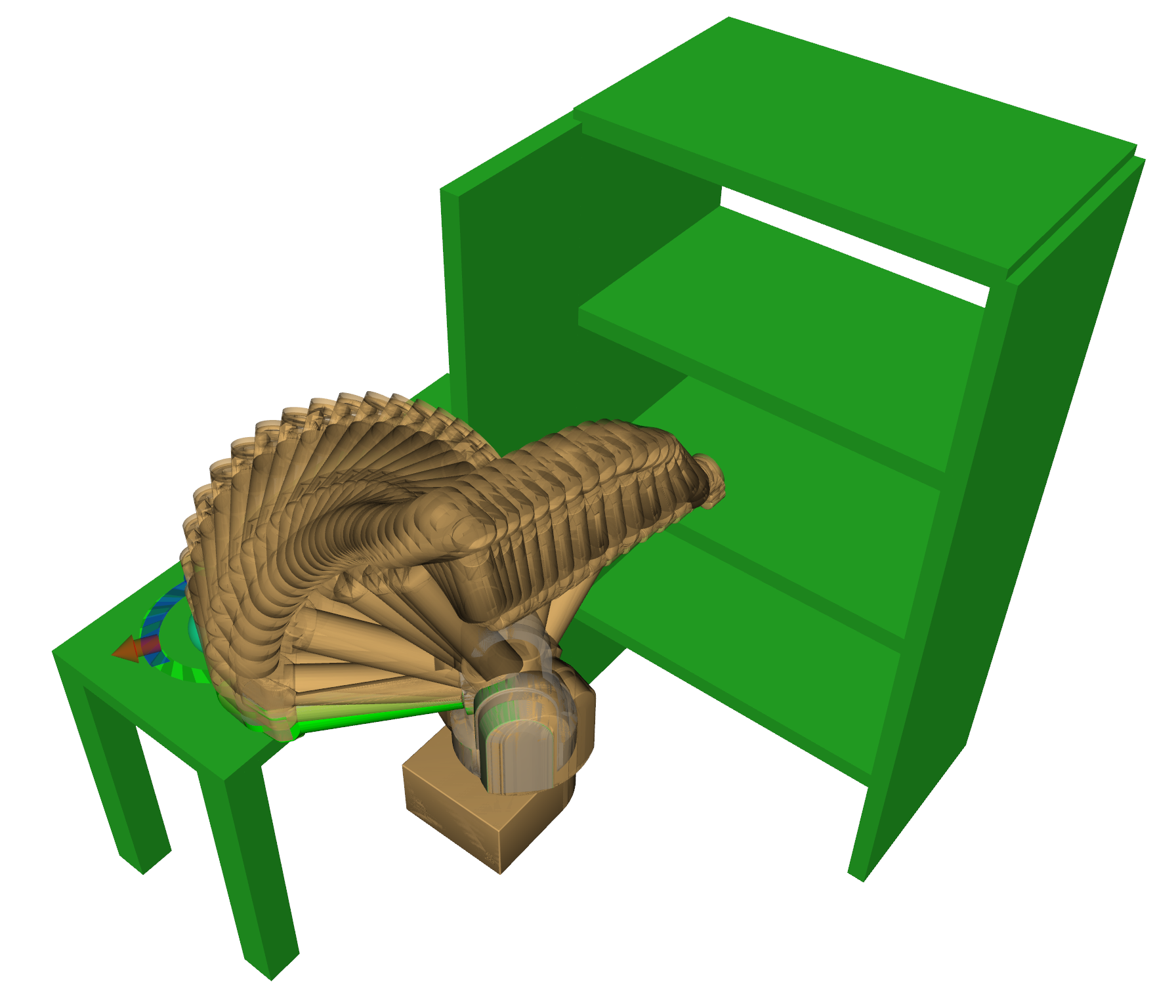}
    \caption{GPMP2}
    \label{fig:wam_gpmp2_trj1}
    \end{subfigure}
    \hfill
    \begin{subfigure}[ht]{0.225\textwidth}
    \centering
    \includegraphics[width=\textwidth]{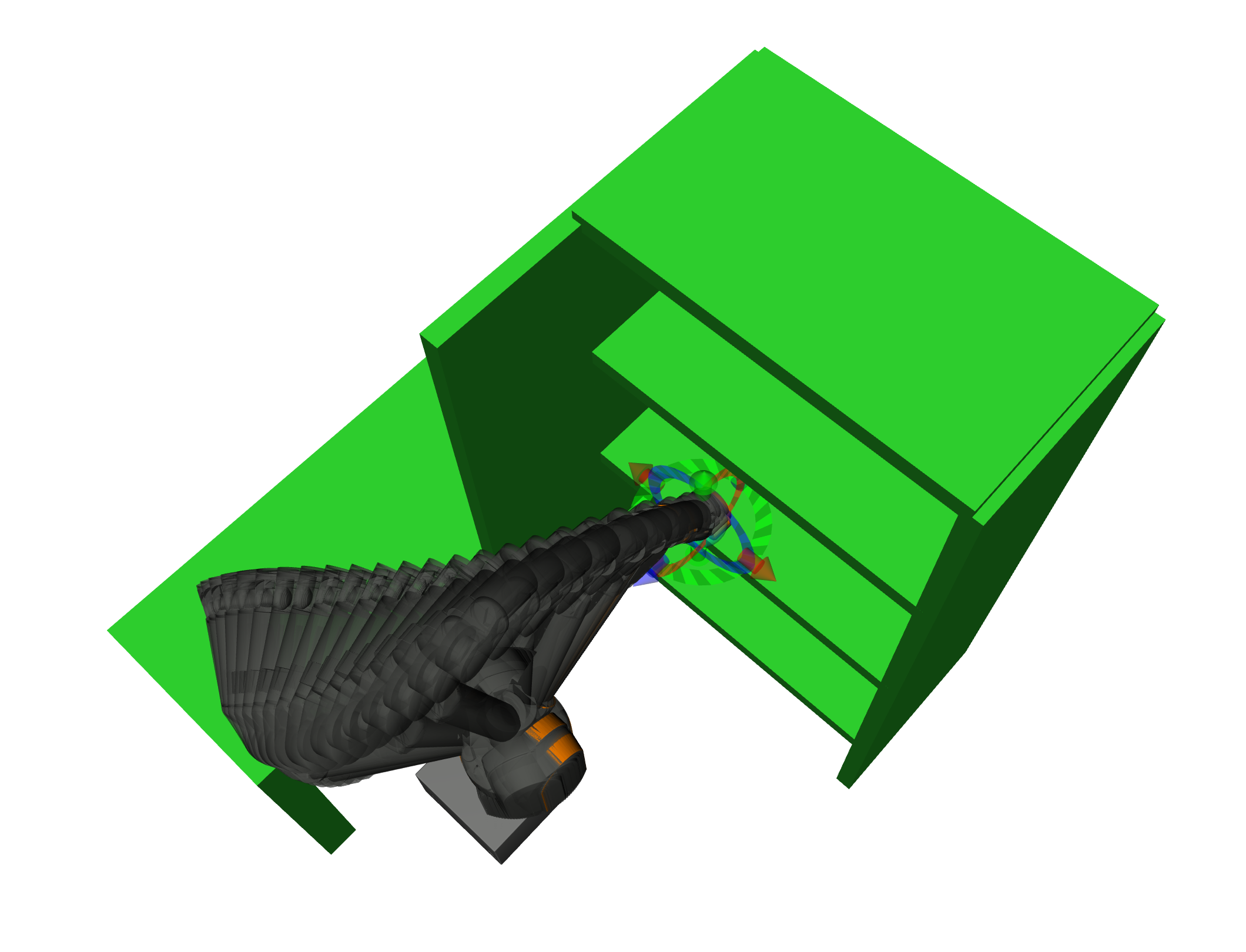}
    \label{fig:wam_gvimp_trj1}
    \caption{GVI-MP mean}  
    \end{subfigure}
    \hfill
    \begin{subfigure}[ht]{0.225\textwidth}
    \centering
    \includegraphics[width=\textwidth]{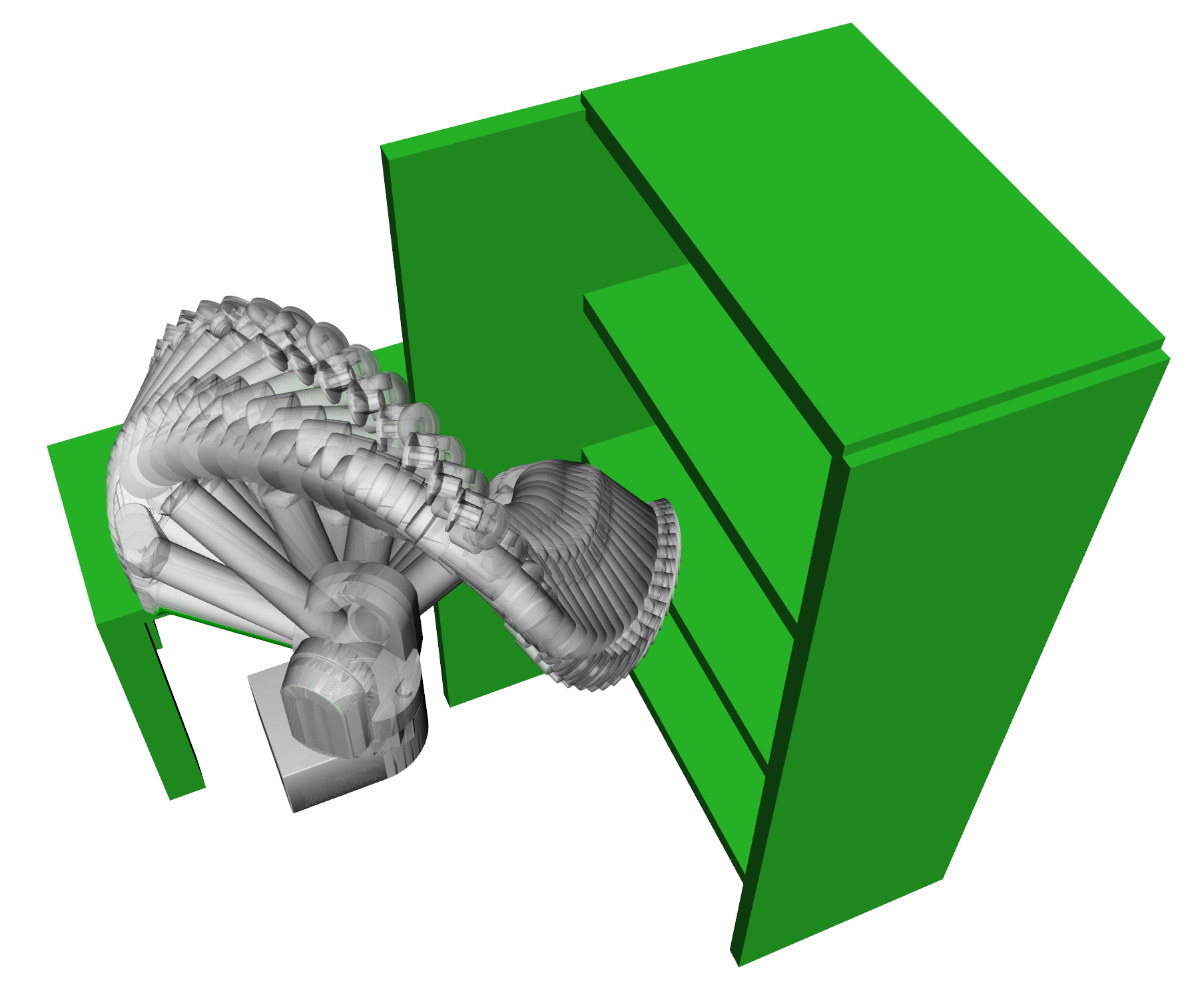}
    \caption{PCS-MP mean}
    \label{fig:wam_pgcs_trj1}
    \end{subfigure}

    \begin{subfigure}[ht]{0.225\textwidth}
    \centering
    \includegraphics[width=\textwidth]{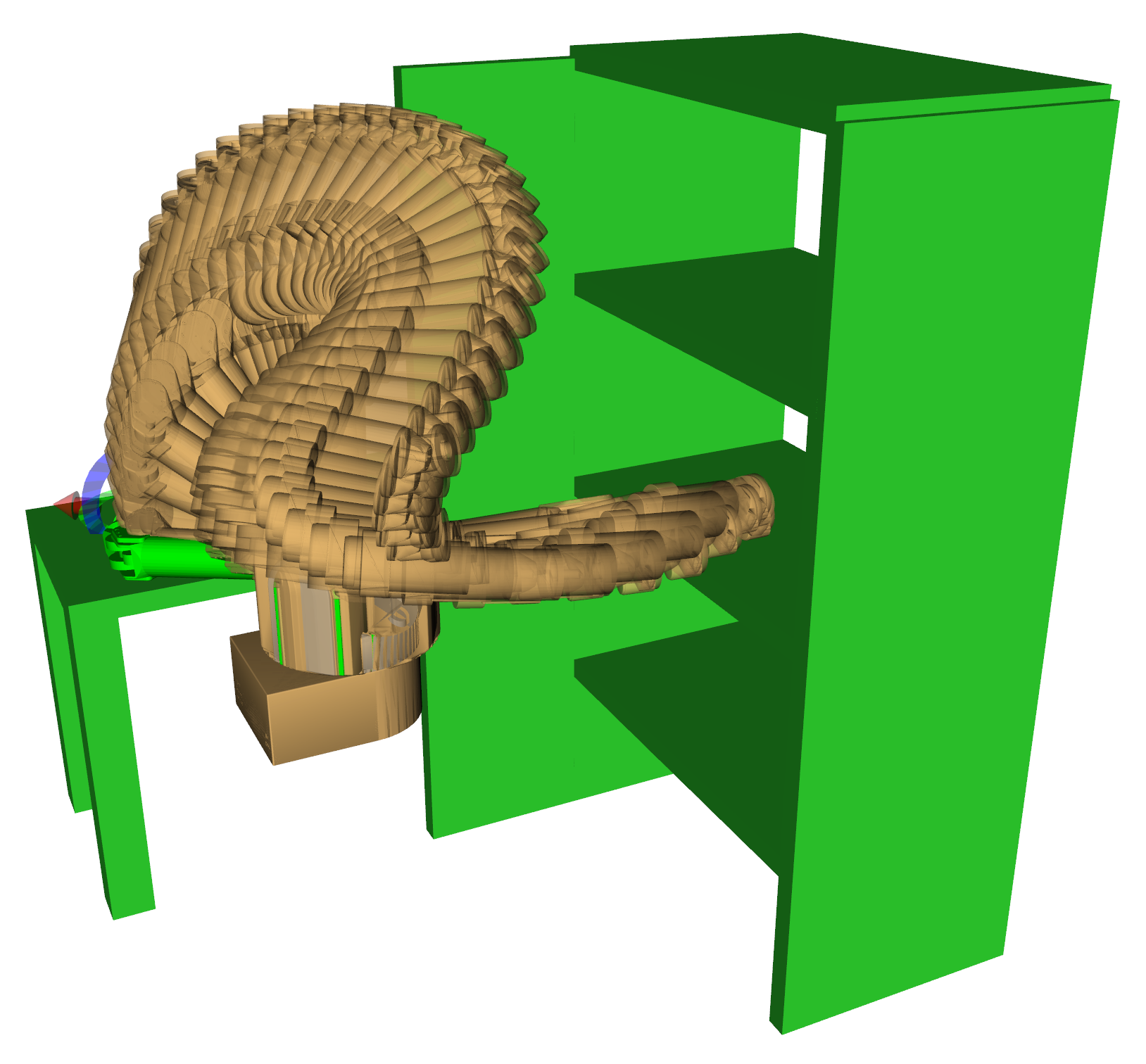}
    \caption{GPMP2}
    \label{fig:wam_gpmp2_trj2}
    \end{subfigure}
    \hfill
    \begin{subfigure}[ht]{0.225\textwidth}
    \centering
    \includegraphics[width=\textwidth]{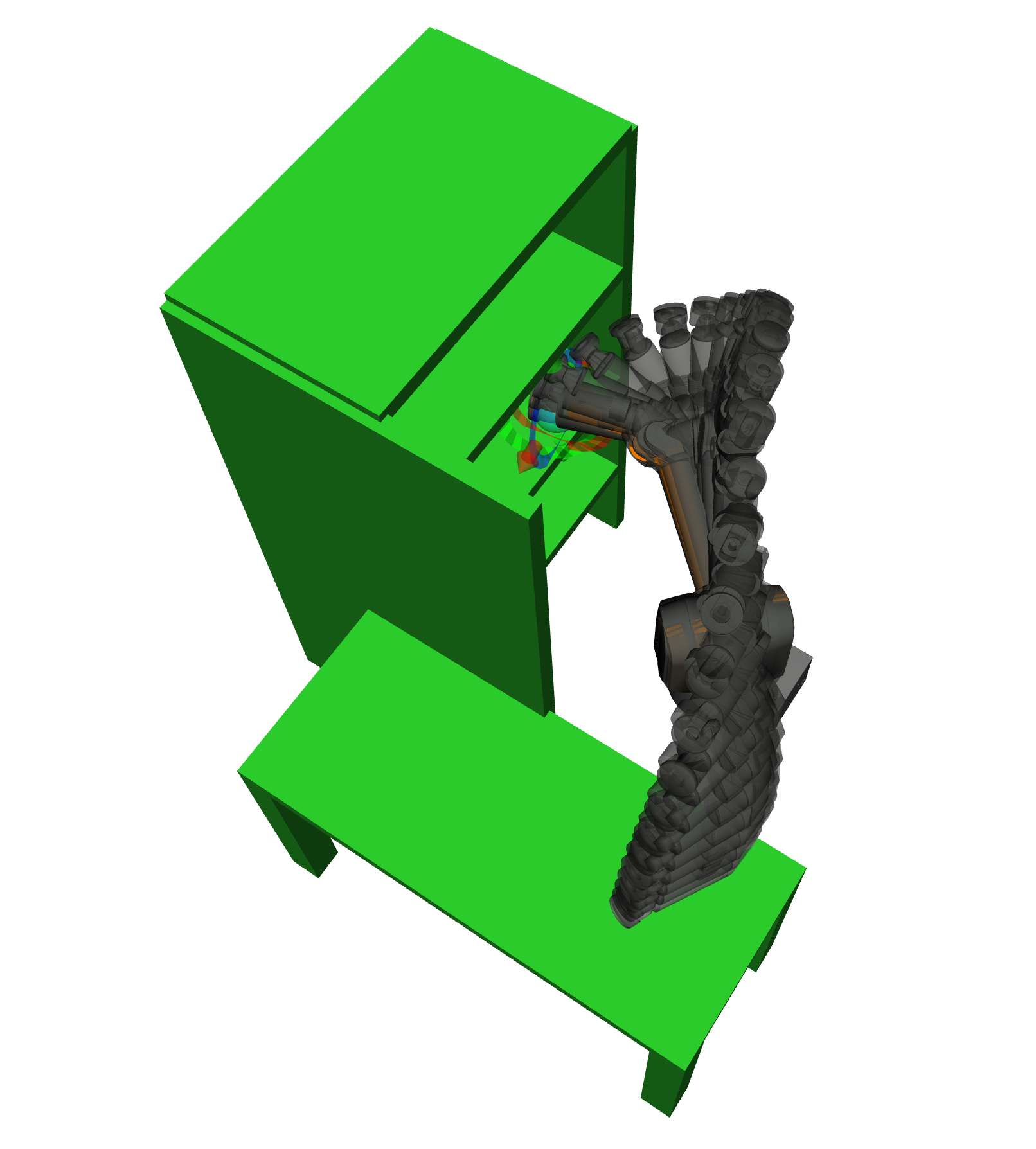}
    \caption{GVI-MP mean}
    \label{fig:wam_gvimp_trj2}
    \end{subfigure}
    \hfill
    \begin{subfigure}[ht]{0.225\textwidth}
    \centering
    \includegraphics[width=\textwidth]{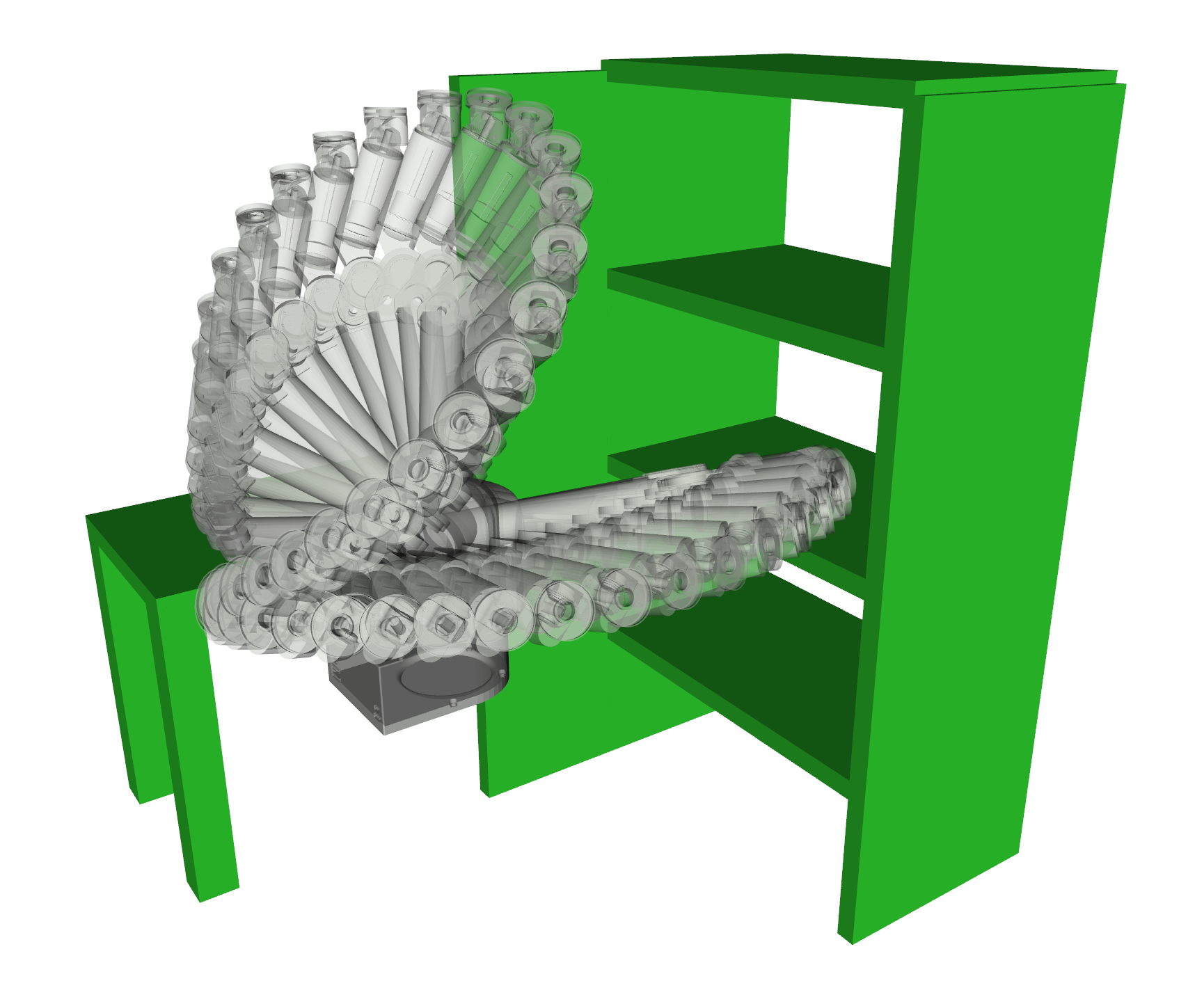}
    \caption{PCS-MP mean}
    \label{fig:wam_pgcs_trj2}
    \end{subfigure}
    \hfill
   
  \caption{Motion planning for a $7$-DOF WAM Arm for two tasks, one on each row. Results are visualized using a WAM robot URDF in \textit{ROS Moveit} \cite{coleman2014reducing} environment and are compared with the GPMP2 baseline. For GVI-MP, $N=30$ and $T=2.0$. $\Sigma_{\rm obs}=21.0 I$ for radius $r+\epsilon_{\rm sdf}=0.21$. The ratio between temperatures are $\hat{T}_{\rm high} / \hat{T}_{\rm low} = 8.0 \sim 10.0$. For PCS-MP, $N=50$ and $T=9.0$. $\Sigma_{\rm obs} = 4 \times 10^2 \sim 2.5 \times 10^3$ for radius $r+\epsilon_{\rm sdf}=0.18$. The covariance constraints are $K_0 = K_T = 0.001 I.$}
  
  \label{fig:wam}
\end{figure*}
\begin{figure*}
    \begin{subfigure}{\textwidth}
    \centering
    \includegraphics[width=0.8\linewidth]{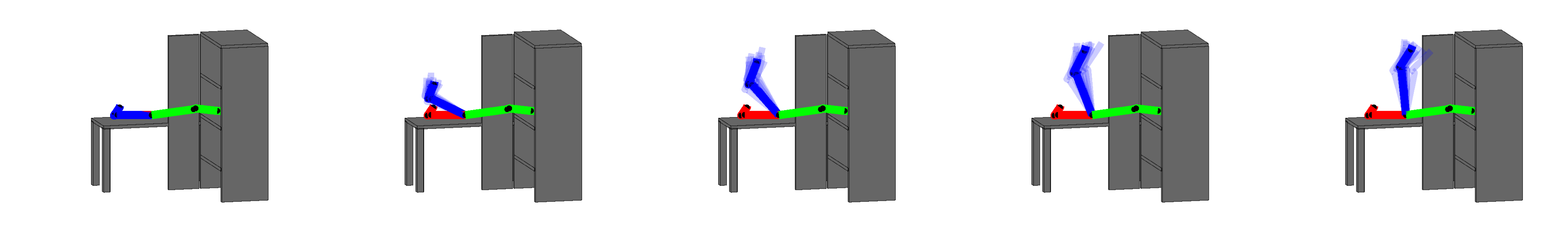}
    \end{subfigure}
    \hfill
    \begin{subfigure}{\textwidth}
    \centering
    \includegraphics[width=0.8\linewidth]{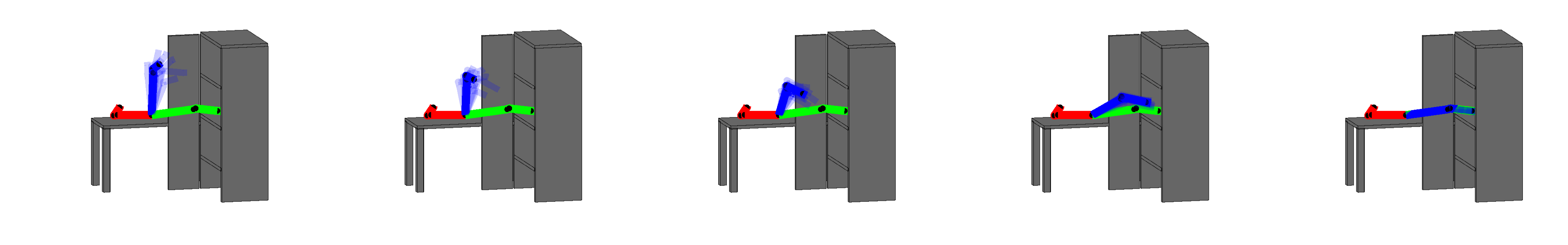}
    \caption{GVI-MP mean and sampled trajectories.}
    \label{fig:wam_gvimp_samples}
    \end{subfigure}
    \begin{subfigure}{\textwidth}
    \centering
    \includegraphics[width=0.8\linewidth]{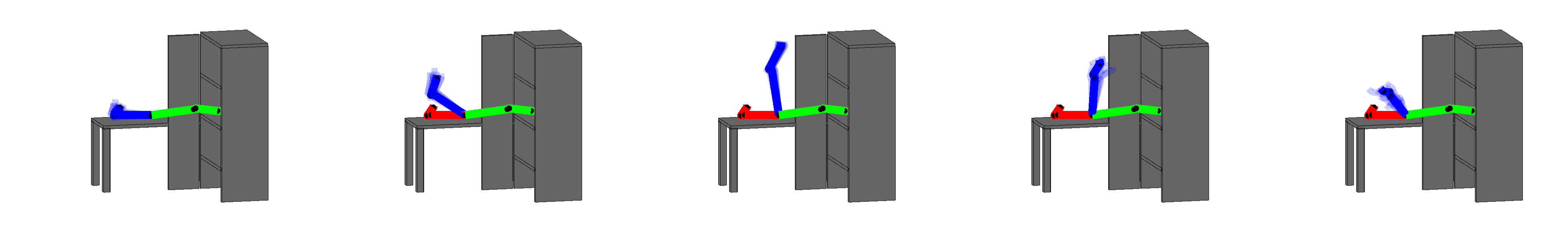}
    \end{subfigure}
    \hfill
    \begin{subfigure}{\textwidth}
    \centering
    \includegraphics[width=0.8\linewidth]{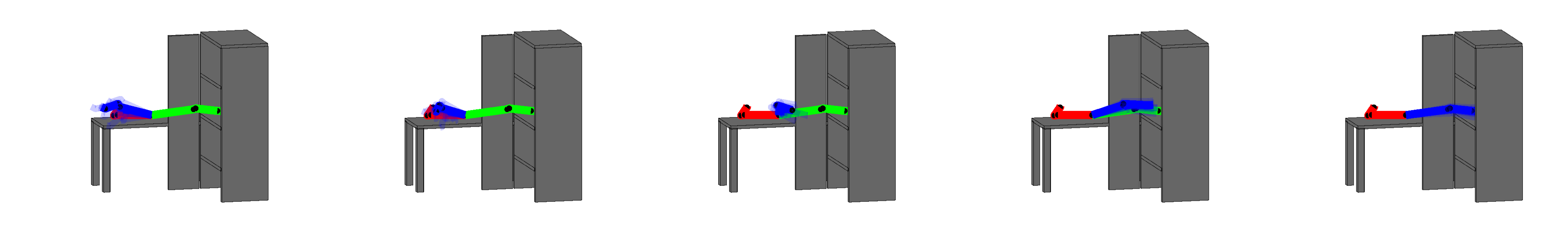}
    \caption{PCS-MP mean and sampled trajectories.}
    \label{fig:wam_pgcsmp_samples}
    \end{subfigure}
    \caption{Mean and sampled trajectories of the second task in Fig. \ref{fig:WAM_problem_settings}. The start configurations are marked as red and the goals are in green. Solid blue bars represent the mean trajectories, and light blue bars represent the samples.}
    \label{fig:wam_samples}
\end{figure*}

\begin{figure}[H]
\centering
    \begin{subfigure}[ht]{0.24\textwidth}
    \centering
    \includegraphics[width=\textwidth]{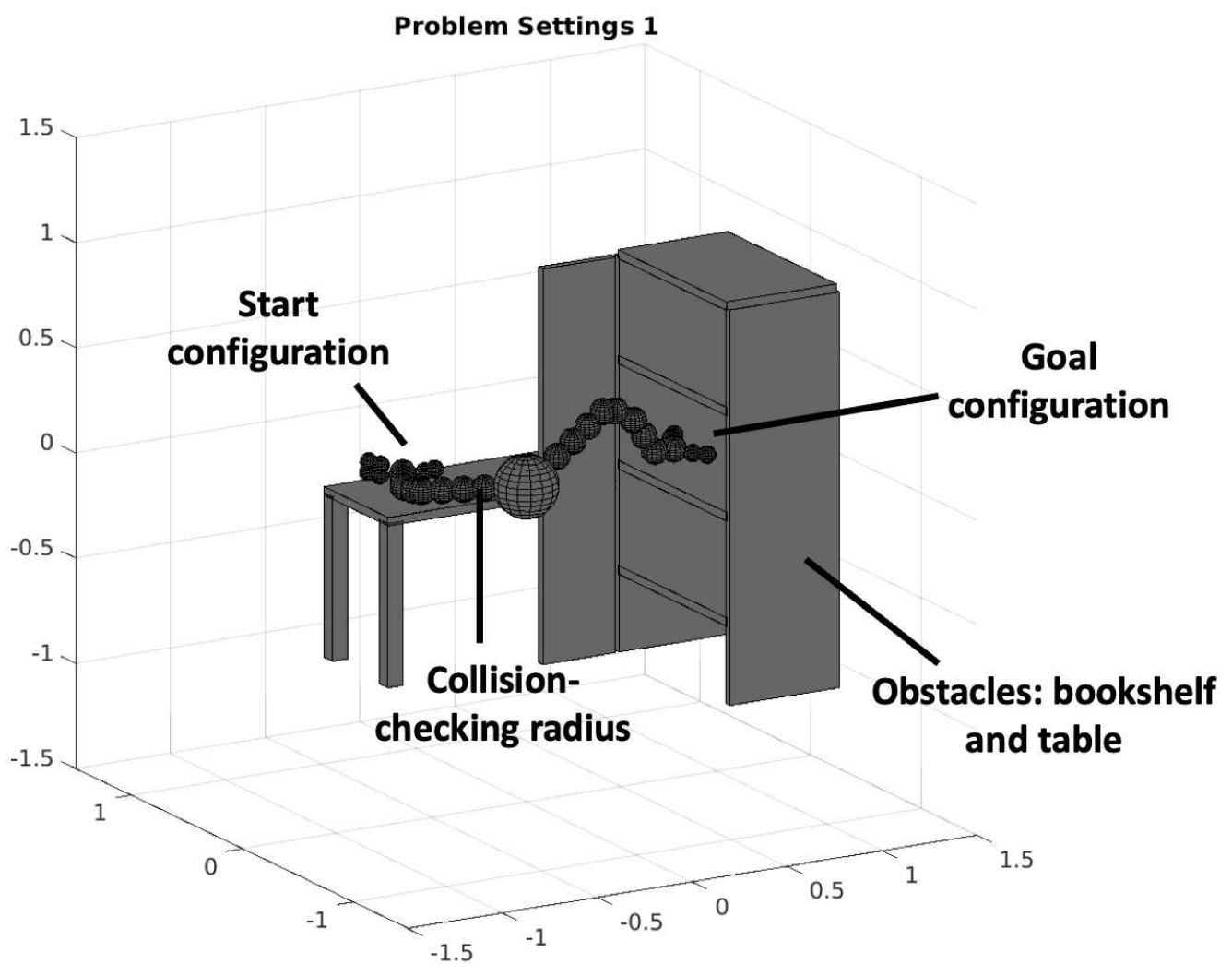}
    \caption{Task setting $1$.}
    \label{fig:wam_problem1}
    \end{subfigure}
    \begin{subfigure}[ht]{0.24\textwidth}
    \centering
    \includegraphics[width=\textwidth]{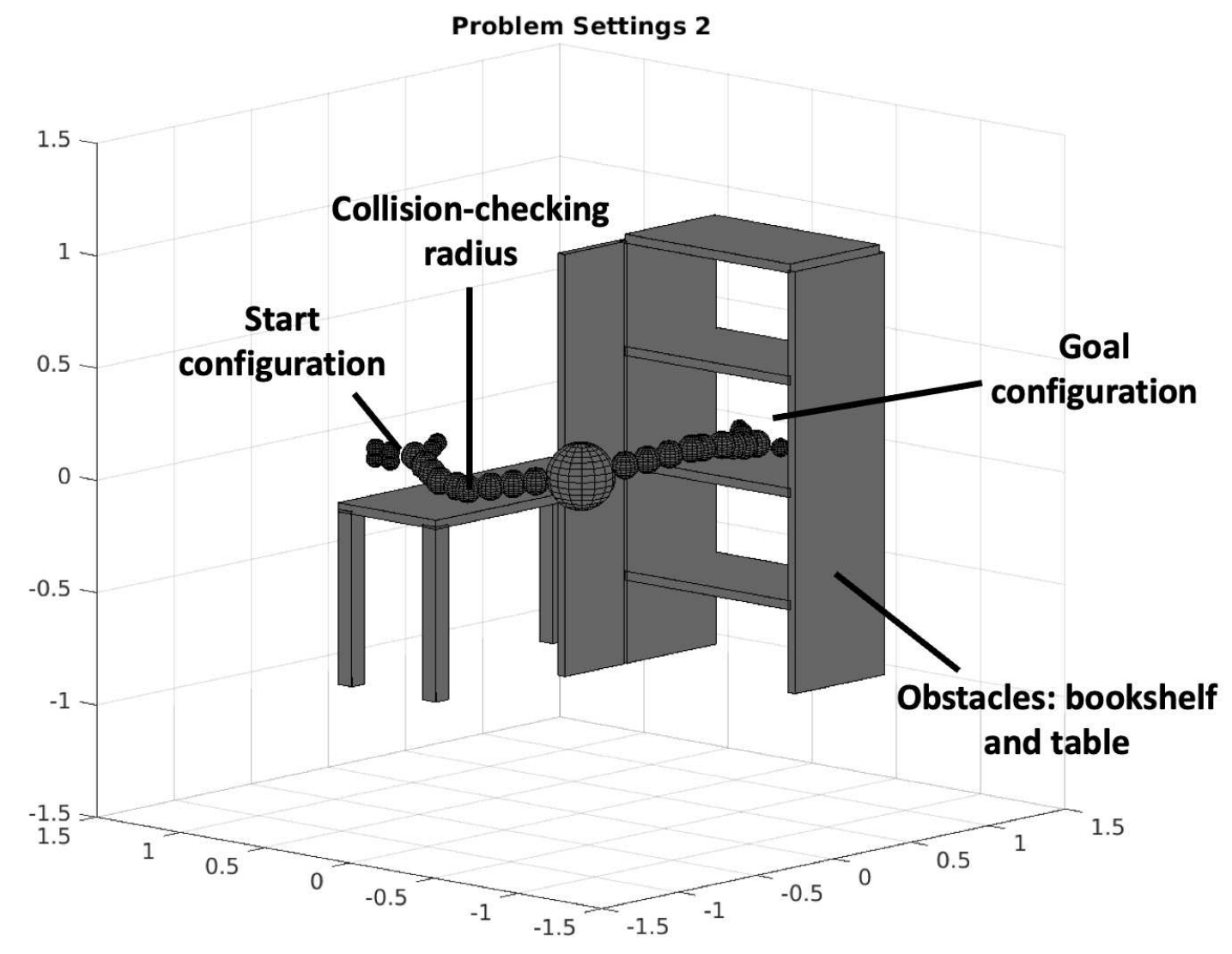}
    \caption{Task setting $2$.}
    \label{fig:wam_problem2}
    \end{subfigure}
    \caption{Problem settings for the $7$-DOF WAM robot arm experiment. The black balls represent the collision-checking radius on the robot.}
    \label{fig:WAM_problem_settings}
\end{figure}

{\em b) $7$-DOF WAM Arm.}
Next, we will experiment in a bookshelf environment with a $7$-DOF WAM robot arm \cite{rooks2006harmonious}. The state consists of the $7$D joint position and velocity. We did $2$ experiments with different start and goal configurations. The robot arm's end-effector starts on the table and moves towards a goal state inside the bookshelf while avoiding colliding with the bookshelf. The start and goal velocities are all set to zero. The problem settings are shown in Fig. \ref{fig:WAM_problem_settings}. Fig. \ref{fig:wam_samples} shows the planned mean trajectories and sampled support states from the obtained trajectory distributions. We then visualize the results in the Robot Operating System (ROS) \cite{quigley2009ros} with a high-fidelity robot model for the WAM Arm integrated into the \textit{Moveit} \cite{coleman2014reducing} motion planning package. The results are shown in Fig. \ref{fig:wam}. 

\subsection{Algorithm Performance}
\label{sec:experiments_performance}
We assess the efficiency of our implementations through experiments. The computation efficiency is improved by leveraging the factorized structure (c.f., Section \ref{sec:GVIMP_algorithm_factorized_optimization}), and the closed-form solution both for GVI-MP and for PCS-MP. 
\begin{figure*}[ht]
\centering
    \begin{subfigure}[ht]{0.22\textwidth}
    \centering
    \label{fig:pgcs_pquad_4exp_1}
    \centering
    \includegraphics[height=\linewidth]{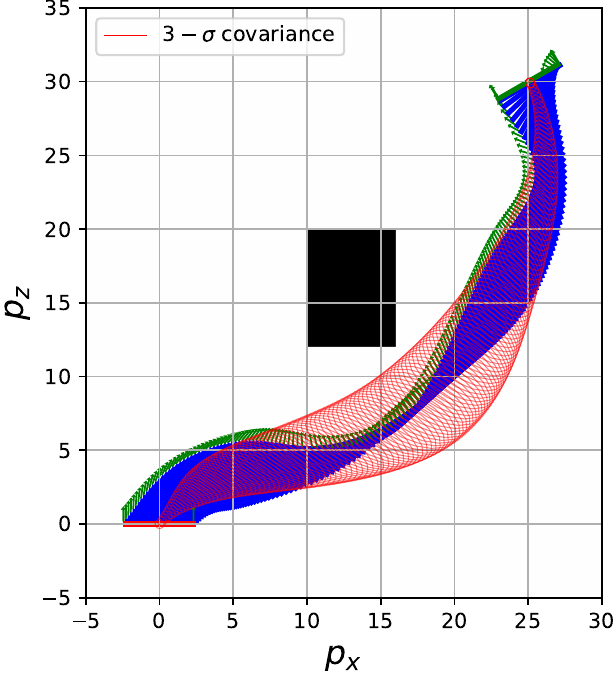}
    \caption{Experiment $1$.}
    \end{subfigure}
    \hfill
    \begin{subfigure}[ht]{0.22\textwidth}
    \centering
    \includegraphics[height=\linewidth]{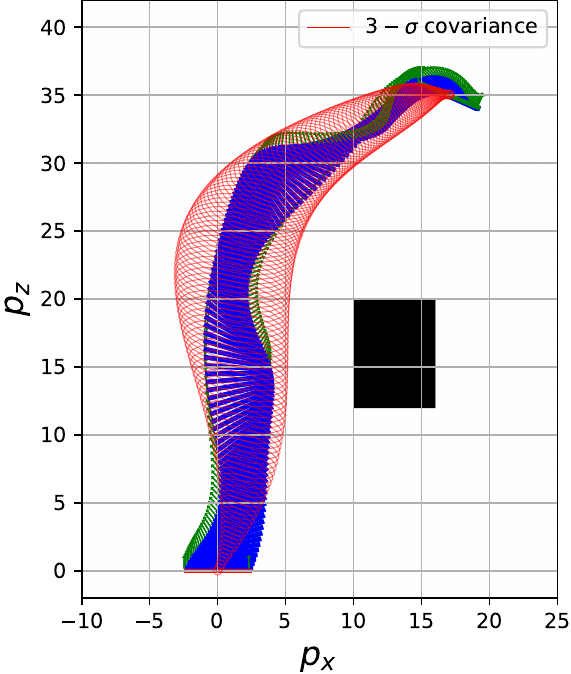}
    \caption{Experiment $2$.}
    \label{fig:pgcs_pquad_4exp_2}
    \end{subfigure}
    \hfill
    \begin{subfigure}[ht]{0.22\textwidth}
    \centering
    \includegraphics[height=\linewidth]{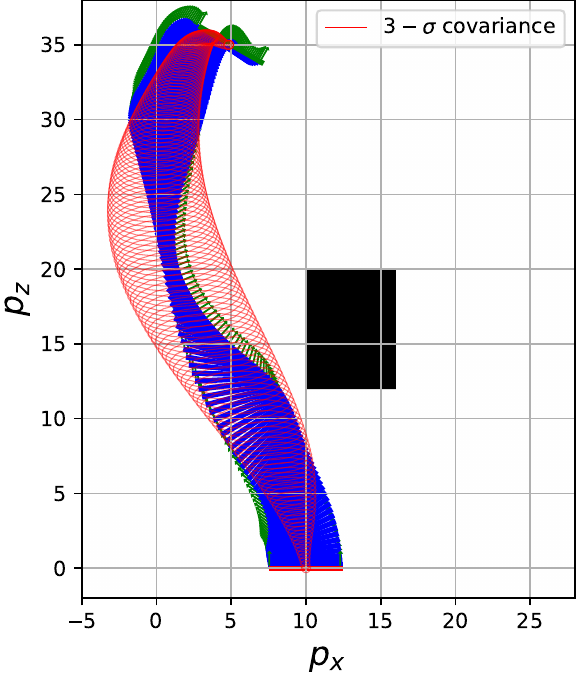}
    \caption{Experiment $3$.}
    \label{fig:pgcs_pquad_4exp_3}
    \end{subfigure}
    \hfill
    \begin{subfigure}[ht]{0.22\textwidth}
    \centering
    \includegraphics[height=\textwidth]{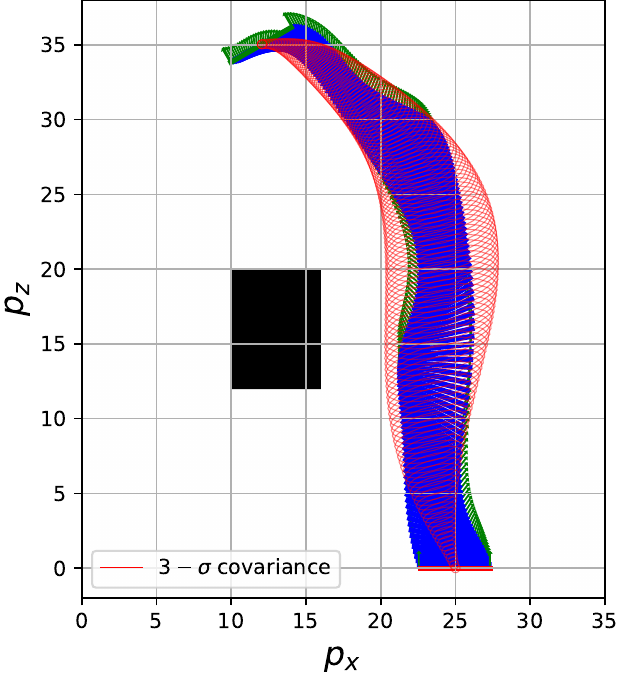}
    \caption{Experiment $4$.}
    \label{fig:pgcs_pquad_4exp_4}
    \end{subfigure}
    \hfill
  \caption{Stochastic motion planning for a linearized time-varying planar quadrotor dynamics \eqref{eq:planar_quad}, PCS-MP. The red ellipsoids are the trajectory of the $3-\sigma$ area of the covariances between positions $p_x$ and $p_z$. $N=3000$ support states are used for $T=4.0\sim4.5$, and $\Sigma_{\rm obs}=50\sim85 I$ for radius $r+\epsilon_{\rm sdf}=1.8\sim2.0$. The covariance constraints are $K_0 = K_T = 0.01 I$.}
  \label{fig:pgcs_pquad_4exp}
\end{figure*}
\begin{table*}[ht]
    \centering
    \begin{tabular}[width=\textwidth]{|c|c|c|c|}
    \hline
    & {\rm Closed-form Prior \eqref{eq:GVI_derivatives_prior}}  & {\rm Sparse-grid GH} & {\rm Full-grid GH} 
    \\
    \hline
    {\rm $2$ DOF Arm}  & $\textbf{0.0150}$ & $0.0180$ & 0.0954  
    \\
   \hline
   {\rm $7$ DOF WAM exp 1}  & $\textbf{0.3086}$ & - & $\infty$  
   \\
    \hline
\end{tabular}
\caption{Implementation time comparison for computing prior costs in GVI-MP, averaged over 50 runs. The trajectory consists of 50 support states. For a 7-DOF robot, a $3$-degree sparse-grid (resp., full-grid) quadrature method requires $28^3 = 21,952$ (resp., $3^{28}\approx 10^{13}$) sigma points to evaluate one expectation. The closed-form expression \eqref{eq:GVI_derivatives_prior} is thus indispensable in GVI-MP.}
\label{tab:computation_time_GVI_prior}
\end{table*}

{\em a) Closed-form prior factors in GVI-MP.} We show by experiment that the prior closed-form solution (Lemma \ref{lem:closed_form_prior}) is indispensable for GVI-MP. We record the computation time comparison for the $7$-DOF WAM Robot arm experiment in Tab. \ref{tab:computation_time_GVI_prior}. This is because, for the GH quadratures, to compute the factor level prior costs \eqref{eq:prior_factors_2}, quadrature rules are evaluated on the marginal Gaussian $q_{i, i+1}$ which has dimension $28$. A $3$-degree full-grid GH quadrature requires $3^{28}$ sigma points. For a sparse-grid GH quadrature, $28^3$ sigma points are needed. The quadrature methods are thus expensive for evaluating the prior. We use the open-sourced sparse-GH quadrature library \href{http://www.sparse-grids.de}{http://www.sparse-grids.de} to evaluate the expectations. Common implementations of quadrature methods consider the problem of dimension less than $20$.
\begin{table*}[ht]
\centering
\begin{tabular}[width=\textwidth]{|c|cc|cc|cc|}
\hline
 & \multicolumn{2}{c|}{GPMP2}    & \multicolumn{2}{c|}{GVI-MP}    & \multicolumn{2}{c|}{PCS-MP}    \\ \hline
 & \multicolumn{1}{c|}{\rm Computation} & {\rm Samping} & \multicolumn{1}{c|}{\rm Computation} & {\rm \textbf{Samping}} & \multicolumn{1}{c|}{\rm Computation} & {\rm \textbf{Samping}} \\ \hline
WAM exp 1 & \multicolumn{1}{c|}{$\textbf{0.2009}$} & - & \multicolumn{1}{c|}{$14.4028$} & $0.009$ & \multicolumn{1}{c|}{$\textbf{0.6311}$} & $0.003$ \\ \hline
WAM exp 2 & \multicolumn{1}{c|}{$\textbf{0.1705}$} & - & \multicolumn{1}{c|}{$24.5887$} & $0.003$ & \multicolumn{1}{c|}{$\textbf{0.5470}$} & $0.002$ \\ \hline
\end{tabular}
\caption{The computation and sampling time for different algorithms in the $7$-DOF experiments averaged over $50$ runs of the same experiment. The sampling time evaluates the time for sampling $1000$ trajectories from the joint Gaussian distributions. }
\label{tab:computation_time}
\end{table*}

{\em b) Implementation Time Comparison.}
We record the performance of the two proposed algorithms in the $7$-DOF WAM experiments compared to the GPMP2 baseline in Tab. \ref{tab:computation_time}. We observe that the computational time for GPMP2 is much faster than GVI-MP's. This is because even though GVI-MP leverages the factorized factors and the prior close-form solutions, it optimizes directly over the joint distribution, which is a much larger problem. In the $7$-DOF example, for $50$ support states, GPMP2 has $700=14\times50$ variables while GVI-MP has $10500 = 14\times14\times50+14\times50$.

The PCS-MP is much more efficient and comparable to the GPMP2 baseline method in terms of computation time, considering the scale of the problem it solves. This is because in PCS-MP, the trajectory distribution is parameterized using the underlying LTV-SDE system $(A_t, a_t, B_t)$, rather than the parameterization of the joint support state distribution. Another reason is the approximation of the nonlinear state cost using a quadratic function in PCS-MP, as opposed to the GH-quadrature estimation employed in GVI-MP. PCS-MP leverages this local approximation around a nominal trajectory and solves a linear covariance steering problem with a closed-form solution at each iteration.

\subsection{Robust Motion Plan and Real-world Impacts}
\label{sec:experiment_robustness}
{\em a) Trajectory distribution and samples.}
The stochastic formulation in GVI-MP and PCS-MP intrinsically embeds robustness against uncertainties by considering the noise in the systems. The output of GVI-MP and PCS-MP is a trajectory distribution from which a sample trajectory realization is easily obtained. Fig. \ref{fig:wam_samples} shows the mean and sampled trajectories from GVI-MP and PCS-MP distributions. We record the sampling time from the joint Gaussian distributions for the 7-DOF WAM robot in Table \ref{tab:computation_time}. Notably, sampling $1000$ trajectories consumes a negligible amount of time. The optimized covariances can be leveraged to resample trajectories that can help avoid obstacles with disturbances. Fig. \ref{fig:resample_2D} showcases the resampling results of a 2-D point robot in environments with disturbed obstacles at execution time.

\begin{figure}
    \centering
    \includegraphics[width=0.8\linewidth]{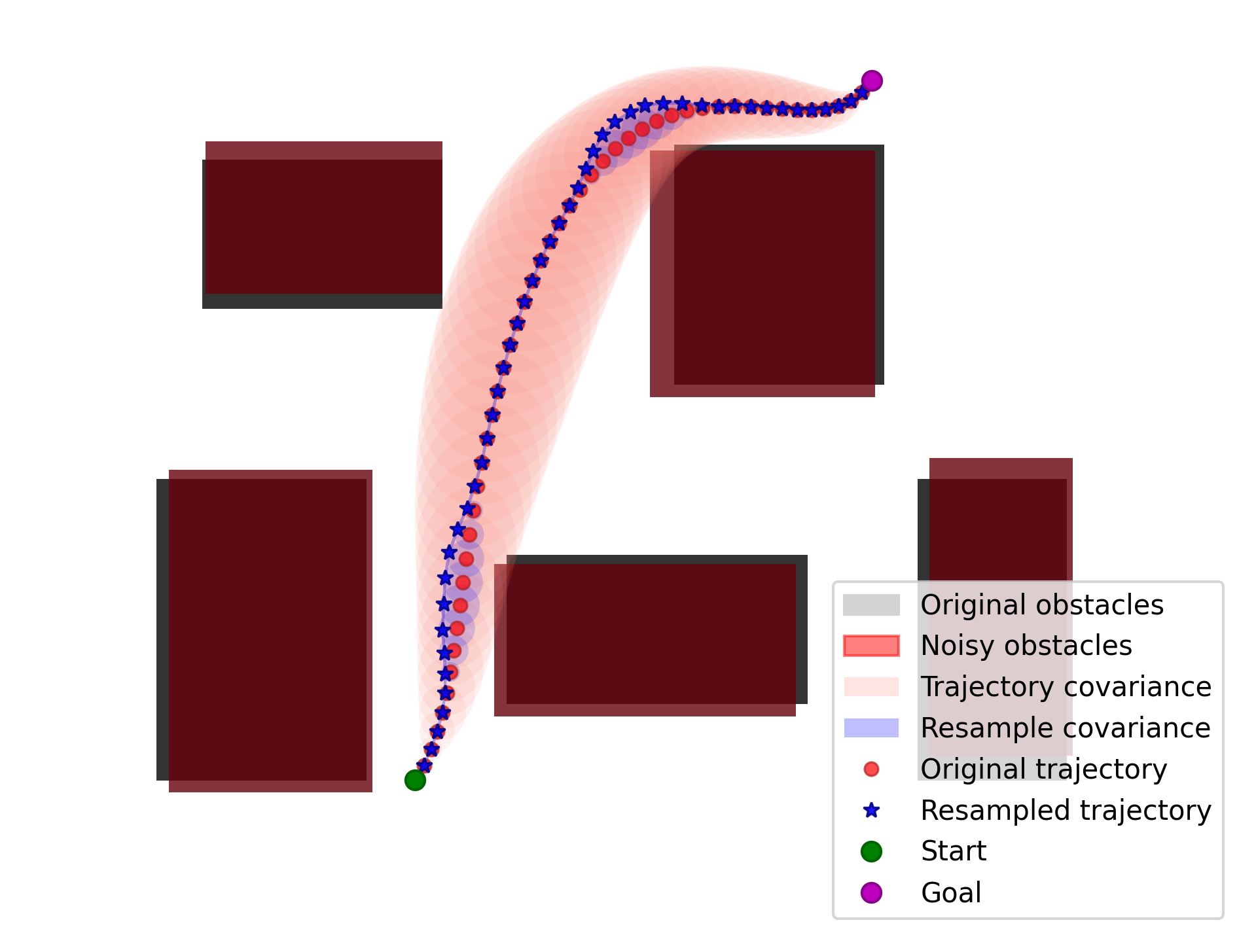}
    \caption{$2$-D point robot example of the resampling procedure using the planned covariances, for uncertain obstacles.}
    \label{fig:resample_2D}
\end{figure}

{\em b) Real-world experiment.} In real-world scenarios, perception and localization noises can lead to environmental uncertainties \cite{chen2022should}. We conducted hardware experiments to showcase the robustness of GVIMP compared to baseline algorithms to showcase that our motion plan with distributional information (covariances) is superior to deterministic baselines in scenarios with disturbances. The hardware experiment setup is shown in Fig. \ref{fig:hardware_setup}, where we use a Franka arm to move from a start to a target position, avoiding a box obstacle. The uncertainties come in two-fold: the noisy readings from the camera, and a disturbance we add by randomly moving the box, representing unmodeled disturbances in real-world scenarios. 

\begin{figure}[ht]
\centering
    \begin{subfigure}[ht]{0.24\textwidth}
    \centering
    \includegraphics[width=\textwidth]{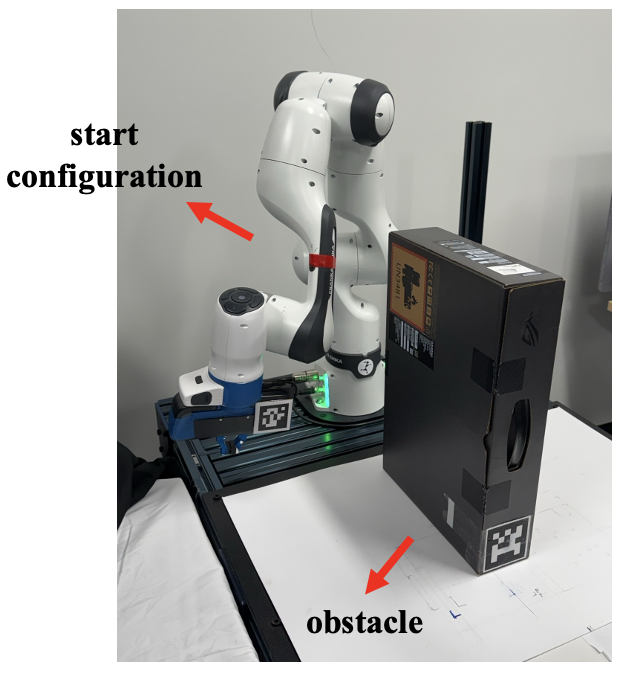}
    \caption{Start configuration}
    \end{subfigure}
    \hfill
    \begin{subfigure}[ht]{0.24\textwidth}
    \includegraphics[width=\textwidth]{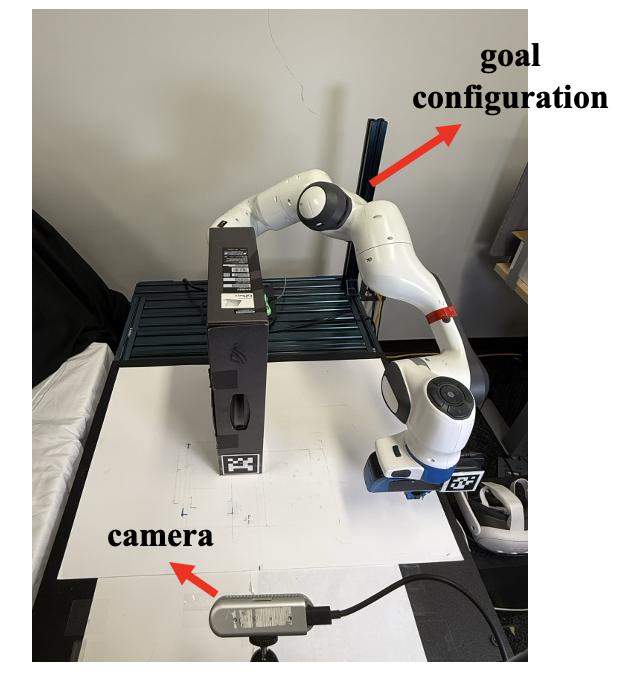}
    \caption{Goal configuration}
    \end{subfigure}
  \caption{Franka arm hardware experiment setting.}
  \label{fig:hardware_setup}
\end{figure}

\begin{figure}[ht]
\centering
    \begin{subfigure}[ht]{0.23\textwidth}
    \centering
    \includegraphics[width=\textwidth]{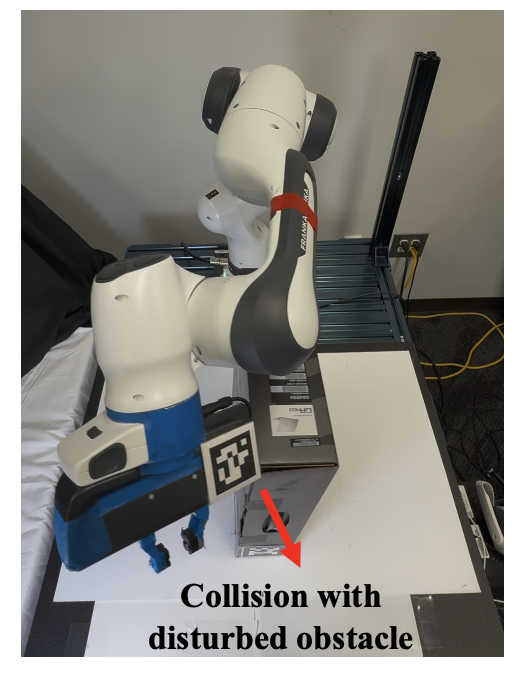}
    \caption{RRT Connect}
    \end{subfigure}
    \hfill
    \begin{subfigure}[ht]{0.23\textwidth}
    \includegraphics[width=\textwidth]{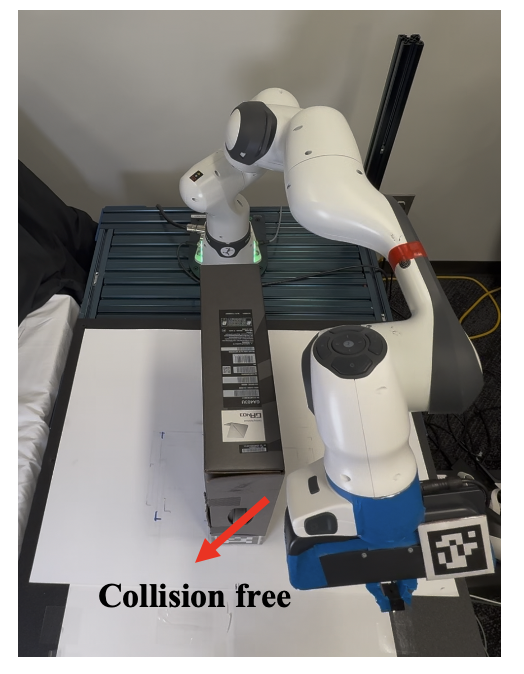}
    \caption{GVIMP}
    \end{subfigure}
  \caption{Comparison results for disturbed obstacles.}
  \label{fig:harware_exp_results}
\end{figure}

The experiment procedure is as follows. First, we obtain a \textit{nominal} pose of the obstacle by averaging the noisy readings of a static box pose from the camera to account for sensor noise. Then, we do motion planning using the obtained nominal obstacles. Next, we add disturbances to the obstacles by manually moving them. Finally, we execute the motion plans with the disturbed obstacles without re-planning. Since our plan stores the covariances, we resample the plan using the solved covariances in the disturbed scenarios during execution. 

To obtain statistically meaningful results, we conducted the experiments with $50$ randomly disturbed obstacle poses and compared our results with baseline motion planners. We computed the signed distances of the motion plan trajectories to the obstacle in each case. The results are summarized in Tab. \ref{tab:signed_distance_disturbed}. Our proposed GVIMP gives the safest plan under uncertainties. Fig. \ref{fig:harware_exp_results} showcases the comparison of the motion planning given by GVIMP versus the RRT-Connect baseline under one disturbance, where a collision happened for the RRT-Connect plan, while the GVIMP plan is robust against the disturbance. More hardware experiment results are available at: \href{https://www.youtube.com/watch?v=c4sFOlEki0Q}{https://www.youtube.com/watch?v=c4sFOlEki0Q}.
\begin{table*}[th]
\centering
\begin{tabular}{|c|c|c|c|c|c|c|}
\hline
 Planner & LBKPIECE  & PRM & RRT & RRT Connect & RRT$^*$ &  GVIMP  \\ \hline
 Avg. Signed Distance & $-0.069$ & $0.01849$ & $-0.0011$ & $-0.0237$ & $0.0111$ & $\textbf{0.0546}$ \\ \hline
\end{tabular}
\caption{Minimum signed distances between the planned trajectories and the disturbed obstacles, averaged over $50$ randomly disturbed obstacle poses. The proposed GVIMP overperformed the baselines in terms of robustness against uncertainties.}
\label{tab:signed_distance_disturbed}
\end{table*}

\begin{figure}[ht]
\centering
    \begin{subfigure}[h]{0.31\textwidth}
    \centering
    \includegraphics[width=0.75\textwidth]{figures/pcsmp/planar_quad/exp2/pquad_sdf.pdf}
    \caption{PCS-MP mean and covariance trajectory}
    \label{fig:pgcs_pquad_1}
    \end{subfigure}
    \hfill
    \par\vspace{2mm}
    \begin{subfigure}[h]{0.3\textwidth}
    \centering
    \includegraphics[width=0.8\textwidth]{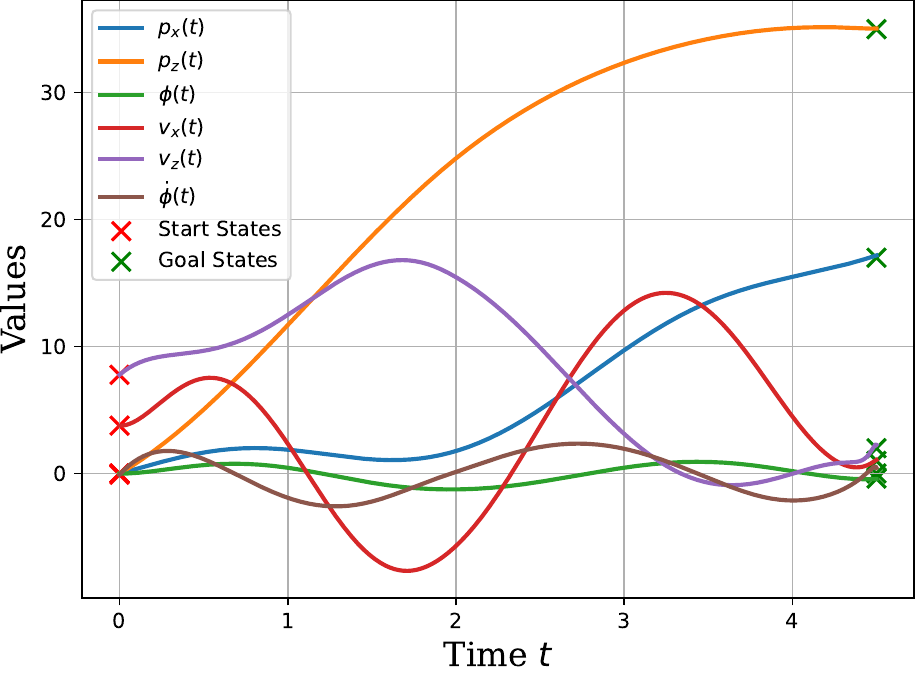}
    \caption{State space mean trajectory}
    \label{fig:pgcs_pquad_2}
    \end{subfigure}
    \hfill
    \begin{subfigure}[h]{0.24\textwidth}
    \centering
    \includegraphics[width=\textwidth]{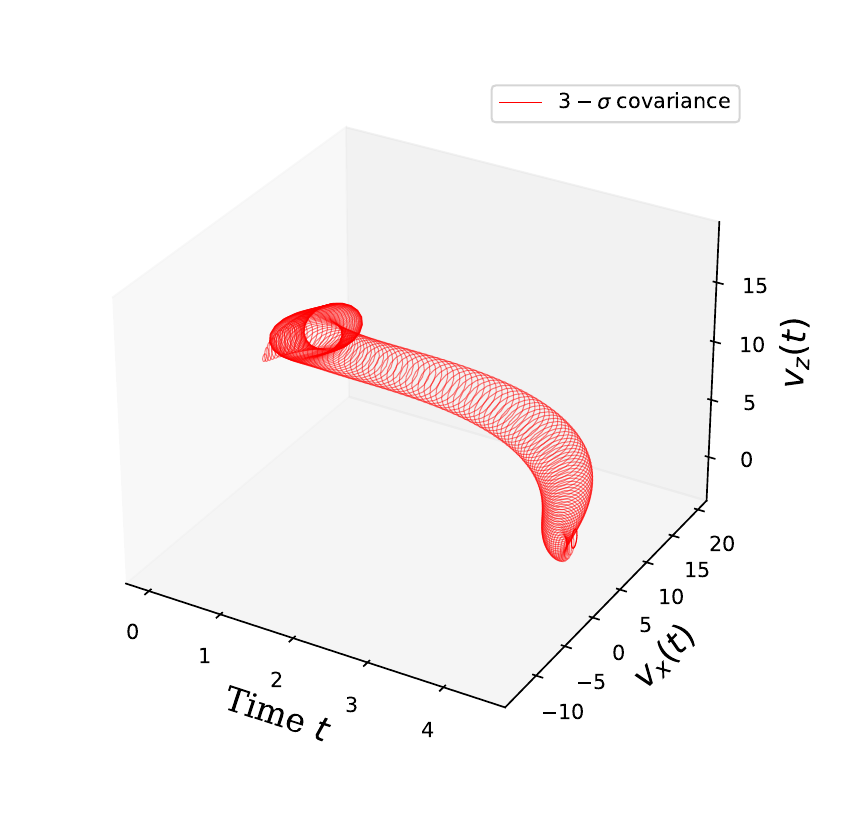}
    \caption{Covariance $\left [v_x(t), v_z(t) \right]$}
    \label{fig:pgcs_pquad_3}
    \end{subfigure}
    \begin{subfigure}[h]{0.24\textwidth}
    \centering
    \includegraphics[width=\textwidth]{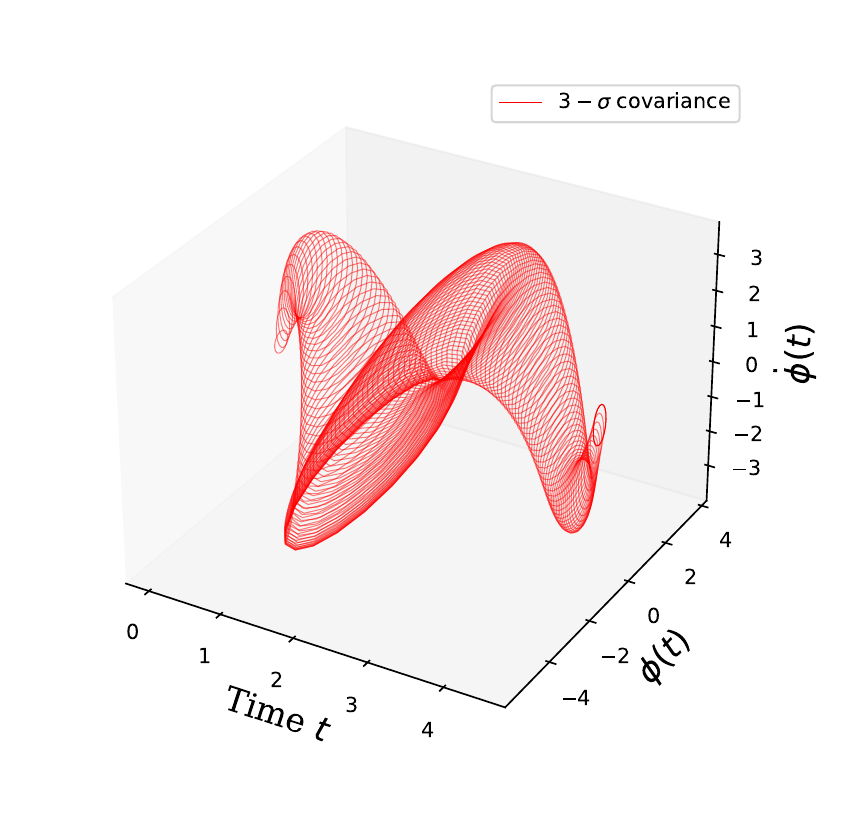}
    \caption{Covariance $\left [\phi(t), \dot \phi(t)\right]$}
    \label{fig:pgcs_pquad_4}
    \end{subfigure}
  \caption{Motion planning for an LTV stochastic system \eqref{eq:planar_quad}. }
  \label{fig:pgcs_pquad}
\end{figure}

\subsection{A Time-Varying System}
\label{sec:experiments_LTVexample}
We validate the PCS-MP algorithm on an LTV stochastic system obtained by linearizing a planar quadrotor dynamics. The system operates within a $6$-dimensional state space denoted as $X_t \triangleq \left[p_x, p_z, \phi, v_x, v_z, \dot{\phi}\right].$ Here, $(p_x, p_z)$ represents the lateral and vertical positions, $(v_x, v_z)$ stands for lateral and vertical velocities in the body frame, with $\phi$ accounting for the rotational degree of freedom. The dynamics is defined as 
\begin{equation}
    \dot{X_t} = \begin{bmatrix}
        v_x \cos(\phi) - v_z \sin(\phi)\\
        v_x \sin(\phi) + v_z\cos(\phi) \\
        \dot\phi \\
        v_z \dot\phi - g\sin(\phi)\\
        -v_x\dot\phi - g\cos(\phi)\\
        0
    \end{bmatrix} + 
    \begin{bmatrix}
        0 & 0 \\
        0 & 0 \\
        0 & 0 \\
        0 & 0 \\
        1/m & 1/m \\
        l/J & -l/J 
    \end{bmatrix}
    \begin{bmatrix}
        u_1 \\
        u_2
    \end{bmatrix},
    \label{eq:planar_quad}
\end{equation}
where $g$ is the gravity, $m$ represents the mass of the planar quadrotor, $l$ is the length of the body, and $J$ is the moment of inertia. The system has two thrust inputs. In our experiments, we set $m=1/\sqrt{2}, l=\sqrt{2}$, and $J=1$.

We test PCS-MP in the environment with an obstacle between the initial and the goal states. The collision cost \eqref{eq:collision_factor} is evaluated at each collision-checking point on the robot at each state on the trajectory. To run PCS-MP for this system, we iteratively linearize \eqref{eq:planar_quad}, and for each linearized LTV system, we run Algorithm \ref{alg:pgcsmp} until convergence; we then linearize the system in the next iteration around the obtained nominal $(z_t^*, \Sigma_t^*)$. We use the linearization around $(X_0, K_0)$ as the initialization. Fig. \ref{fig:pgcs_pquad} shows the planning results. Fig. \ref{fig:pgcs_pquad_1} shows the controlled nominal trajectory $(z_t^*, \Sigma_t^*)$, where the $3-\sigma$ contour of the covariance between the position variables $(p_x, p_z)$ are displayed in red ellipsoids. Fig. \ref{fig:pgcs_pquad_2} showcases the mean trajectories of all $6$ states, together with the start and goal states. Fig. \ref{fig:pgcs_pquad_3} and \ref{fig:pgcs_pquad_4} show the covariance trajectories between the velocities and the angles. We tested the PCS-MP algorithm in the same environment with $4$ different tasks. The results are shown in Fig. \ref{fig:pgcs_pquad_4exp}. PCS-MP finds a feasible trajectory distribution in all the experiments. The minimum-energy solution in our formulation \eqref{eq:SCP_formulation_GVI} maintains a feasible tilting angle of the quadrotor along the trajectory. We use $3000$ time discretizations. 

\section{conclusion and future directions}
\label{sec:conclusion}
\subsection{Conclusion}
We proposed the Gaussian Variational Inference Motion Planning (GVI-MP) framework, formulating the motion planning problem as a Gaussian variational inference to a posterior. GVI-MP is equivalent to a stochastic control problem over time-discrete states. We proposed the Gaussian Variational Inference Motion Planner to solve this inference problem using a natural gradient scheme over the sparse factor graph. Closed-form and sparse quadrature are used to evaluate functional gradients efficiently. We then proposed the Proximal Covariance Steering Motion Planner (PCS-MP) algorithm to solve the equivalent control problem with boundary constraints. PCS-MP leverages a closed-form solution from a linear covariance steering problem through quadratic approximations for nonlinear factors, making its performance comparable to the deterministic baseline. Our algorithms are efficient and robust in handling environmental uncertainties.

\subsection{Limitation and Future Works}
\label{sec:experiments_limitations}
In this work, we chose Gaussian distributions as our optimization variable because of their tractability and sampling efficiency. Gaussian simplification of the real-world noise is widely used in the nonlinear filtering and optimal control literature. Despite its simplicity, a drawback of Gaussian is that it is single-modal, while motion planning posterior is naturally non-convex and thus multi-modal. The KL divergence minimization objective promotes a mode-seeking result, concentrating on one feasible solution distribution. Our future directions include developing algorithms for models that can better handle multi-modality, such as the Gaussian mixture models, testing on more complex dynamical systems, and developing algorithms in a receding horizon fashion.

\appendix 
\label{sec:appendix}
\subsection{Proof of Lemma \ref{lem:closed_form_prior}}
\label{sec:proof_lemma_closedform_prior}
For a general quadratic factor $\lVert \Lambda\bX - \Psi\bmu \rVert_{\bK^{-1}}^2$, define $\Tilde{\epsilon} = \Lambda\mu_\theta - \Psi\bmu, \;\; y=\bX-\mu_{\theta}$, then we have
\begin{eqnarray*}\label{eq:GVI_linear_factor1}
    \!\!\!&&\!\!\! \mE_{q\sim \cN(\mu_{\theta}, \Sigma_{\theta})} [\psi_{\rm{Prior}}] 
    \\
    \!\!\!&=&\!\!\! \mE[(\Lambda y + \Tilde{\epsilon})^T \bK^{-1} (\Lambda y + \Tilde{\epsilon})] = \tr[\Lambda^T\bK^{-1} \Lambda \Sigma_{\theta}] + \Tilde{\epsilon}^T\bK^{-1}\Tilde{\epsilon}. \nonumber
\end{eqnarray*}
Similarly, 
\begin{eqnarray*}\label{eq:GVI_linear_factor2}
    &&\!\!\! \mE_{q\sim \cN(\mu_{\theta}, \Sigma_{\theta})}[(\bX-\mu_\theta)\psi_{\rm{Prior}}] 
    \\
    &=&\!\!\! \mE[y(\Lambda y + \Tilde{\epsilon})^T \bK^{-1} (\Lambda y + \Tilde{\epsilon})] = 2\Sigma_{\theta}\Lambda^T\bK^{-1}\Tilde{\epsilon}. \nonumber
\end{eqnarray*}
Finally, we compute 
\begin{eqnarray*}\label{eq:GVI_linear_factor3}
    && \mE_{q\sim \cN(\mu_{\theta}, \Sigma_{\theta})}[(\bX-\mu_\theta)(\bX-\mu_\theta)^T\psi_{\rm{Prior}}]  
    \\
    &=& \mE[yy^T(\Lambda y + \Tilde{\epsilon})^T \bK^{-1} (\Lambda y + \Tilde{\epsilon})] \nonumber
    \\
    &=& \mE[yy^Ty^T\Lambda^T\bK^{-1}\Lambda y] + \Sigma_\theta \Tilde{\epsilon}^T\bK^{-1}\Tilde{\epsilon}. \nonumber
\end{eqnarray*}
Here we leveraged the fact that the moments of a Gaussian $\mE[yy^T\Lambda^T\bK^{-1}\Lambda y], \mE[y^T\Tilde{\epsilon}^T\bK^{-1}\Tilde{\epsilon}], \mE[yy^Ty^T\Lambda^T\bK^{-1}\Tilde{\epsilon}]$ are $0$, and we conclude that the update rules \eqref{eq:GVI_derivatives} for a quadratic factor has closed forms in terms of $\mu_\theta$ and $\Sigma_\theta$. 

\subsection{Gauss-Hermite quadratures for nonlinear factors}
\label{sec:GH_quadratures}
In this section, we introduce the Gauss-Hermite quadratures. To compute the integration of a scalar function $\varphi(\bX)$
\begin{equation*}
    \int \varphi(\bX) \cN(\bX; m, P) d\bX,
    \label{eq:target_integration}
\end{equation*}
G-H quadrature methods consist of the following steps
\begin{enumerate}
    \item Compute the $p$ roots of a $p\; th$ order Hermite polynomial, also denoted as \textit{sigma points}: $\xi = [\xi_1, \dots, \xi_p].$

    \item Compute the \textit{weights}: $W_i = \frac{p!}{p^2[H_{p-1}(\xi_i)]}, i=1, \dots, p.$

    \item Approximate: 
    \begin{equation*}
        \int \varphi(\bX) \cN(\bX;m,P) d\bX \approx \sum_{l=1}^{p^n} \Tilde{W}_l \varphi(\sqrt{P}\Tilde{\xi}_l + m),
    \end{equation*}
    where $\Tilde{\xi}_l$ is a $n$ dimensional vector concatenating a permutation of the sigma points, i.e., $\Tilde{\xi_l} = [\xi_{l_1}, \dots, \xi_{l_n}]$, and $l_1, \dots, l_n \in [1, \dots, p]$, and $\Tilde{W}_l$ is the the product of all the corresponding weights, i.e., $\Tilde{W}_l = \prod_{i=1}^{n}W_{l_i}$.
\end{enumerate}
To approximate a $n$-dimensional integration, $p^n$ sigma points are needed. This is due to the tensor product of the nodes across different dimensions. Sparse grid quadrature uses Smolyak's rule to drop out the cross-terms in Taylor's approximation. The Smolyak rule distributes the total polynomial exactness into different dimensions. The resulting complexity is reduced to a polynomial dependence on the dimension\cite{heiss2008likelihood}.
\begin{lemma}[\cite{heiss2008likelihood}, Theorem 2]
\label{lem:sparseGH_complexity}
For a $D_q-$variate function, the number of computations needed by a sparse grid quadrature that is exact for $2k_q-1$ total polynomial order is bounded by
\[
\frac{e^{k_q}}{(k_q - 1)!}(D_q)^{k_q}.
\]
\end{lemma}

\subsection{Proof of Proposition \ref{thm:lemma1}}
\label{sec:proof_lemma_OCP_MAP}
Consider the time discretization \eqref{eq:time_discretization} and \eqref{eq:definition_bX}. We start from the minimum energy problem in the $i^{\rm{th}}$ time discretization
\begin{subequations}\label{eq:optimal_control_formulation_proof}
\begin{eqnarray*}
    \min_{u(\cdot)}\!\!\!\!\!\!\!&& \int_{t_i}^{t_{i+1}} \frac{1}{2}\|u_t\|^2 dt, \;\; {\rm s.t.} \; \dot{\Tilde{X}}_t = A_t\Tilde{X}_t + B_tu_t,
\end{eqnarray*}
\end{subequations}
where $\Tilde{X}_t$ is the deviated state from mean defined in \eqref{eq:tilde_Xt}. Let $\Phi$ be the state transition matrix, then 
\begin{equation}\label{eq:state_transition_u}
    \Tilde{X}_{t_{i+1}} = \Phi(t_{i+1}, t_i)\Tilde{X}_i + \int_{t_i}^{t_{i+1}}\Phi(t_{i+1}, \tau)B_{\tau}u_\tau d \tau.
\end{equation}
Assume now that the control $u_t$ has the form $u_t = B_{t}^T \Phi(t_{i+1}, t)q$, where $q$ is an arbitrary vector that will be compatible with the dimensions, then 
\begin{eqnarray*}
    \Tilde{X}_{t_{i+1}} \!\!\!\!\!&=&\!\!\!\!\! \Phi(t_{i+1}, t_i)\Tilde{X}_i +  \left (\int_{t_i}^{t_{i+1}} \Phi_{i+1, \tau}B_{\tau} B_{\tau}^T \Phi_{i+1, \tau}^T d \tau \right) q.
\end{eqnarray*}
Define the Grammian for the controllable system
\begin{equation}
    \nonumber Q_{i,i+1} \triangleq \int_{t_i}^{t_{i+1}} \Phi(t_{i+1}, \tau)B_{\tau} B_{\tau}^T \Phi(t_{i+1}, \tau)^T d \tau,
\end{equation}
then $q = Q_{i,i+1}^{-1}(\Tilde{X}_{t_{i+1}} - \Phi(t_{i+1}, t_i)\Tilde{X}_i)$, and 
\begin{equation}
\label{eq:min_energy_u}
    u_t = B_{t}^T \Phi(t_{i+1}, t) Q_{i,i+1}^{-1}(\Tilde{X}_{t_{i+1}} - \Phi(t_{i+1}, t_i)\Tilde{X}_i).
\end{equation}
We next show that $u_t$ in \eqref{eq:min_energy_u} is the solution to the minimum energy control problem \eqref{eq:optimal_control_formulation_proof}. Assume another control $v_{t} = u_t + \Tilde{u}_{t}$ also drives the system from $\Tilde{X}_{t_i}$ to $\Tilde{X}_{t_{i+1}}$, so that we have $\Tilde{X}_{t_{i+1}} = \Phi(t_{i+1}, t_i)\Tilde{X}_i + \int_{t_i}^{t_{i+1}}\Phi(t_{i+1}, \tau)B_{\tau}v(\tau)d \tau$. Combined with \eqref{eq:state_transition_u}, we have $\int_{t_i}^{t_{i+1}}\Phi(t_{i+1}, \tau)B_{\tau}\Tilde{u}_\tau d \tau = 0.$
The energy associated with an arbitrarily different control $v$ 
\begin{eqnarray*}
    \!\!\!\!\!\!\int_{t_i}^{t_{i+1}} \|v_{t}\|^2 dt \!\!\!\!\!\!&=&\!\!\!\!\!\! \int_{t_i}^{t_{i+1}} \frac{1}{2}\|u_{t} + \Tilde{u}_{t}\|^2 dt
    \\
    \!\!\!\!\!\!&=&\!\!\!\!\!\!\int_{t_i}^{t_{i+1}} \|u_{t}\|^2 dt + \int_{t_i}^{t_{i+1}} \|\Tilde{u}_{t}\|^2 dt \geq \int_{t_i}^{t_{i+1}} \|u_{t}\|^2 dt.
\end{eqnarray*}
This is because $\int_{t_i}^{t_{i+1}} u_t^T\Tilde{u}_t dt $ equals to
\begin{subequations}
    \begin{eqnarray*}
    (\Tilde{X}_{t_{i+1}} - \Phi_{i+1, i}\Tilde{X}_i)^T Q_{i,i+1}^{-1} \int_{t_i}^{t_{i+1}} \Phi_{i+1, i}^T B_t \Tilde{u}_t d t = 0.
\end{eqnarray*}
\end{subequations}
Finally, a direct calculation of $\int_{t_i}^{t_{i+1}} u_t^Tu_t d t$ gives 
\begin{eqnarray*}
    \int_{t_i}^{t_{i+1}} u_t^Tu_t d t &=& (\Tilde{X}_{t_{i+1}} - \Phi(t_{i+1}, t_i)\Tilde{X}_i)^T  Q_{i,i+1}^{-T}
    \\
    &&\hspace{-2.8cm}\overbrace{(\int_{t_i}^{t_{i+1}} \Phi(t_{i+1}, t)^T B_{t}B_{t}^T \Phi(t_{i+1}, t) d t)}^{Q_{i,i+1}}Q_{i,i+1}^{-1}
    \\ 
    &&\hspace{-2.8cm}(\Tilde{X}_{t_{i+1}} - \Phi(t_{i+1}, t_i)\Tilde{X}_i) = \lVert \Tilde{X}_{t_{i+1}} - \Phi(t_{i+1}, t_i)\Tilde{X}_i \rVert_{Q_{i,i+1}^{-1}}^2.
\end{eqnarray*}
It follows that formulation \eqref{eq:OCP_formulation_MAP} is equivalent to
\begin{eqnarray*}
    \min_{X(\cdot)}\!\!\!\!\!\!\!&& \sum_{i=0}^{N} [\lVert X_{t_{i+1}} - \mu_{t_{i+1}} - \Phi(t_{i+1}, t_i)(X_i-\mu_{t_i}) \rVert_{Q_{i,i+1}^{-1}}^2 
    \\ + && \hspace{-0.55cm}
    \lVert \Tilde{\mathbf{h}}(X_i) \rVert_{\Sigma_{obs}}^2] + \frac{1}{2}\lVert X_{t_0} -\! \mu_0 \rVert_{K_0^{-1}}^2+ \frac{1}{2}\lVert X_{t_N} -\! \mu_N \rVert_{K_N^{-1}}^2,
\end{eqnarray*}
which is exactly \eqref{eq:MAP_formulation} with sparsity \eqref{eq:sparse_G} and \eqref{eq:sparse_Q}.

\subsection{Proof of Lemma \ref{thm:variationP}}
\label{sec:proof_lemma_variationP}
Using the first-order approximation, we write
\begin{eqnarray}\label{eq:variationF}
    && \langle \frac{1}{\epsilon}\hat V (\cP+\delta \cP)-\log d\cP^0(\cP+\delta \cP), \cP+\delta \cP\rangle
    \\&&-  \langle \frac{1}{\epsilon}\hat V (\cP)-\log d\cP^0(\cP), \cP\rangle \nonumber
    \\\nonumber &\approx&\langle \frac{1}{\epsilon}\hat V (\cP)-\log d\cP^0(\cP),\delta \cP\rangle 
    \\&&+\langle \frac{1}{\epsilon}\hat V (\cP+\delta \cP)-\frac{1}{\epsilon}\hat V (\cP)-\log \frac{d\cP^0(\cP+\delta \cP)}{d\cP^0(\cP)}, \cP\rangle. \nonumber
\end{eqnarray}
The first term is linear, and we approximately expand the term $\langle \hat V (\cP+\delta \cP)-\hat V (\cP), \cP\rangle$
\begin{eqnarray*}
&& \langle V(z_t+\delta z_t)-V(z_t)-\delta z_t^T \nabla V(z_t) ,\cP\rangle
\\&+& \langle \frac{1}{2}(x^T-z_t^T-\delta z_t^T) \nabla^2 V(z_t+\delta z_t) (x-z_t-\delta z_t)
\\&-&\frac{1}{2}(x^T-z_t^T) \nabla^2 V(z_t) (x-z_t),\cP\rangle.
\end{eqnarray*}
The first term is $0$ if we keep only the first-order approximations. The second term is approximately
\begin{eqnarray*}
    &&\langle \frac{1}{2}(x^T-z_t^T) \nabla^2 V(z_t+\delta z_t) (x-z_t)
    \\&&-\frac{1}{2}(x^T-z_t^T) \nabla^2 V(z_t) (x-z_t),\cP\rangle
         \\ =&& \frac{1}{2} \int_0^1[\tr (\nabla^2 V(z_t+\delta z_t) \Sigma_t)-\tr (\nabla^2 V(z_t) \Sigma_t)] dt
    \\\approx&& \frac{1}{2}\int_0^1[\delta z_t^T \nabla\tr(\nabla^2 V(z_t) \Sigma_t)] dt.
\end{eqnarray*}
Here we remove the dependence on the iteration $k$ for notation simplicity. It follows that 
\begin{equation*}
    \langle \hat V (\cP+\delta \cP)-\hat V (\cP), \cP\rangle \approx \frac{1}{2}\langle x^T \nabla\tr(\nabla^2 V(z_t) \Sigma_t), \delta \cP\rangle.
\end{equation*}
Now, the remaining term
\begin{equation*}
\begin{split}
    &\langle\log \frac{d \cP^0(\cP)}{d \cP^0(\cP+\delta \cP)}, \cP\rangle = \langle \log \frac{d\cP(\cP)}{d \cP^0(\cP+\delta \cP)} - \log \frac{d \cP(\cP)}{d \cP^0(\cP)}, \cP\rangle = 0. 
\end{split}
\end{equation*}
This is because in the linear dynamics case $\hat{\cP^0} = \cP^0$, and the measure $d\cP^0$ induced by \eqref{eq:dynamics_uncontrolled} is independent of $\cP$. Combining the above two equations into \eqref{eq:variationF} yields
\begin{eqnarray}
    \frac{\delta F}{\delta \cP}(\cP)(t, x)= \frac{1}{\epsilon}\hat V -\log d \cP^0 + \frac{1}{2\epsilon} x^T \nabla\tr(\nabla^2 V(z_t) \Sigma_t),
\end{eqnarray}
where $(\cdot)^\dagger$ is pseudo-inverse, $z_t$ and $\Sigma_t$ are the mean and covariance of $\cP$ respectively and $\hat V, \cP^0$ are as in \eqref{eq:V_quadratic_approximation} and \eqref{eq:dynamics_uncontrolled}.

\subsection{Proof of Lemma \ref{thm:Qrk}}
\label{sec:proof_lemma_Qrk}
The prior distribution $\cP^0$ is generated by linear dynamics \eqref{eq:cov_steering_lin_dyn2}. We rewrite the optimization at the update step 
\begin{equation*}
\begin{split}
    \min_{\cP^u} &\int [\frac{1}{\epsilon}\hat V+ \frac{1}{2\epsilon} X^T \nabla\tr((\nabla^2 V) \Sigma)] d\cP^u  
    \\
    & + {\rm KL} (\cP^k\| \cP^0) + \frac{1}{\eta} {\rm KL}(\cP^k\| \cP^u)
\end{split}
\end{equation*}
into a linear covariance steering. Dropping the dependence on time $t$ and from the Girsanov theorem we have
    \begin{align*}
        &\frac{1}{\eta}{\rm KL}(\cP^k \| \cP^0) 
        \\
        &= \frac{1}{2\epsilon\eta} \mE\{ \int_{t_0}^{t_N} \lVert \frac{1}{1+\eta} ((A^k- A)X + a^k - a) + B u \rVert_{(BB^T)^\dagger}^2 d t \},
        \\
        &{\rm KL}(\cP^k \| \cP^u) 
        \\
        &= \frac{1}{2\epsilon} \mE\{ \int_{t_0}^{t_N} \lVert \frac{\eta}{1+\eta} ((A -A^k)X + a-a^k) + B u \rVert_{(BB^T)^\dagger}^2 d t\}.
    \end{align*}
    A direct computation of ${\rm KL}(\cP \| \cP^0) + \frac{1}{\eta} {\rm KL}(\cP \| \cP^u)$ gives
    \begin{equation*}
         \frac{1}{2\epsilon} (\frac{1+\eta}{\eta}\lVert u \rVert^2 + \frac{1}{1+\eta}\lVert (A -A^k)X + a-a^k \rVert_{(BB^T)^\dagger}^2).
    \end{equation*}
    Recall that
    \begin{equation*}
    \hat V(X) \approx V(z) + (X-z)^T \nabla V(z) + \frac{1}{2}(X^T-z^T) \nabla^2 V(z) (X-z).
    \end{equation*}
    Combining the above calculations and $\hat{V}$, together with the linear term $\frac{1}{2\epsilon} X^T \nabla\tr((\nabla^2 V) \Sigma)$, we get the result, with $r_t^k$ equaling
    \begin{align*}
        r_t^k &= \frac{\eta}{1+\eta}\nabla V 
        + \frac{\eta}{2(1+\eta)}[\nabla \tr(\nabla^2V\Sigma^k_t)-2\nabla^2V z^k_t] 
		\\
        &\hspace{0.5cm}+\frac{\eta}{(1+\eta)^2}(A_t^k-\hat A_t^k)^T(B_tB_t^T)^{\dagger}(a_t^k-\hat a_t^k).
    \end{align*}
    We ignore derivatives of orders higher than $2$.

\subsection{Experiment Details and Practical Issues}
\label{sec:experiment_details}
This section discusses the implementation details and practical issues in GVI-MP and PCS-MP. The factorization in GVI-MP allows for parallel computing of the derivatives to make the implementation more efficient. We use Heun's numerical integrations for point and arm robots and Runge–Kutta numerical integrations for the quadrotor experiment.

{\em a) Derivatives approximation in PCS-MP.}
In Algorithm \ref{alg:pgcsmp}, the state cost is the norm of a hinge loss \eqref{eq:PGCS_state_cost_h} 
\begin{equation*}
    V(X_t) = \lVert \bh(X_t) \rVert_{\Sigma_{\rm obs}}^2 = \lVert \Tilde{\bh}(S(F(X_t))) \rVert_{\Sigma_{\rm obs}}^2,
\end{equation*}
which leads to an explicit gradient using the chain rule
\begin{equation*}
    \nabla V = \frac{\partial F}{\partial X_t} \frac{\partial S}{\partial F} \frac{\partial \Tilde{\bh}}{\partial S } (2\Sigma_{\rm obs} \Tilde{\bh} ).
\end{equation*}
All functions are potential vectors corresponding to multiple collision checking points on a robot. 

{\em b) Experiment setups and parameter choices.}
The choice of parameters involves balancing collision-avoiding and prior costs and between the sum of the two and the entropy cost. For a fixed collision-checking range in a map with a known resolution, the tunable collision hinge loss parameter is $\Sigma_{\rm obs}$. We choose a smaller collision-related cost in more cluttered environments to balance the prior cost. Starting from a collision-free configuration, we increase the temperature to put more weight on the entropy costs to seek robustness.

\bibliographystyle{Bibliography/IEEETrans}
\bibliography{root}

\begin{thebibliography}{10}
\providecommand{\url}[1]{#1}
\csname url@samestyle\endcsname
\providecommand{\newblock}{\relax}
\providecommand{\bibinfo}[2]{#2}
\providecommand{\BIBentrySTDinterwordspacing}{\spaceskip=0pt\relax}
\providecommand{\BIBentryALTinterwordstretchfactor}{4}
\providecommand{\BIBentryALTinterwordspacing}{\spaceskip=\fontdimen2\font plus
\BIBentryALTinterwordstretchfactor\fontdimen3\font minus \fontdimen4\font\relax}
\providecommand{\BIBforeignlanguage}[2]{{%
\expandafter\ifx\csname l@#1\endcsname\relax
\typeout{** WARNING: IEEEtran.bst: No hyphenation pattern has been}%
\typeout{** loaded for the language `#1'. Using the pattern for}%
\typeout{** the default language instead.}%
\else
\language=\csname l@#1\endcsname
\fi
#2}}
\providecommand{\BIBdecl}{\relax}
\BIBdecl

\bibitem{singh2017robust}
S.~Singh, A.~Majumdar, J.-J. Slotine, and M.~Pavone, ``Robust online motion planning via contraction theory and convex optimization,'' in \emph{2017 IEEE International Conference on Robotics and Automation (ICRA)}.\hskip 1em plus 0.5em minus 0.4em\relax IEEE, 2017, pp. 5883--5890.

\bibitem{majumdar2017}
A.~Majumdar and R.~Tedrake, ``Funnel libraries for real-time robust feedback motion planning,'' \emph{The International Journal of Robotics Research}, vol.~36, no.~8, pp. 947--982, 2017.

\bibitem{dai2012optimizing}
H.~Dai and R.~Tedrake, ``Optimizing robust limit cycles for legged locomotion on unknown terrain,'' in \emph{2012 IEEE 51st IEEE Conference on Decision and Control (CDC)}.\hskip 1em plus 0.5em minus 0.4em\relax IEEE, 2012, pp. 1207--1213.

\bibitem{kappler2015data}
D.~Kappler, P.~Pastor, M.~Kalakrishnan, M.~W{\"u}thrich, and S.~Schaal, ``Data-driven online decision making for autonomous manipulation.'' in \emph{Robotics: science and systems}, vol.~11, 2015.

\bibitem{prentice2009belief}
S.~Prentice and N.~Roy, ``The belief roadmap: Efficient planning in belief space by factoring the covariance,'' \emph{The International Journal of Robotics Research}, vol.~28, no. 11-12, pp. 1448--1465, 2009.

\bibitem{kohler2020computationally}
J.~K{\"o}hler, R.~Soloperto, M.~A. M{\"u}ller, and F.~Allg{\"o}wer, ``A computationally efficient robust model predictive control framework for uncertain nonlinear systems,'' \emph{IEEE Transactions on Automatic Control}, vol.~66, no.~2, pp. 794--801, 2020.

\bibitem{mukadam2016}
M.~Mukadam, X.~Yan, and B.~Boots, ``Gaussian process motion planning,'' in \emph{2016 IEEE International Conference on Robotics and Automation (ICRA)}.\hskip 1em plus 0.5em minus 0.4em\relax IEEE, 2016, pp. 9--15.

\bibitem{mukadam2018}
M.~Mukadam, J.~Dong, X.~Yan, F.~Dellaert, and B.~Boots, ``Continuous-time {G}aussian process motion planning via probabilistic inference,'' \emph{The International Journal of Robotics Research}, vol.~37, no.~11, pp. 1319--1340, 2018.

\bibitem{bertsekas1996stochastic}
D.~Bertsekas and S.~E. Shreve, \emph{Stochastic optimal control: the discrete-time case}.\hskip 1em plus 0.5em minus 0.4em\relax Athena Scientific, 1996, vol.~5.

\bibitem{lavalle1998}
S.~LaValle, ``Rapidly-exploring random trees: A new tool for path planning,'' \emph{Research Report 9811}, 1998.

\bibitem{kavraki1996}
L.~E. Kavraki, P.~Svestka, J.-C. Latombe, and M.~H. Overmars, ``Probabilistic roadmaps for path planning in high-dimensional configuration spaces,'' \emph{IEEE Transactions on Robotics and Automation}, vol.~12, no.~4, pp. 566--580, 1996.

\bibitem{donald1993kinodynamic}
B.~Donald, P.~Xavier, J.~Canny, and J.~Reif, ``Kinodynamic motion planning,'' \emph{Journal of the ACM (JACM)}, vol.~40, no.~5, pp. 1048--1066, 1993.

\bibitem{wang2022geometrically}
Z.~Wang, X.~Zhou, C.~Xu, and F.~Gao, ``Geometrically constrained trajectory optimization for multicopters,'' \emph{IEEE Transactions on Robotics}, vol.~38, no.~5, pp. 3259--3278, 2022.

\bibitem{liu2022free}
G.~Liu, J.~C. Trinkle, Y.~Yang, and J.~Li, ``On free velocity cones of narrow passages of high-dof kinematic chains,'' \emph{IEEE Transactions on Robotics}, vol.~38, no.~5, pp. 2686--2702, 2022.

\bibitem{ratliff2009chomp}
N.~Ratliff, M.~Zucker, J.~A. Bagnell, and S.~Srinivasa, ``Chomp: Gradient optimization techniques for efficient motion planning,'' in \emph{2009 IEEE International Conference on Robotics and Automation}.\hskip 1em plus 0.5em minus 0.4em\relax IEEE, 2009, pp. 489--494.

\bibitem{kalakrishnan2011stomp}
M.~Kalakrishnan, S.~Chitta, E.~Theodorou, P.~Pastor, and S.~Schaal, ``Stomp: Stochastic trajectory optimization for motion planning,'' in \emph{2011 IEEE International Conference on Robotics and Automation}.\hskip 1em plus 0.5em minus 0.4em\relax IEEE, 2011, pp. 4569--4574.

\bibitem{schulman2014}
J.~Schulman, Y.~Duan, J.~Ho, A.~Lee, I.~Awwal, H.~Bradlow, J.~Pan, S.~Patil, K.~Goldberg, and P.~Abbeel, ``Motion planning with sequential convex optimization and convex collision checking,'' \emph{The International Journal of Robotics Research}, vol.~33, no.~9, pp. 1251--1270, 2014.

\bibitem{zucker2013chomp}
M.~Zucker, N.~Ratliff, A.~D. Dragan, M.~Pivtoraiko, M.~Klingensmith, C.~M. Dellin, J.~A. Bagnell, and S.~S. Srinivasa, ``Chomp: Covariant hamiltonian optimization for motion planning,'' \emph{The International Journal of Robotics Research}, vol.~32, no. 9-10, pp. 1164--1193, 2013.

\bibitem{sarkka2013spatiotemporal}
S.~S{\"a}rkk{\"a}, A.~Solin, and J.~Hartikainen, ``Spatiotemporal learning via infinite-dimensional bayesian filtering and smoothing: A look at {G}aussian process regression through kalman filtering,'' \emph{IEEE Signal Processing Magazine}, vol.~30, no.~4, pp. 51--61, 2013.

\bibitem{dellaert2012factor}
F.~Dellaert, ``Factor graphs and gtsam: A hands-on introduction,'' Georgia Institute of Technology, Tech. Rep., 2012.

\bibitem{kurniawati2008sarsop}
H.~Kurniawati, D.~Hsu, and W.~S. Lee, ``Sarsop: Efficient point-based pomdp planning by approximating optimally reachable belief spaces.'' in \emph{Robotics: Science and systems}, vol. 2008.\hskip 1em plus 0.5em minus 0.4em\relax Citeseer, 2008.

\bibitem{somani2013despot}
A.~Somani, N.~Ye, D.~Hsu, and W.~S. Lee, ``Despot: Online pomdp planning with regularization,'' \emph{Advances in neural information processing systems}, vol.~26, 2013.

\bibitem{jordan1999introduction}
M.~I. Jordan, Z.~Ghahramani, T.~S. Jaakkola, and L.~K. Saul, ``An introduction to variational methods for graphical models,'' \emph{Machine learning}, vol.~37, pp. 183--233, 1999.

\bibitem{opper2009}
M.~Opper and C.~Archambeau, ``The variational {G}aussian approximation revisited,'' \emph{Neural Computation}, vol.~21, no.~3, pp. 786--792, 2009.

\bibitem{rezende2014stochastic}
D.~J. Rezende, S.~Mohamed, and D.~Wierstra, ``Stochastic backpropagation and variational inference in deep latent {G}aussian models,'' in \emph{International Conference on Machine Learning}, vol.~2, 2014, p.~2.

\bibitem{shankar2020learning}
T.~Shankar and A.~Gupta, ``Learning robot skills with temporal variational inference,'' in \emph{International Conference on Machine Learning}.\hskip 1em plus 0.5em minus 0.4em\relax PMLR, 2020, pp. 8624--8633.

\bibitem{garcia2015posterior}
{\'A}.~F. Garc{\'\i}a-Fern{\'a}ndez, L.~Svensson, M.~R. Morelande, and S.~S{\"a}rkk{\"a}, ``Posterior linearization filter: Principles and implementation using sigma points,'' \emph{IEEE Transactions on Signal Processing}, vol.~63, no.~20, pp. 5561--5573, 2015.

\bibitem{barfoot2020exactly}
T.~D. Barfoot, J.~R. Forbes, and D.~J. Yoon, ``Exactly sparse {G}aussian variational inference with application to derivative-free batch nonlinear state estimation,'' \emph{The International Journal of Robotics Research}, vol.~39, no.~13, pp. 1473--1502, 2020.

\bibitem{liu2016stein}
Q.~Liu and D.~Wang, ``Stein variational gradient descent: A general purpose bayesian inference algorithm,'' \emph{Advances in Neural Information Processing Systems}, vol.~29, 2016.

\bibitem{kappen2012optimal}
H.~J. Kappen, V.~G{\'o}mez, and M.~Opper, ``Optimal control as a graphical model inference problem,'' \emph{Machine Learning}, vol.~87, pp. 159--182, 2012.

\bibitem{todorov2008general}
E.~Todorov, ``General duality between optimal control and estimation,'' in \emph{2008 47th IEEE Conference on Decision and Control}.\hskip 1em plus 0.5em minus 0.4em\relax IEEE, 2008, pp. 4286--4292.

\bibitem{toussaint2009robot}
M.~Toussaint, ``Robot trajectory optimization using approximate inference,'' in \emph{Proceedings of the 26th annual International Conference on Machine Learning}, 2009, pp. 1049--1056.

\bibitem{millidge2020relationship}
B.~Millidge, A.~Tschantz, A.~K. Seth, and C.~L. Buckley, ``On the relationship between active inference and control as inference,'' in \emph{Active Inference: First International Workshop, IWAI 2020, Co-located with ECML/PKDD 2020, Ghent, Belgium, September 14, 2020, Proceedings 1}.\hskip 1em plus 0.5em minus 0.4em\relax Springer, 2020, pp. 3--11.

\bibitem{friston2016active}
K.~Friston, T.~FitzGerald, F.~Rigoli, P.~Schwartenbeck, G.~Pezzulo \emph{et~al.}, ``Active inference and learning,'' \emph{Neuroscience \& Biobehavioral Reviews}, vol.~68, pp. 862--879, 2016.

\bibitem{todorov2005generalized}
E.~Todorov and W.~Li, ``A generalized iterative lqg method for locally-optimal feedback control of constrained nonlinear stochastic systems,'' in \emph{Proceedings of the 2005, American Control Conference, 2005.}\hskip 1em plus 0.5em minus 0.4em\relax IEEE, 2005, pp. 300--306.

\bibitem{CheGeoPav15a}
Y.~Chen., T.~Georgiou, and M.~Pavon, ``Optimal steering of a linear stochastic system to a final probability distribution, {P}art {I},'' \emph{IEEE Trans.\ on Automatic Control}, vol.~61, no.~5, pp. 1158--1169, 2016.

\bibitem{CheGeoPav15b}
Y.~Chen, T.~T. Georgiou, and M.~Pavon, ``Optimal steering of a linear stochastic system to a final probability distribution, {P}art {II},'' \emph{IEEE Trans.\ on Automatic Control}, vol.~61, no.~5, pp. 1170--1180, 2016.

\bibitem{CheGeoPav17a}
Y.~Chen, T.~Georgiou, M.~Pavon \emph{et~al.}, ``Optimal steering of a linear stochastic system to a final probability distribution-{P}art {III},'' \emph{IEEE Transactions on Automatic Control}, vol.~63, no.~9, pp. 3112--3118, 2018.

\bibitem{goldshtein2017finite}
M.~Goldshtein and P.~Tsiotras, ``Finite-horizon covariance control of linear time-varying systems,'' in \emph{2017 IEEE 56th Annual Conference on Decision and Control (CDC)}.\hskip 1em plus 0.5em minus 0.4em\relax IEEE, 2017, pp. 3606--3611.

\bibitem{Chen2016relation}
Y.~Chen, T.~T. Georgiou, and M.~Pavon, ``On the relation between optimal transport and schr{\"o}dinger bridges: A stochastic control viewpoint,'' \emph{Journal of Optimization Theory and Applications}, vol. 169, no.~2, pp. 671--691, 2016.

\bibitem{yu2023covariance}
H.~Yu, Z.~Chen, and Y.~Chen, ``Covariance steering for nonlinear control-affine systems,'' \emph{arXiv preprint arXiv:2108.09530}, 2023.

\bibitem{ridderhof2020chance}
J.~Ridderhof, K.~Okamoto, and P.~Tsiotras, ``Chance constrained covariance control for linear stochastic systems with output feedback,'' in \emph{2020 59th IEEE Conference on Decision and Control (CDC)}.\hskip 1em plus 0.5em minus 0.4em\relax IEEE, 2020, pp. 1758--1763.

\bibitem{okamoto2018optimal}
K.~Okamoto, M.~Goldshtein, and P.~Tsiotras, ``Optimal covariance control for stochastic systems under chance constraints,'' \emph{IEEE Control Systems Letters}, vol.~2, no.~2, pp. 266--271, 2018.

\bibitem{yu2023}
H.~Yu and Y.~Chen, ``A {G}aussian variational inference approach to motion planning,'' \emph{IEEE Robotics and Automation Letters}, vol.~8, no.~5, pp. 2518--2525, 2023.

\bibitem{barfoot2014batch}
T.~D. Barfoot, C.~H. Tong, and S.~S{\"a}rkk{\"a}, ``Batch continuous-time trajectory estimation as exactly sparse {G}aussian process regression.'' in \emph{Robotics: Science and Systems}, vol.~10.\hskip 1em plus 0.5em minus 0.4em\relax Citeseer, 2014, pp. 1--10.

\bibitem{sarkka2019applied}
S.~S{\"a}rkk{\"a} and A.~Solin, \emph{Applied stochastic differential equations}.\hskip 1em plus 0.5em minus 0.4em\relax Cambridge University Press, 2019, vol.~10.

\bibitem{van2014renyi}
T.~Van~Erven and P.~Harremos, ``R{\'e}nyi divergence and kullback-leibler divergence,'' \emph{IEEE Transactions on Information Theory}, vol.~60, no.~7, pp. 3797--3820, 2014.

\bibitem{magnus2019}
J.~R. Magnus and H.~Neudecker, \emph{Matrix differential calculus with applications in statistics and econometrics}.\hskip 1em plus 0.5em minus 0.4em\relax John Wiley \& Sons, 2019.

\bibitem{arasaratnam2007discrete}
I.~Arasaratnam, S.~Haykin, and R.~J. Elliott, ``Discrete-time nonlinear filtering algorithms using {G}auss-{H}ermite quadrature,'' \emph{Proceedings of the IEEE}, vol.~95, no.~5, pp. 953--977, 2007.

\bibitem{ito2000gaussian}
K.~Ito and K.~Xiong, ``Gaussian filters for nonlinear filtering problems,'' \emph{IEEE Transactions on Automatic Control}, vol.~45, no.~5, pp. 910--927, 2000.

\bibitem{liu1994}
Q.~Liu and D.~A. Pierce, ``A note on {G}auss-{H}ermite quadrature,'' \emph{Biometrika}, vol.~81, no.~3, pp. 624--629, 1994.

\bibitem{heiss2008likelihood}
F.~Heiss and V.~Winschel, ``Likelihood approximation by numerical integration on sparse grids,'' \emph{Journal of Econometrics}, vol. 144, no.~1, pp. 62--80, 2008.

\bibitem{bungartz2004sparse}
H.-J. Bungartz and M.~Griebel, ``Sparse grids,'' \emph{Acta numerica}, vol.~13, pp. 147--269, 2004.

\bibitem{wainwright2000tree}
M.~J. Wainwright, E.~Sudderth, and A.~Willsky, ``Tree-based modeling and estimation of {G}aussian processes on graphs with cycles,'' \emph{Advances in Neural Information Processing Systems}, vol.~13, 2000.

\bibitem{chou1994multiscale}
K.~C. Chou, A.~S. Willsky, and R.~Nikoukhah, ``Multiscale systems, kalman filters, and riccati equations,'' \emph{IEEE Transactions on Automatic Control}, vol.~39, no.~3, pp. 479--492, 1994.

\bibitem{su2015convergence}
Q.~Su and Y.-C. Wu, ``On convergence conditions of gaussian belief propagation,'' \emph{IEEE Transactions on Signal Processing}, vol.~63, no.~5, pp. 1144--1155, 2015.

\bibitem{Stra22}
G.~Strang, \emph{Introduction to linear algebra}.\hskip 1em plus 0.5em minus 0.4em\relax SIAM, 2022.

\bibitem{higgins2016beta}
I.~Higgins, L.~Matthey, A.~Pal, C.~Burgess, X.~Glorot, M.~Botvinick, S.~Mohamed, and A.~Lerchner, ``beta-vae: Learning basic visual concepts with a constrained variational framework,'' in \emph{International Conference on Learning Representations}, 2016.

\bibitem{nguyen2012modeling}
C.~V. Nguyen, S.~Izadi, and D.~Lovell, ``Modeling kinect sensor noise for improved 3d reconstruction and tracking,'' in \emph{2012 second International Conference on 3D Imaging, Modeling, Processing, Visualization \& Transmission}.\hskip 1em plus 0.5em minus 0.4em\relax IEEE, 2012, pp. 524--530.

\bibitem{chen2022should}
G.~Chen, H.~Yu, W.~Dong, X.~Sheng, X.~Zhu, and H.~Ding, ``What should be the input: Investigating the environment representations in sim-to-real transfer for navigation tasks,'' \emph{Robotics and Autonomous Systems}, vol. 153, p. 104081, 2022.

\bibitem{Bec17}
A.~Beck, \emph{First-order methods in optimization}.\hskip 1em plus 0.5em minus 0.4em\relax SIAM, 2017.

\bibitem{coleman2014reducing}
D.~Coleman, I.~Sucan, S.~Chitta, and N.~Correll, ``Reducing the barrier to entry of complex robotic software: a moveit! case study,'' \emph{arXiv preprint arXiv:1404.3785}, 2014.

\bibitem{rooks2006harmonious}
B.~Rooks, ``The harmonious robot,'' \emph{Industrial Robot: An International Journal}, vol.~33, no.~2, pp. 125--130, 2006.

\bibitem{quigley2009ros}
M.~Quigley, K.~Conley, B.~Gerkey, J.~Faust, T.~Foote, J.~Leibs, R.~Wheeler, A.~Y. Ng \emph{et~al.}, ``Ros: an open-source robot operating system,'' in \emph{ICRA workshop on open source software}, vol.~3, no. 3.2.\hskip 1em plus 0.5em minus 0.4em\relax Kobe, Japan, 2009, p.~5.

\end{thebibliography}

\end{document}